\documentclass{article}
\usepackage{natbib}
\usepackage{microtype}
\usepackage{graphicx}
\usepackage{subfigure}
\usepackage{booktabs} % for professional tables

\usepackage[accepted]{icml2022}
\usepackage{hyperref}
\icmltitlerunning{A Robust Phased Elimination Algorithm for Corruption-Tolerant Gaussian Process Bandits}

\usepackage{graphicx}
\usepackage{enumitem}
\usepackage{multicol}
\usepackage{amsmath,amsthm,amssymb}
\usepackage{dsfont}
\usepackage{hyperref}
\usepackage[capitalise]{cleveref}
\usepackage{xcolor}
\usepackage{bbm}
\usepackage{thmtools}
\usepackage{thm-restate}

\DeclareMathOperator*{\argmax}{arg\,max}

\newcommand{\E}{\mathbb{E}}

\newcommand{\R}{\mathbb{R}}

\newcommand{\Sc}{\mathcal{S}}
\newcommand{\Xc}{\mathcal{X}}

\newcommand{\Hc}{\mathcal{H}}

\newcommand{\Tc}{\mathcal{T}}

\newcommand{\ty}{\tilde{y}}
\newcommand{\tY}{\tilde{Y}}

\newcommand{\tmu}{\tilde{\mu}}

\newcommand{\bone}{\mathds{1}}

\newtheorem{theorem}{Theorem}

\newtheorem{lemma}[theorem]{Lemma}
\newtheorem{corollary}[theorem]{Corollary}

\definecolor{mightadd}{rgb}{0.53, 0.66, 0.42}
\definecolor{comment}{rgb}{0.88, 0.66, 0.37}

\newcommand{\hilbert}{\Hc_{k}(\Xc)}

\newcommand{\Rc}{\mathcal{R}}
\newcommand{\xtilde}{\widetilde{x}}
\newcommand{\ytilde}{\widetilde{y}}
\newcommand{\RR}{\mathbb{R}}
\newcommand{\Ac}{\mathcal{A}}
\newcommand{\xbar}{\bar{x}}
\newcommand{\Hbar}{\bar{H}}
\newcommand{\Otilde}{O^*}
\newcommand{\tepsilon}{\widetilde{\epsilon}}
\renewcommand{\tmu}{\widetilde{\mu}}
\renewcommand{\tY}{\widetilde{Y}}
\renewcommand{\ty}{\widetilde{y}}
\usepackage{lipsum}

\begin{document}

\twocolumn[
\icmltitle{A Robust Phased Elimination Algorithm for \\ Corruption-Tolerant Gaussian Process Bandits}




\begin{icmlauthorlist}
    \icmlauthor{Ilija Bogunovic}{to}
    \icmlauthor{Zihan Li}{goo}
    \icmlauthor{Andreas Krause}{to}
    \icmlauthor{Jonathan Scarlett}{goo}
\end{icmlauthorlist}

\icmlaffiliation{to}{ETH Z\"urich}
\icmlaffiliation{goo}{National University of Singapore}

\icmlcorrespondingauthor{Ilija Bogunovic}{ilijab@ethz.ch}

%\icmlkeywords{Machine Learning, ICML}

\vskip 0.3in
]



\printAffiliationsAndNotice{}  % leave blank if no need to mention equal contribution

\begin{abstract}
    We consider the sequential optimization of an unknown, continuous, and expensive to evaluate reward function, from noisy and adversarially corrupted observed rewards. 
    When the corruption attacks are subject to a suitable budget $C$ and the function lives in a Reproducing Kernel Hilbert Space (RKHS), the problem can be posed as {\em corrupted Gaussian process (GP) bandit optimization}. We propose a novel robust elimination-type algorithm that runs in epochs, combines exploration with infrequent switching to select a small subset of actions, and plays each action for multiple time instants. Our algorithm, {\em Robust GP Phased Elimination (RGP-PE)}, successfully balances robustness to corruptions with exploration and exploitation such that its performance degrades minimally in the presence (or absence) of adversarial corruptions.  
    When $T$ is the number of samples and $\gamma_T$ is the maximal information gain, the corruption-dependent term in our regret bound is $O(C \gamma_T^{3/2})$, which is significantly tighter than the existing $O(C \sqrt{T \gamma_T})$ for several commonly-considered kernels. 
    We perform the first empirical study of robustness in the corrupted GP bandit setting, and show that our algorithm is robust against a variety of adversarial attacks. \looseness=-1
\end{abstract}
%, e.g., for squared-exponential kernel $O(C (\log T)^{\tfrac{3(d+1)}{2}})$ in comparison to the state of the art $O(C\sqrt{T} (\log T)^{\tfrac{d+1}{2}})$.
   
\section{Introduction}

Black-box optimization is a fundamental problem with far-reaching applications including hyperparameter tuning \cite{snoek2012practical}, robotics \cite{lizotte2007automatic}, and chemical design \cite{griffiths202constrained}, among others.  To make the problem tractable, a variety of smoothness properties have been adopted, and Reproducing Kernel Hilbert Space (RKHS) functions have proved to provide a versatile framework that can be tackled via Gaussian process (GP) based algorithms \cite{srinivas2009gaussian,chowdhury17kernelized}.  This problem is often referred to as {\em GP bandits} or {\em kernelized bandits}.

While an extensive line of works have established GP bandit algorithms and regret bounds, settings with adversarial corruptions have only arisen relatively recently.  Such corruptions may come in the form of outliers \cite{martinez2018practical}, perturbations of sampled inputs \cite{beland2017bayes,nogueira2016unscented,dai2017stable}, adversarial noise in the rewards \cite{bogunovic2020corruption}, or perturbations of the final recommendation \cite{bogunovic2018adversarially}.  In this work, we are interested in the setting of adversarial noise in the rewards, 
in which the performance of standard non-robust GP bandit algorithms can deteriorate significantly (see \cref{fig:inverted_final}).\looseness=-1

The first work considering this setting \cite{bogunovic2020corruption} established regret bounds for various algorithms depending on the degree of knowledge on the corruption level $C$ (defined formally in Section \ref{sec:setup}).  A key limitation in their regret bound is that the main corruption-dependent term, $C$, and the usual uncorrupted regret term, which is $\sqrt{T}$ or higher (with time horizon $T$), are {\em multiplied together}.  That is, the dependence on $C$ is multiplicative with respect to the uncorrupted bound.   Analogous studies of bandits with independent arms \cite{lykouris2018stochastic,gupta2019better} or linear rewards \cite{bogunovic2021stochastic} suggest that {\em additive} dependence may be possible, but this has remained very much open in the GP bandit setting.

In this paper, we address this fundamental gap in the literature by introducing a novel algorithm in which the uncorrupted term and the $C$-dependent term are clearly decoupled, and the latter is only multiplied by a kernel-dependent function of $T$ that can be much smaller than $\sqrt{T}$. 
    
\begin{figure*}[t]
  \centering\includegraphics[height=0.2\textwidth, width=0.4\textwidth]{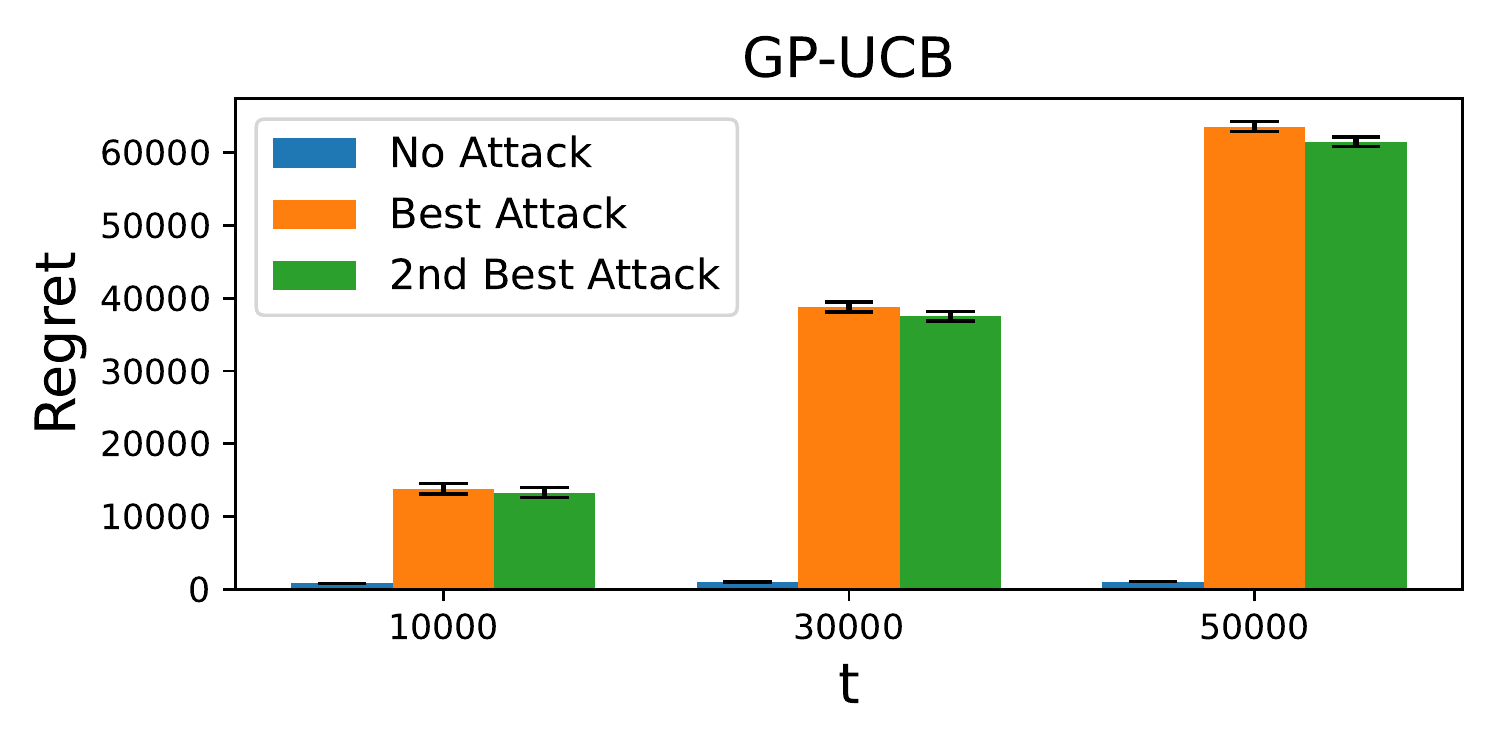} \hspace{15mm}
  \centering\includegraphics[height=0.2\textwidth, width=0.4\textwidth]{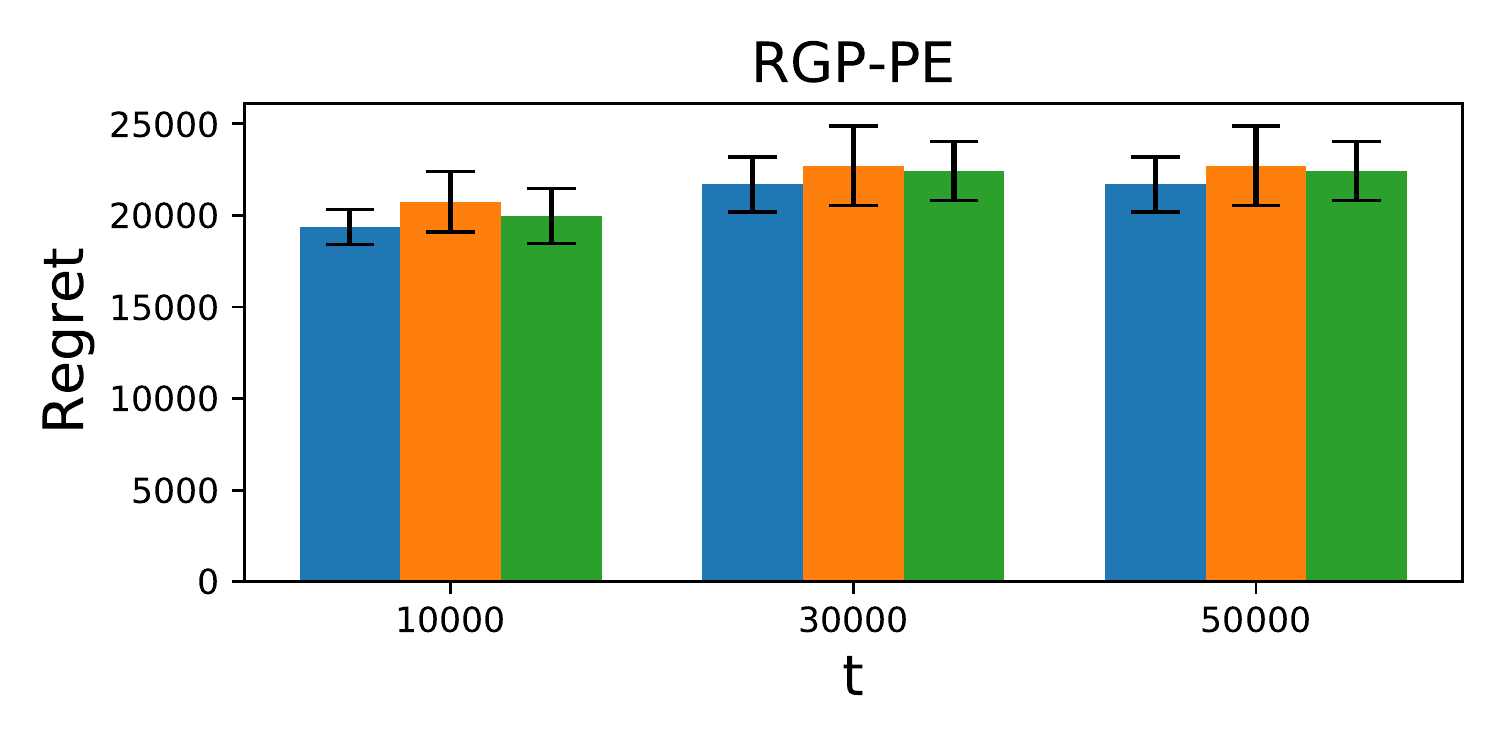}
    \vspace{-1em}

    \caption{Performance of GP-UCB \cite{srinivas2009gaussian} and Robust GP Phased Elimination (RGP-PE, this work) with no attacks and the two most effective corruption attacks on the Robot3D pushing task. As the number of samples $t$ increases, the performance of non-robust GP-UCB deteriorates significantly under both attacking strategies, while the performance of the proposed algorithm remains robust.\looseness=-1}
%   \vspace{-1em}
  \label{fig:inverted_final}
\end{figure*}

% Related work paragraph structure: 
% \begin{itemize}[noitemsep, topsep=0pt]
%     \item Related work in the K-armed bandit and linear bandit setting; distinguish between two adversaries (one that observes the decision and the other one that does not) 
%     \item Paragraph on our Linear bandit paper and our Corrupted GP bandit setting
%     \item Robustness considerations in GP bandits? Also, related work of misspecification (IB: not sure if we need this one.)
%     \item Paragraph on corrupted: Active learning; RL; online learning.
% \end{itemize}

\textbf{Related work.} The closest work to ours is the one of \citet{bogunovic2020corruption}, which also considers the Corruption-Tolerant GP Bandit setting. In that work, the authors propose a confidence-bound-based algorithm with enlarged confidence.  As outlined above, the regret bound therein scales as $O(C\sqrt{T \gamma_T})$, and the possibility of additive $C$ dependence was left as an open problem.

The question of additive vs.~multiplicative dependence first arose in multi-armed bandits with independent arms, with an initial work \cite{lykouris2018stochastic} being multiplicative, and a subsequent work  \cite{gupta2019better} improving to additive.  Closer to our setup (and in fact a special case of it via the linear kernel) is the case of corrupted stochastic linear bandits, in which additive dependence was obtained in \cite{bogunovic2021stochastic}, with the corruption term more precisely being $O(Cd^{3/2} \log T)$ under mild assumptions.  Our main result will achieve a similar bound as a special case, while being much more general due to handling general kernels, and adopting largely distinct GP-based algorithmic and mathematical techniques.  Other less related results for corrupted linear bandits (e.g., contextual or instance-dependent) are given by \citet{li2019stochastic} and \citet{zhao2021linear}.

Recently, \cite{kirschner2021bias} considered a related GP-based optimization setting with adversarially corrupted observations, in which the corruptions are not allowed to depend on the chosen action (similar to \cite{krishnamurthy2018semiparametric}), in contrast to the stronger adversary considered in \cite{bogunovic2020corruption} and our work. Under the assumption that corruptions are bounded by a constant in every round, they provide an algorithm based on a reduction to dueling bandits that attains similar sublinear cumulative regret bounds to the uncorrupted setting.  Since such a result is attained despite the total amount of corruption being linear in $T$, this highlights that the assumed power of the adversary can drastically impact the problem.

In the GP setting, other notions of robustness have included outliers \cite{martinez2018practical}, misspecification \cite{camilleri2021high,bogunovic2021misspecified}, input noise \cite{beland2017bayes,nogueira2016unscented,dai2017stable}, risk-aversion \cite{nguyen2021value,cakmak2020bayesian,makarova2021risk}, and corruptions in the final recommendation \cite{bogunovic2018adversarially, kirschner2020distributionally}.  Moreover, other settings with adversarial corruptions have included online \cite{ito2021optimal}, active \cite{chen2021corruption}, reinforcement learning \cite{lykouris2021corruption,wei2021model,banihashem2021defense}, and multi-agent RL \cite{liu2021cooperative}.  To our knowledge, none of the techniques in the preceding works are directly applicable in our setting.\looseness=-1

\textbf{Contributions.} Our main contributions are as follows:
\begin{itemize}[noitemsep, topsep=0pt]
    \item We provide a novel algorithm for GP bandit optimization with adversarial corruptions, that attains the first regret bound to avoid multiplying the uncorrupted part by the corruption level $C$.  Our algorithm crucially incorporates a {\em rare switching} idea, along with a non-standard robust estimator, enlarged confidence bounds, and a minimal number of plays of each selected action; see Sections \ref{sec:model} and \ref{sec:alg} for details.
    \item We show that our regret bound is {\em provably near-optimal} for the SE kernel, and recovers recently-established bounds for stochastic linear bandits \cite{bogunovic2021stochastic} that are also known to be near-optimal.  For the Mat\'ern kernel, the degree of tightness depends on the dimension and smoothness parameter, but our bound strictly improves on that of \citet{bogunovic2020corruption} in all scaling regimes where the latter is non-trivial (i.e., sub-linear in $T$); see Table \ref{tbl:summary} on Page \pageref{tbl:summary} for a summary.
    \item We demonstrate that our algorithm is able to successfully defend against various attacks, including those recently proposed by \citet{han2021adversarial}.
    \item In Appendix \ref{sec:linear}, we explore an alternative approach based on a reduction from GP bandits to linear bandits, and show that it can reap some, but not all, of the advantages discussed above.
\end{itemize}
    
\section{Problem Setting and Preliminaries} \label{sec:setup}

We consider the Gaussian process bandit (i.e., kernelized bandit) problem, in which the goal of the learner is to maximize the collected rewards by sequentially querying the unknown reward function $f: \Xc \to \R$ over $T$ rounds. In particular, at every time $t$, the learner selects $x_t \in \Xc$ and receives a noisy reward observation\looseness=-1
\begin{equation}
    y_t = f(x_t) + \epsilon_t,
\end{equation}
where $\epsilon_t$ is assumed to be $\sigma$-sub-Gaussian with independence over time steps, and $\sigma$ is also known.% to the learner.\looseness=-1 

We consider the corrupted setting in which, besides the stochastic noise, the observations at every time step are adversarially corrupted, so that the learner observes\looseness=-1
\begin{equation}
    \ty_t = y_t + c_t.
\end{equation}

Following \citet{bogunovic2020corruption}, we make the following assumptions on the adversary:
\begin{itemize}[noitemsep, topsep=0pt, leftmargin=5ex]
    \item The adversary knows the true reward function $f(\cdot)$, and, at every round $t$, it observes $x_t$ before deciding upon the corruption $c_t$.
    \item The total adversarial corruption budget over $T$ rounds is bounded as follows:    
    \begin{equation}
        \sum_{t=1}^T |c_t| \leq C.
    \end{equation}
    In this paper, we focus on the case where $C$ is known to the learner.  We expect unknown-$C$ extensions to be possible in a similar spirit to \cite{bogunovic2021stochastic}, but since the known $C$ case is already challenging, we prefer not to obfuscate our new ideas with the added technical difficulty of addressing unknown $C$.
\end{itemize}

The domain $\Xc$ is assumed to either be finite, or a compact subset of $\RR^d$ for some dimension $d$ (e.g., $\Xc = [0,1]^d$).  In either case, $\Xc$ is endowed with a continuous, positive semidefinite kernel function $k(\cdot,\cdot): \Xc \times \Xc \to \mathbb{R}$ that is normalized to satisfy $k(x, x')\leq 1$ for all $x, x' \in \Xc$. We further assume that $f$ has a bounded norm in the corresponding Reproducing Kernel Hilbert Space (RKHS) $\Hc_k$, i.e., $ \|f\|_k \leq B$ (see \cref{sec:prelim} for more details). This assumption permits the construction of confidence bounds via Gaussian process (GP) models (\cref{sec:specific}).\looseness=-1

The learner's performance is measured using the widely-considered notion of cumulative regret:
\begin{equation} \label{eq:regret}
  R_T =  \sum_{t=1}^T \Big(\max_{x\in \Xc} f(x) - f(x_t)\Big),
\end{equation}
and we are interested in the {\em joint} dependence of $R_T$ on $C$ and $T$. 
As noted by \citet{lykouris2018stochastic} and \citet{bogunovic2020corruption}, one could alternatively define the cumulative regret with respect to the corrupted values
(i.e., $f(x) + c_t$), but the difference between the two is minor since these notions coincide to within an additive term of $2C$.

\subsection{Gaussian Process Model under Corruptions} \label{sec:model}
In the standard (non-corrupted) setting, previous algorithms use
(i) zero-mean GP priors for modeling the uncertainty in
$f$ (i.e., they assume $f \sim GP(0, k)$), and (ii) Gaussian likelihood models for the observations. 
As more data points become available, Bayesian posterior updates are then performed according to a misspecified model in which the noise variables $\epsilon_t = y_t - f(x_t)$ are assumed to be drawn independently across $t$ from $\mathcal{N}(0, \lambda)$, where $\lambda$ is a hyperparameter that may differ from the true noise variance $\sigma^2$. In particular, in the absence of corruptions, given a sequence of points $\lbrace x_1, \dots, x_t \rbrace$ and their noisy observations $\lbrace y_1, \dots, y_t \rbrace$, the posterior mean and variance are given by
\begin{align}
    \mu_{t}(x) &= k_t(x)^T\big(K_t + \lambda I_t \big)^{-1} Y_t,  \label{eq:posterior_mean} \\ 
    \sigma_{t}^2(x) &= k(x,x) - k_t(x)^T \big(K_t + \lambda I_t \big)^{-1} k_t(x), \label{eq:posterior_variance}
\end{align}
where $k_t(x) = \big[k(x_i,x)\big]_{i=1}^t$, $K_t = \big[k(x_t,x_{t'})\big]_{t,t'}$ is the kernel matrix, and $Y_t \in \RR^t$ contains the non-corrupted observations up to time $t$, i.e., $Y_t[i] = y_i$ for $i \in [t]$.

% \textbf{Corrupted model.} 
In the corrupted setting, given a sequence of inputs $\lbrace x_1, \dots , x_t \rbrace$ and their corrupted observations $\lbrace \ty_1, \dots , \ty_t\rbrace$ (with $\ty_i = y_i + c_i$), we propose the following non-standard robust posterior mean estimator:
\begin{equation}
    \tmu_{t}(x) = k_t(x)^T(K_t + \lambda I_t)^{-1} \tY_t, \label{eq:corrupted_mean}
\end{equation}
where $\tY_t \in \R^t$ and $\tY_t[i] = \tfrac{\sum_{j=1}^t \bone \lbrace x_i = x_j \rbrace \ty_j}{\sum_{j=1}^t \bone \lbrace x_i = x_j \rbrace}$ for $i \in [t]$.  
Intuitively, the averaging of terms corresponding to identical actions is done in order to diminish the impact of corruption, and this will be a crucial component of our analysis.
% Intuitively, $\tmu_{t}(\cdot)$ is more robust than $\mu_t(\cdot)$, since rewards that correspond to the same action are averaged which in turn diminishes the impact of corruption. 

In our algorithm, besides $\tmu_{t}(\cdot)$, we will also make use of the standard posterior variance $\sigma_{t}^2(\cdot)$ as given in \cref{eq:posterior_variance}; the use of this quantity is intuitively reasonable because GP posterior variances do not depend on the observations.  % As in the standard case, our regret guarantees in the corrupted setting will also depend on the maximum information gain $\gamma_T$. 

% Next, we explain the main steps of our Robust GP Phased Elimination algorithm (\cref{alg:cpe}).

The main quantity that characterizes the regret bounds in the non-corrupted setting is the \emph{maximum information gain} \cite{srinivas2009gaussian}, defined at time $t$ as 
\begin{equation}
    \label{eq:max_info_gain}
    \gamma_t = \max_{x_1, \dots, x_t} \frac{1}{2} \ln  \det(I_t + \lambda^{-1}K_t),
\end{equation}
and we will also make use of this quantity.

\section{Robust GP Phased Elimination} \label{sec:alg}

\subsection{Algorithm and Confidence Bounds}
    
Our algorithm works in epochs indexed by $h=0,1,\dots, H-1$, each of which consists of sampling a batch of points.  The epoch lengths may be chosen adaptively, and hence $H$ may not be deterministic, but we will ensure with probability one that $H \le \Hbar$ with $\Hbar = \log_2 T$.  The length of epoch $h$ is denoted by $u_h$, so that $\sum_{h=0}^{H-1} u_h=T$. 

The algorithm and analysis are based on the widespread notion of confidence bounds.  While our confidence bounds will be expanded to account for corruptions, it is useful to consider the following generic assumption regarding non-corrupted observations (although the algorithm cannot access these, they will appear in our mathematical analysis).

\begin{restatable}[Regular confidence bounds]{assumption}{rcb}
\label{assumption:regular_confidence_bounds}
	Let $\mu^{(h)}(x)$ and $\sigma^{(h)}(x)$ denote the posterior mean and standard deviation computed (hypothetically) using only the non-corrupted observations $\lbrace (x_i, y_i) \rbrace_{i=1}^{u_h}$ in epoch $h$ using \cref{eq:posterior_mean,eq:posterior_variance}. We assume that given $\delta \in (0,1)$, there exists a sequence of parameters $\beta_h = \beta_h(\delta)$ which is non-decreasing in $h$ and yields with probability at least $1-\delta$ that
    \begin{equation}
        |\mu^{(h)}(x)-f(x)|\leq \beta_h\sigma^{(h)}(x)
    \end{equation}
    simultaneously for all $h \ge 0$ and $x\in\Xc$.
\end{restatable}

Specific choices of $\beta_h$ satisfying this assumption will be considered in Section \ref{sec:specific}.

Similarly to previous kernelized algorithms (e.g., \citet{bogunovic2020corruption,bogunovic2021misspecified}), our proposed algorithm makes use of enlarged confidence bounds. Hence, our first result concerns concentration of an RKHS member under corrupted observations, where we make use of the proposed estimator from \cref{eq:corrupted_mean}.   

\begin{restatable}[Corrupted confidence bounds]{lemma}{ccb}
    \label{lemma:corrupted_confidence_bounds}
    Under \cref{assumption:regular_confidence_bounds}, let $\tmu^{(h)}(x)$ denote the posterior mean based on only the corrupted observations $\lbrace (x_i, \ty_i) \rbrace_{i=1}^{u_h}$ in epoch $h$ using \cref{eq:corrupted_mean}, 
    and let  
    $u_{\min} \geq 1$ denote the minimum number of times any single action from $\lbrace x_i \rbrace_{i=1}^{u_h}$ is played, i.e., 
    $u_{\min} = \min_{x \in \lbrace x_1, \dots, x_{u_h} \rbrace} \sum_{i=1}^{u_h} \bone \lbrace x_i = x \rbrace$. Then, with probability at least $1-\delta$, it holds for all $x \in \Xc$ and $h\geq 0$ that\looseness=-1
    \begin{equation}
      |\tmu^{(h)} (x) - f (x)| \leq \Big( \beta_h+ \tfrac{C\sqrt{u_h}}{u_{\min} \lambda} \Big) \sigma^{(h)}(x).
    \end{equation} 
\end{restatable}

The confidence-bound enlargement is proportional to the total amount of corruption $C$. While this is similar to the confidence intervals used by \citet{bogunovic2020corruption} (Lemma 2), we note the following two important differences: 
\begin{itemize}[noitemsep, topsep=0pt, leftmargin=5ex]
    \item We make use of a novel kernelized mean estimator (\cref{eq:corrupted_mean}) that takes average over rewards corresponding to the same played action;
    \item Our enlargement term is $O(C \tfrac{\sqrt{u_h}}{u_{\min}})$, as opposed to $O(C)$ used in \cite{bogunovic2020corruption}(Lemma 2).  We will typically apply this lemma with $\tfrac{\sqrt{u_h}}{u_{\min}} \ll 1$, so that our confidence width is much smaller.
\end{itemize}
For the second of these, the intuition is that if the same action is played multiple times, it becomes harder for the adversary to hide the true value (i.e., since the rewards of the same played actions are averaged, the adversary needs to spend more of its budget corrupting the reward).

The Robust GP-Phased Elimination algorithm (\cref{alg:cpe}) proceeds in epochs (indexed by $h$) of exponentially increasing length $u_h$.
At every round $t$ (where $t \in \lbrace 1, \dots, l_h \rbrace$ and $l_h = 2^{h+1}$) within an epoch $h$, the algorithm selects an action maximizing a posterior uncertainty computed at some (possibly strictly earlier) time $t'$:
\begin{equation}
    x_t = \argmax_{x \in \Xc_h} \sigma_{t'}(x), \label{eq:xt_choice}
\end{equation}
where $\Xc_h$ denotes the set of active actions in epoch $h$. The selected action is then added to $\Sc_h$ which is a set 
that contains distinct actions selected in epoch $h$. 

The key idea behind using $t'$ instead of $t$ in Eq.~\eqref{eq:xt_choice} is to ensure that our algorithm \emph{rarely switches}, based on a condition relating to the information gain (Line 6), meaning that the same action $x_t$ is typically selected multiple times.  Whenever there are ties, they are resolved arbitrarily but consistently over rounds (i.e., if $\sigma_{t'}(\cdot)$ does not change, the same points are selected).  Based on Lines 6 to 9, we update $t'$ and recompute $\sigma_{t'}(x)$ only when $\det(I_t + \lambda^{-1} K_t)$ increases by a constant factor $\eta$.  % We also note that $\sigma_{t'}(x)$ is computed by using all the previously selected points $\lbrace x_i \rbrace_{i=1}^{t'}$. 

We note that related ideas of rare switching have appeared in the literature \citep[e.g.,][]{abbasi2011improved, wang2021provably}, but to our knowledge we are the first to use this idea in the kernelized bandit problem, and more importantly, the first to use it for the purpose of improving robustness.  Intuitively, by rarely switching, we obtain more samples of the same point, allowing us to average more of them together and making the ``averaged'' observation harder to corrupt.\looseness=-1

\begin{algorithm}[t!]
    \caption{Robust GP Phased Elimination (RGP-PE)}
    \label{alg:cpe}
    \begin{algorithmic}[1]
        \INPUT Domain $ \Xc \subset \R^d$, truncation parameter $\psi > 0$, corruption budget $C$, switching parameter $\eta > 1$, regularization parameter $\lambda > 0$
        \STATE Initialize $l_0 = 2$, and $h=0$ and $\Xc_h = \Xc$
        \STATE Set  $\Sc_h=\emptyset$, $t' = 0$, $\sigma_0(x) = 1$ for all $x \in \Xc_h$
        \FOR{$t = 1,2, \dots, l_h$} 
%            \STATE \todo{It might make sense replacing $t$ and $t'$ with $i$ and $i'$ here; since we use $t$ later on; The only issue is how this is used in the proofs;}
            \STATE Select $x_t = \argmax_{x \in \Xc_h} \sigma_{t'}(x)$ %\Comment{ties are broken arbitrarily but consistently \todo{Explain this.}}
            \STATE Update $\Sc_h \leftarrow \Sc_h \cup \lbrace x_t \rbrace$
            \IF{$\det(I_t + \lambda^{-1} K_t) > \eta \det(I_{t'} + \lambda^{-1} K_{t'})$}
                \STATE Set $t' \leftarrow t$
                \STATE Compute $\sigma_{t'}(\cdot)$ via \cref{eq:posterior_variance} by using $\lbrace x_i \rbrace_{i=1}^{t'}$
            \ENDIF
        \ENDFOR 
        \STATE Set $\xi_h(x)= \tfrac{\sum_{i=1}^{l_h} \bone \lbrace  x = x_i \rbrace}{l_h}$ for every $x \in \Sc_h$ %if $\zeta_h(a) = 0$, $u_h(a)=\lceil m_{h} \max\lbrace\zeta_h(a), \nu\rbrace \rceil$.
        \STATE Set $u_h(x)=\lceil l_{h} \max\lbrace \xi_h(x), \psi \rbrace \rceil$ for every $x \in \Sc_h$
        \STATE Take each action $x \in \Sc_{h}$ exactly $u_h(x)$ times with corresponding rewards $(\ty_j)_{j=1}^{u_h}$ where $u_h = \sum_{x \in \Sc_{h}} u_h(x)$
        \STATE Estimate $\tmu^{(h)}(\cdot)$ and $\sigma^{(h)}(\cdot)$ according to \cref{eq:corrupted_mean} and \cref{eq:posterior_variance} using only the $u_h$ points from the current epoch.
        \STATE Update the active set of actions to:
        \begin{align} 
            \Xc_{h+1} \leftarrow& \Big\lbrace x \in \Xc_{h}: \tmu^{(h)}(x) + \big(\beta_h + \tfrac{C\sqrt{u_h}}{l_h \psi \lambda}\big)\sigma^{(h)}(x) \geq \nonumber \\
            & \max_{x \in \Xc_h} \; \tmu^{(h)}(x) - \big(\beta_h + \tfrac{C\sqrt{u_h}}{l_h \psi \lambda}\big)\sigma^{(h)}(x) \Big\rbrace \nonumber
        \end{align}
        \STATE Set $l_{h+1} \leftarrow 2l_{h}$, $h \leftarrow h+1$ and return to Step 2 (terminating after $T$ total actions are played).
    \end{algorithmic}
\end{algorithm}

After the set $\Sc_h$ is constructed, we define $\xi_h(x)= \tfrac{\sum_{i=1}^{l_h} \bone \lbrace  x = x_i \rbrace}{l_h}$ for every $x \in \Sc_h$, representing the empirical frequency of selecting $x_t \in \Xc_h$ in $l_h$ rounds. The algorithm then plays actions from $\Sc_h$ only, where the number of times each action $x$ from $\Sc_h$ is played is denoted by $u_h(x) = \lceil l_{h} \max\lbrace \xi_h(x), \psi \rbrace \rceil$. Here, the \emph{truncation parameter} $\psi$ ensures that each action from $\Sc_h$ is played sufficiently many times; this idea was used for corrupted linear bandits by \citet{bogunovic2021stochastic}. Our theory suggests a particular choice of $\psi$; see \cref{theorem:main}. Each action $x \in \Sc_h$ is played for $u_h(x)$ times in an arbitrary order, leading to the total epoch length $u_h = \sum_{x \in \Sc_h} u_h(x)$.  

Based on the received noisy and potentially corrupted rewards $\lbrace x_j, \ty_j  \rbrace_{j=1}^{u_h}$, the algorithm updates its estimates $\tmu^{(h)}(\cdot)$ and $\sigma^{(h)}(\cdot)$ according to \cref{eq:corrupted_mean} and \cref{eq:posterior_variance}. Finally, each epoch $h$ ends by updating the set of active actions $\Xc_{h+1}$. To do so, we use the confidence bounds from \cref{lemma:corrupted_confidence_bounds} with $u_{\min} = l_h \psi$, where $l_h \psi$ is a lower bound on the number of times each distinct action from $\Sc_h$ is played. These confidence bounds are valid in the sense that the true function is contained within the confidence bounds with high probability. The definition of $\Xc_{h+1}$ (Line 15) 
ensures that with high probability, the optimal action is never eliminated.\looseness=-1

Besides the standard exploration/exploitation trade-off (controlled via $\beta_h$), our algorithm additionally balances robustness to corruptions. This is done via two parameters: the switching parameter $\eta$ and truncation parameter $\psi$. We set these parameter to ensure that the number of distinct actions played per epoch is sufficiently small, while the number of plays per each such action is sufficiently large. This trade-off is non-trivial; for example, in the case that $C=0$ (i.e., the non-corrupted setting), resampling the same actions (controlled via $\psi$) increases the regret.\looseness=-1   % Next, we present our main theoretical result and discuss how these parameters are set in our algorithm.\looseness=-1  

\textbf{Main result.} We now present our main theoretical result, where we use $\Otilde(\cdot)$ notation to hide constants and dimension-independent log factors.  We treat the RKHS norm bound $B$ as being fixed, so its dependence is also hidden in $O(\cdot)$ or $\Otilde(\cdot)$ notation.

\begin{restatable}[Main result]{theorem}{mainthm}\label{theorem:main}
    Under the preceding setup and \cref{assumption:regular_confidence_bounds}, for any corruption budget $C \geq 0$, \cref{alg:cpe} with a constant switching parameter $\eta > 1$ and truncation parameter $\psi = \tfrac{\ln \eta}{2 \gamma_{T}}$ satisfies the following with probability at least $1-\delta$:
\begin{equation}
    R_T = \Otilde\big(\beta_{\Hbar} \sqrt{T \gamma_{T}} + C \gamma_{T}^{3/2}\big).
\end{equation}
\end{restatable}

\subsection{Applications to Specific Confidence Bounds} \label{sec:specific}

Now we discuss specific choices of $\beta_h$ satisfying \cref{assumption:regular_confidence_bounds}, and the resulting final regret bounds. 

We observe that the actions in each fixed epoch are sampled non-adaptively, and the resulting GP posterior formed only depends on the points in that epoch.  As noted by \citet{li2021gaussian}, these conditions are sufficient to make use of the following confidence bounds for non-adaptive sampling.

\begin{lemma}[\citet{vakili2021optimal}, Theorem 1]
%	Let $\mu_t(x)$ and $\sigma_t(x)^2$ denote the posterior mean and variance based on any $t$ non-corrupted observations $\lbrace (x_i,y_i) \rbrace_{i=1}^t$ for any $t\geq 1$.
	When $\lbrace x_i \rbrace_{i=1}^t$ are selected independently of all the observations $\lbrace y_i \rbrace_{i=1}^t$, it holds for any fixed $x\in\Xc$ and any $t\geq 1$ with probability at least $1-\delta$ that $|\mu_t(x)-f(x)|\leq \big( B+\frac{\sigma}{\sqrt{\lambda}}\sqrt{2\log\frac{1}{\delta}}\big) \sigma_t(x)$.
\end{lemma}

For finite domains, applying the union bound leads to a choice of $\beta_h$ for the proposed algorithm such that $\beta_{\bar H}$ only contributes to logarithmic terms in the cumulative regret.

\begin{corollary}\label{corollary:main}
    Defining $\bar{\beta}_h(\delta) = B+\frac{\sigma}{\sqrt{\lambda}}\sqrt{2\log\frac{|\Xc|}{\delta}}$, we have that \cref{assumption:regular_confidence_bounds} holds with $\beta_h=\bar{\beta}_h(\delta_h)$ and $\delta_h=\frac{6\delta}{(h+1)^2\pi^2}$ . Hence, with probability at least $1-\delta$, \cref{alg:cpe} with switching parameter $\eta > 1$,  truncation parameter $\psi = \tfrac{\ln \eta}{2 \gamma_{T}}$, and $\beta_h$ as above achieves
\begin{equation}
    R_T = \Otilde\big(\sqrt{T \gamma_{T}} + C \gamma_{T}^{3/2}\big). \label{eq:regret_main}
\end{equation}
\end{corollary}

This corollary is obtained by noting that the error probability is at most $\delta$ as desired, since a union bound over $\Xc$ gives a per-epoch term of at most $\delta_h$, and $\sum_{h=0}^{H-1}\delta_h\leq\sum_{h=0}^\infty \frac{6\delta}{(h+1)^2\pi^2}=(\sum_{h=0}^\infty \frac{1}{(h+1)^2})\frac{6\delta}{\pi^2}\leq\frac{\pi^2}{6}\cdot\frac{6\delta}{\pi^2}=\delta$.

For general (possibly continuous) domains, one option is to set $\beta_h$ according to a widely-used confidence bound as follows, though we will shortly discuss improved choices.

\begin{lemma}[\citet{chowdhury17kernelized}, Theorem 2]
%	Let $\mu_t(x)$ and $\sigma_t(x)^2$ denote the posterior mean and variance based on any $t$ non-corrupted observations $\lbrace (x_i,y_i) \rbrace_{i=1}^t$ for any $t\geq 1$.
	For any (possibly adaptive) sampling strategy, it holds with probability at least $1-\delta$ that $|\mu_t(x)-f(x)|\leq \big(B+\sigma\sqrt{2(\gamma_t+1+\ln(1/\delta))}\big) \sigma_t(x)$ for all $x\in\Xc$ and $t\geq 1$.
\end{lemma}

By a similar argument to Corollary \ref{corollary:main} and the fact that $\gamma_{t}$ is increasing in $t$, we obtain the following.

\begin{corollary}
    If $u_h \le \bar{u}_h$ almost surely, then defining $\check{\beta}_h(\delta)=B+\sigma\sqrt{2(\gamma_{\bar{u}_h}+1+\ln(1/\delta))}$, we have that \cref{assumption:regular_confidence_bounds} holds with $\beta_h = \check{\beta}_h(\delta_h)$ and $\delta_h=\frac{6\delta}{(h+1)^2\pi^2}$.  Hence, with probability at least $1-\delta$, \cref{alg:cpe} with a constant switching parameter $\eta > 1$, truncation parameter $\psi = \tfrac{\ln \eta}{2 \gamma_{T}}$, and $\beta_h$ as above achieves
    \begin{equation}
        R_T = \Otilde\big(\sqrt{T}\gamma_{T} + C \gamma_{T}^{3/2}\big),
    \end{equation}
    where we crudely selected $\bar{u}_h = T$.
\end{corollary}
% As in the corollary above, the $\beta_{\Hbar}$ under this choice contributes an $\Otilde(\sqrt{\gamma_T})$ factor into $R_T$, making the uncorrupted term fail to sublinear for the Mat\'ern kernel with $\nu\leq d/2$.

While this regret bound can be significantly weaker than Corollary \ref{corollary:main} due to the $\Otilde(\sqrt{T}\gamma_{T})$ term, we can also obtain an analog of Corollary \ref{corollary:main} (i.e., attaining the improved dependence in Eq.~\eqref{eq:regret_main}) for continuous domains, under the mild assumption that functions in the RKHS are Lipschitz continuous (which is true for the kernels we consider below).  A crude approach is to have the algorithm use a very fine discretization \cite{janz2020bandit,li2021gaussian}, and a more sophisticated approach is to only discretize as part of the analysis \cite{vakili2021optimal}.  The details can be found in the preceding references, and we avoid repeating them.

% To handle continuous domain, we can construct a fine discretization following [\cite{vakili2021optimal}, Assumption 4] then apply our algorithm on it, leading to a regret bound with the same scaling as \cref{theorem:main} [\cite{vakili2021optimal}, Remark 2].

%[todo] Mention $B$ in setup

\subsection{Comparisons to Existing Bounds}

% {\color{blue} Here is a summary of things to consider comparing to (ignoring dimension-independent log factors):
%     \begin{itemize}
%         \item (Above result with $\beta_T$ ignored) $\sqrt{T (\log T)^d} + C(\log T)^{3d/2}$ for SE, $T^{\frac{\nu+d}{2\nu + d}} + C T^{\frac{3d}{4\nu+2d}}$ for Mat\'ern, $\sqrt{Td} + C d^{3/2}$ for linear.
%         \item (Existing) \cite{bogunovic2020corruption} $C \sqrt{T (\log T)^d}$ for SE, $C T^{\frac{\nu+d}{2\nu + d}} $ for Mat\'ern.
%         \item (Reduction to linear, Appendix \ref{sec:linear}) $\sqrt{T(\log T)^d} + C(\log T)^{3d/2}$ for SE, \eqref{eq:lin_mat1}--\eqref{eq:lin_mat2} for Mat\'ern (worse).
%         \item (Lower bound) \cite{cai2020lower}: $C (\log T)^{d/2}$ for SE, $C^{\frac{\nu}{d+\nu}} T^{\frac{d}{d+\nu}}$ for Mat\'ern.
%         \item (Linear) \cite{bogunovic2021stochastic} $\sqrt{Td} + C d^{3/2}$ for linear.
%     \end{itemize}
% }
% =======
\begin{table*}[!t]
    \begin{center}
        \begin{tabular}{|c|c|c|c|}
            \hline 
            \textbf{Kernel} & \textbf{Lower Bound} & \textbf{Existing} & \textbf{Ours} \tabularnewline
            \hline 
            \hline 
            Linear & $\sqrt{Td}+Cd$ & $\ensuremath{\sqrt{Td}+Cd^{3/2}}$ & $\ensuremath{\sqrt{Td}+Cd^{3/2}}$\tabularnewline
            \hline 
            SE & $\sqrt{T(\log T)^{d/2}}+C(\log T)^{d/2}$ & $\sqrt{T}(\log T)^{d}+C \sqrt{T} (\log T)^{d/2}$ & $\sqrt{T(\log T)^{d}}+C(\log T)^{3d/2}$\tabularnewline
            \hline 
            Mat\'ern & $T^{\frac{\nu+d}{2\nu+d}}+C^{\frac{\nu}{d+\nu}}T^{\frac{d}{d+\nu}}$ & $T^{\frac{2\nu+3d}{4\nu+2d}}+CT^{\frac{\nu+d}{2\nu+d}}$ & $T^{\frac{\nu+d}{2\nu+d}}+CT^{\frac{3d}{4\nu+2d}}$\tabularnewline
            \hline 
        \end{tabular}
        \caption{Summary of regret bounds with constants and dimension-independent log factors omitted.  For the SE and Mat\'ern kernels, the upper bounds are from \citet{bogunovic2020corruption} and the lower bounds are from \citet{cai2020lower}.  For the linear kernel, the existing bounds are from \citet{bogunovic2021stochastic}, except the $\sqrt{Td}$ lower bound which is from \citet{dani2008stochastic}. \label{tbl:summary}}
    \end{center}
\end{table*}

We specialize our regret bound in Eq.~\eqref{eq:regret_main} to specific kernels by substituting $\gamma_T = \Otilde(d)$ for the linear kernel, $\gamma_T = \Otilde( (\log T)^d )$ for the SE kernel, and $\gamma_T = \Otilde(T^{\frac{d}{2\nu + d}})$ for the Mat\'ern kernel \cite{srinivas2009gaussian}.  The resulting regret bounds are shown in Table \ref{tbl:summary} (omitting constants and dimension-independent log factors), along with the best known existing upper and lower bounds.  We observe the following:\looseness=-1
\begin{itemize}
    \item For the linear kernel, we recover the recent upper bound of \citet{bogunovic2021stochastic}, and this is tight up to the presence of $d$ vs.~$d^{3/2}$ in the corrupted part.
    \item For the SE kernel, we match the lower bound of \citet{cai2020lower} up to small changes in the implied constant in each $(\log T)^{\Theta(d)}$ term.  In contrast, the existing upper bound of \citet{bogunovic2020corruption} incurs a much larger $\sqrt{T}$ term in the corrupted part.
    \item For the Mat\'ern kernel, compared to the existing result by \citet{bogunovic2020corruption}, we obtain an improvement in the non-corrupted part recently established by \citet{li2021gaussian}, matching the non-corrupted lower bound.  In the corrupted part, the existing result has a better exponent to $T$ when $\nu < \frac{d}{2}$, whereas ours is better when $\nu > \frac{d}{2}$, in particular approaching zero (instead of $\frac{1}{2}$) as $\nu \to \infty$ and nearly matching the lower bound in this limit.  However, when $\nu < \frac{d}{2}$ we find that the non-corrupted part in \cite{bogunovic2020corruption} is super-linear in $T$, making the bound trivial.  Hence, our bound is better whenever non-trivial scaling is attained.
\end{itemize}
The bounds based on a reduction to linear bandits, which we derive in Appendix \ref{sec:linear}, are omitted in Table \ref{tbl:summary}.  We briefly note that they are able to provide a similar upper bound to our main one under the SE kernel, but are always strictly worse under the Mat\'ern kernel.

%\todo{Comments on the obtained result:
%    \begin{itemize}
%        \item Comment on the regret rates for different kernels. 
%        \item Comment on the known lower bounds. 
%        \item Constrast to our previous kernelized results. We also improve the uncorrupted regret in comparison to our previous UCB algorithm.
%        \item Perhaps also comment on the reduction to the linear setting, and constrast with our linear bandit algorithm.  
%    \end{itemize}
%}

%{\color{blue} Here is a summary of things to consider comparing to (ignoring dimension-independent log factors):
%    \begin{itemize}
%        \item (Above result with $\beta_T$ ignored) $\sqrt{T (\log T)^d} + C(\log T)^{3d/2}$ for SE, $T^{\frac{\nu+d}{2\nu + d}} + C T^{\frac{3d}{4\nu+2d}}$ for Mat\'ern, $\sqrt{Td} + C d^{3/2}$ for linear.
%        \item (Existing) \cite{bogunovic2020corruption} $C \sqrt{T (\log T)^d}$ for SE, $C T^{\frac{\nu+d}{2\nu + d}} $ for Mat\'ern.
%        \item (Reduction to linear, Appendix \ref{sec:linear}) $\sqrt{T(\log T)^d} + C(\log T)^{3d/2}$ for SE, \eqref{eq:lin_mat1}--\eqref{eq:lin_mat2} for Mat\'ern (worse).
%        \item (Lower bound) \cite{cai2020lower}: $C (\log T)^{d/2}$ for SE, $C^{\frac{\nu}{d+\nu}} T^{\frac{d}{d+\nu}}$ for Mat\'ern.
%        \item (Linear) \cite{bogunovic2021stochastic} $\sqrt{Td} + C d^{3/2}$ for linear.
%    \end{itemize}
%}

\section{Experiments}

In this section, we experimentally evaluate the performance of our proposed algorithm, along with two baselines, one robust and one non-robust.  Our experiments serve as a proof of concept for our proposed approach, but also highlight possible remaining gaps between theory and practice, e.g., arising from large constant factors in the regret bounds.  We emphasize that our contributions are primarily theoretical.

\subsection{Algorithms}
We consider the following three algorithms:
\begin{enumerate}
    \item RGP-PE: Robust GP-Phased Elimination with constant $\beta_h$; this is a slight variation of \cref{corollary:main} in view of the fact that the number of epochs $H$ turns out to be a small constant in our experiments.
    \item GP-UCB: a representative non-robust fully sequential algorithm with slowly growing $\beta_t$, where $t\in[T]$ \citep[][Algorithm 1]{srinivas2009gaussian}.
    \item RGP-UCB: the robust version of GP-UCB with slowly growing $\beta_t$ \citep[][Algorithm 1]{bogunovic2020corruption}, where the only difference from GP-UCB is that the theoretical coefficient of $\sigma_{t-1}$ in the UCB is $\beta_t+\frac{C}{\sqrt{\lambda}}$.
\end{enumerate}
% We note that our algorithm uses a constant value of $\beta_h$ in accordance with our theory, whereas GP-UCB uses a slowly growing value $\beta_t$ in accordance with extensive prior works.
We found the term $\beta_h + \tfrac{C\sqrt{u_h}}{l_h \psi \lambda}$ multiplying $\sigma^{(h)}$ in Algorithm \ref{alg:cpe} to be overly conservative, so we instead replace it by $\beta_h+b\cdot\frac{C}{\sqrt{u_h}}$ (since $l_h$ and $u_h$ are similar, we replace $\frac{\sqrt{u_h}}{l_h}$ by $\frac{1}{\sqrt{u_h}}$), where $b\in(0,1]$ is an additional parameter controlling the degree of exploration and robustness. Similarly, in RGP-UCB we use the coefficient $\beta_t+b\cdot\frac{C}{\sqrt{\lambda}}$.  The remaining parameters $\beta_h$ and $\beta_t$ are specified below.

\subsection{Functions}

\subsubsection{Synthetic Function}
 
We produce a synthetic 2D function $f_1$, shown in Figure~ \ref{fig:f1}, which is randomly sampled from a Gaussian Process with zero mean and the SE kernel with lengthscale $l=0.5$. The domain $\Xc$ of $f_1$ contains 100 points obtained by evenly splitting $[-5,5]^2$ into a $10\times10$ grid. We use the true kernel as the prior for all three algorithms, and use $\beta_h=4$ for RGP-PE, and $\beta_t=\sqrt{\log t}/2$ for GP-UCB and RGP-UCB.

\begin{figure}[!t]
    \centering
    \minipage[c]{0.25\textwidth}
    \includegraphics[width=\linewidth]{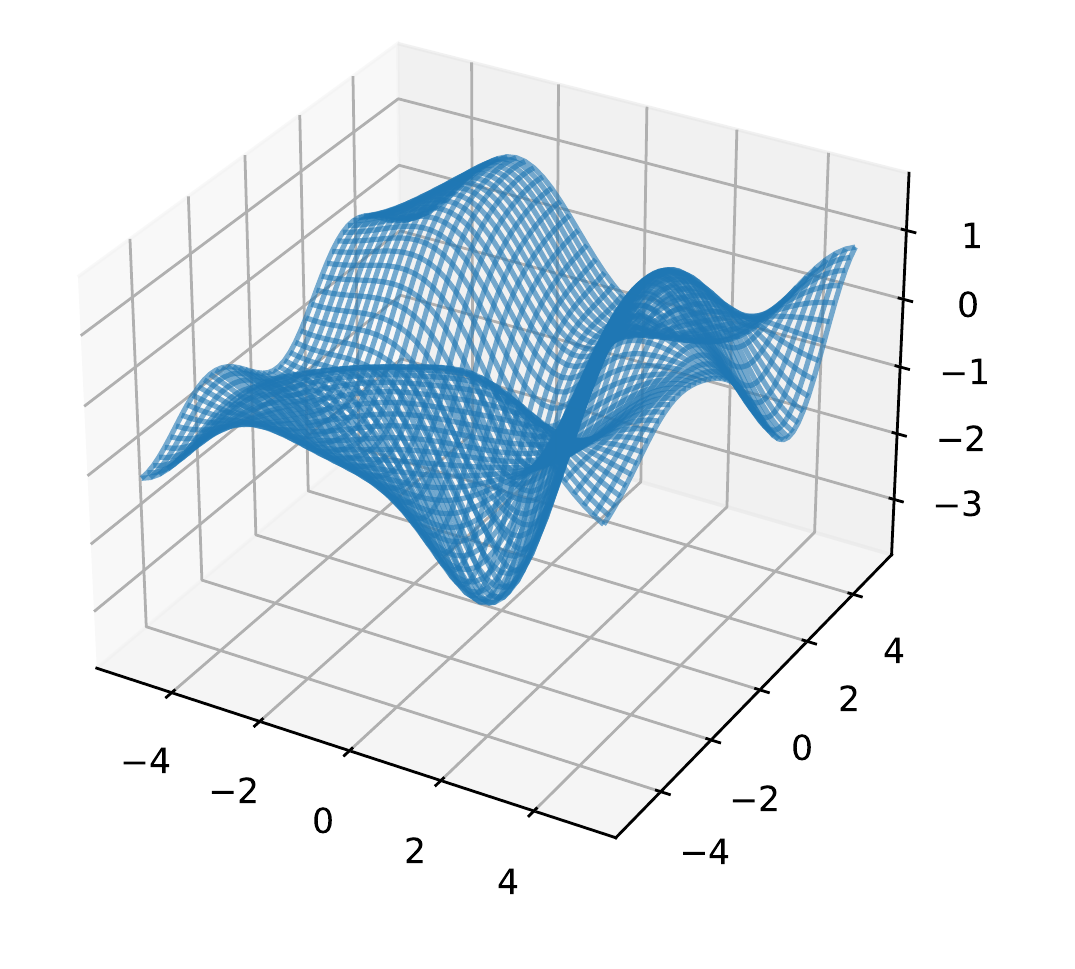}
    \endminipage
    \minipage[c]{0.25\textwidth}
    \includegraphics[width=\linewidth]{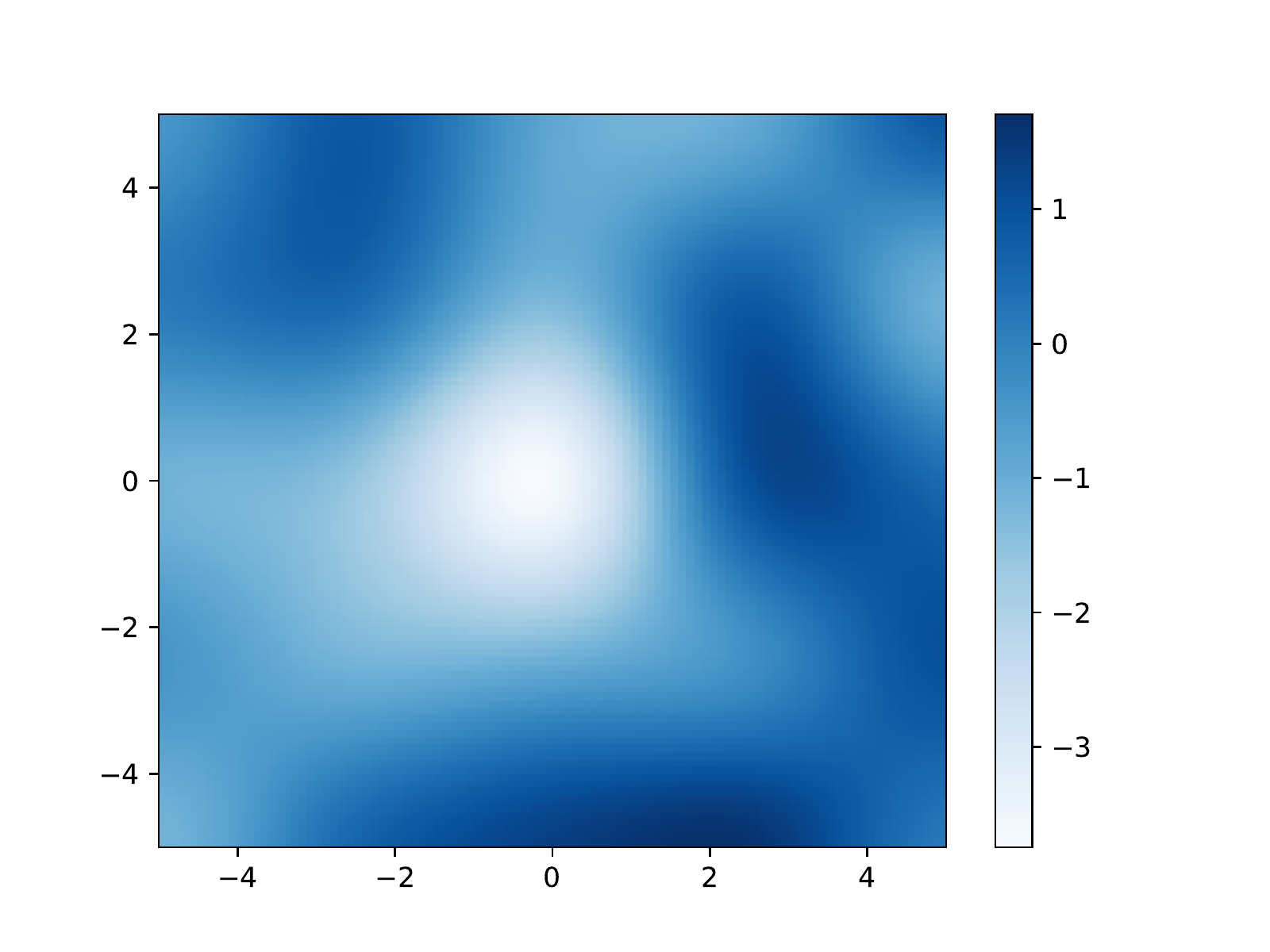}
    \endminipage
    \caption{Illustration of 2D synthetic function.}
    \label{fig:f1}
\end{figure}

\subsubsection{Robot Pushing Objective Function}

We consider the deterministic robot pushing objective function on a 2D plane introduced by \citet{wang2017max}, which aims to find suitable parameters to push an object to the target location $r_g$. We use the Robot3d function, which takes the robot location $(r_x,r_y)$ and pushing duration $t_r$ as a 3D input, and outputs the reversed distance between the pushed robot location and the target location $r_g$, i.e., $$\mathrm{Robot3D}(r_x,r_y,t_r)=5-\|\mathrm{push}(r_x,r_y,t_r)-r_g\|,$$ where $\mathrm{push}(\cdot)$ outputs the pushed robot location.

We let the domain $\Xc$ contain 100 points $(r_x,r_y,t_r)$ randomly sampled from $[-5,5]^2\times[1,30]$, and the target location $r_g$ is set to be $(3,2)$. Since the lengthscale of the SE kernel with maximum likelihood given the noiseless data is $1.94\approx 2$, we use the SE kernel with $l=2$ as prior for all three algorithms.  We found it beneficial for all algorithms to be slightly more explorative for this function, and accordingly use $\beta_h=6$ for KE and $\beta_t=2\sqrt{\log t}$ for GP-UCB and RGP-UCB.

\subsection{Attack Methods}

We consider the following five attack methods, which continue until the corruption budget is exhausted:
\begin{itemize}[topsep=0pt,itemsep=-1ex]
    \item Clipping: This attack proposed by \citet{han2021adversarial} adversarially perturbs $f$ and produces another reward function $\widetilde{f}$ whose optima are in some region $\Rc_{\text{target}}$ that does not contain $x^\ast$ by setting
    \begin{align*}
        \widetilde{f}(x)=\begin{cases}f(x)&x\in\Rc_{\text{target}},\\
            \min\{f(x),f(\widetilde{x}^\ast)-\Delta\}&x\not\in\Rc_{\text{target}},\end{cases}
    \end{align*}
    where $\widetilde{x}^\ast=\argmax_{x\in\Rc_{\text{target}}}f(x)$. We let $\Delta=0.5$ and choose $\Rc_{\text{target}} = \{(x_1,x_2)\in\Xc:x_1\leq x_2\}$ for the function $f_1$, and $\Rc_{\text{target}} = \{(r_x,r_y,t_r)\in\Xc:r_x\geq 0\}$ for the function Robot3D.
    \item Aggressive Subtraction (AggSub): This attack proposed by \citet{han2021adversarial} sets
    \begin{align*}
        \widetilde{f}(x)=\begin{cases}f(x)&x\in\Rc_{\text{target}},\\
            f(x)-h_\text{max}&x\not\in\Rc_{\text{target}},\end{cases}
    \end{align*}
    for some $h_\text{max}>f(x^\ast)-f(\widetilde{x}^\ast)$. We use the same $\Rc_{\text{target}}$ as the Clipping attack, and let $h_\text{max}=1$ for $f_1$ and $h_\text{max}=3$ for Robot3D.
    \item Top-$K$: When $x$ is one of the top $K$ remaining actions, this attack perturbs the reward down to $-1$.  We consider both $K=3$ and $K=5$.
    \item Flip: This attack simply flips the reward from $f(x)$ to $-f(x)$.  Both this attack and the previous one are simple variations of attacks considered for linear bandits by \citet{bogunovic2021stochastic}.
\end{itemize}
For all three algorithms, we consider the attack budgets $C=50$ and $C=100$.  By default, the attack starts at $t=1$, but for the robust algorithms RGP-PE and RGP-UCB, we also conduct experiments with a {\it later} attack, where (i) the attack in RGP-PE starts when at least one action is eliminated from the domain; and (ii) the attack in RGP-UCB starts when at least one action has UCB strictly lower than $\max_{x\in\Xc}\mathrm{LCB}(x)$.

\subsection{Hyperparameters and Trials}

We let $T=50000$, $\sigma=0.02$, and $\lambda=1$ for all three algorithms, $b=0.1$ for RGP-PE and RGP-UCB, and $\psi=0.5,\eta=2$ for RGP-PE.   The results are produced by performing $10$ trials and plotting the average cumulative regret, with error bars indicating one standard deviation.

\subsection{Results}

\begin{figure*}[t!]
    \centering
    \minipage[t]{0.33\textwidth}
    \includegraphics[width=\linewidth]{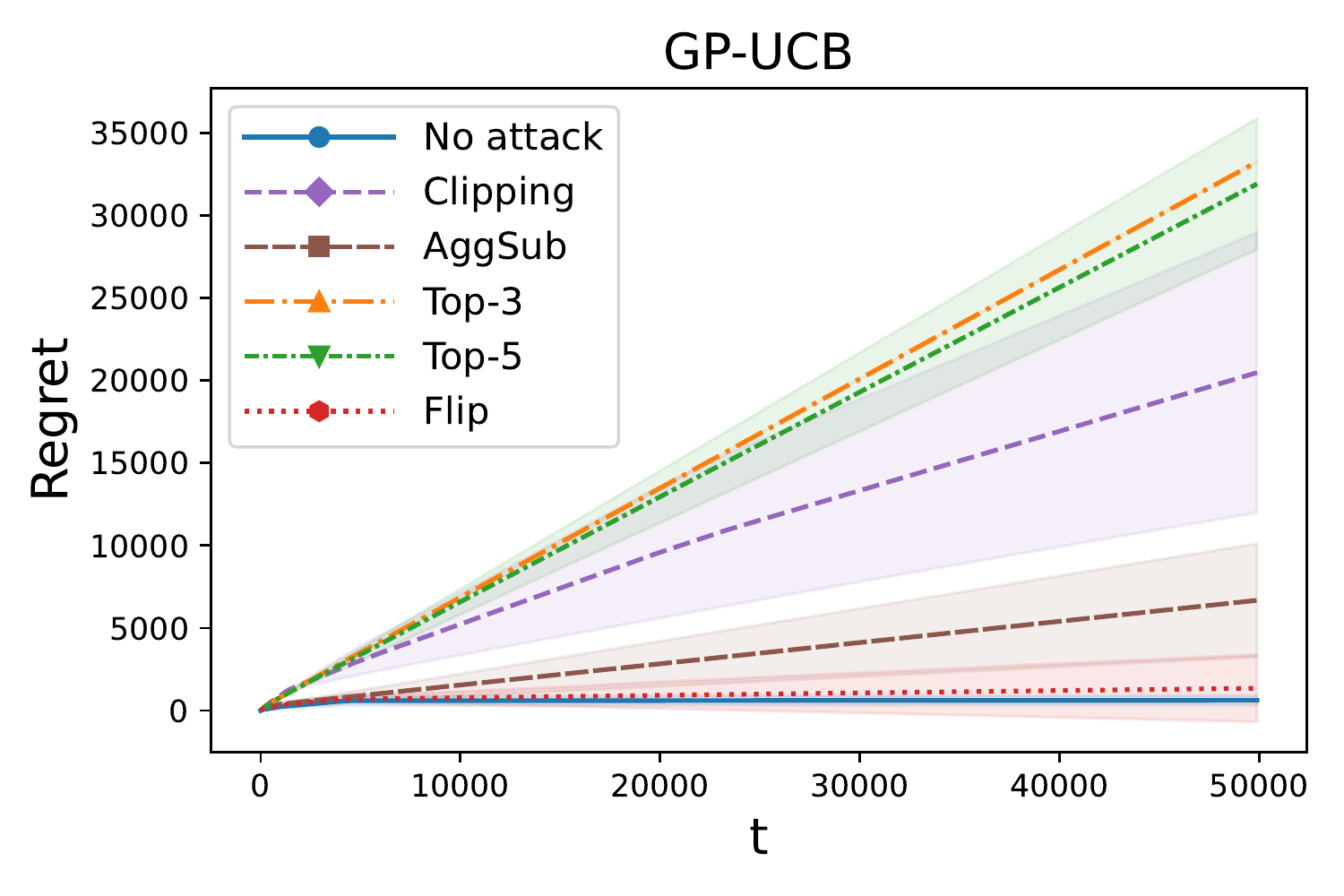}
    \endminipage
    \minipage[t]{0.33\textwidth}
    \includegraphics[width=\linewidth]{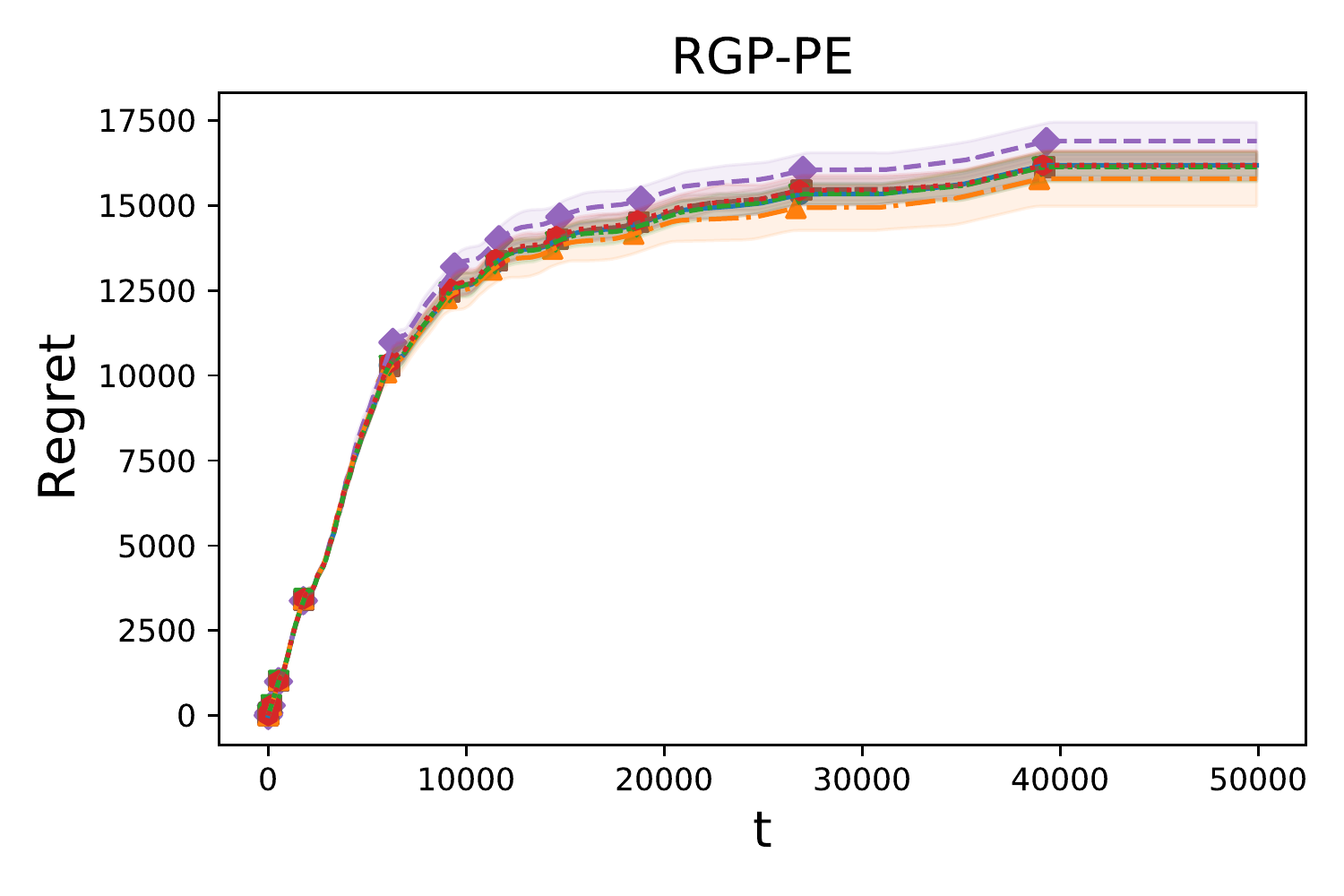}
    \endminipage
    \minipage[t]{0.33\textwidth}
    \includegraphics[width=\linewidth]{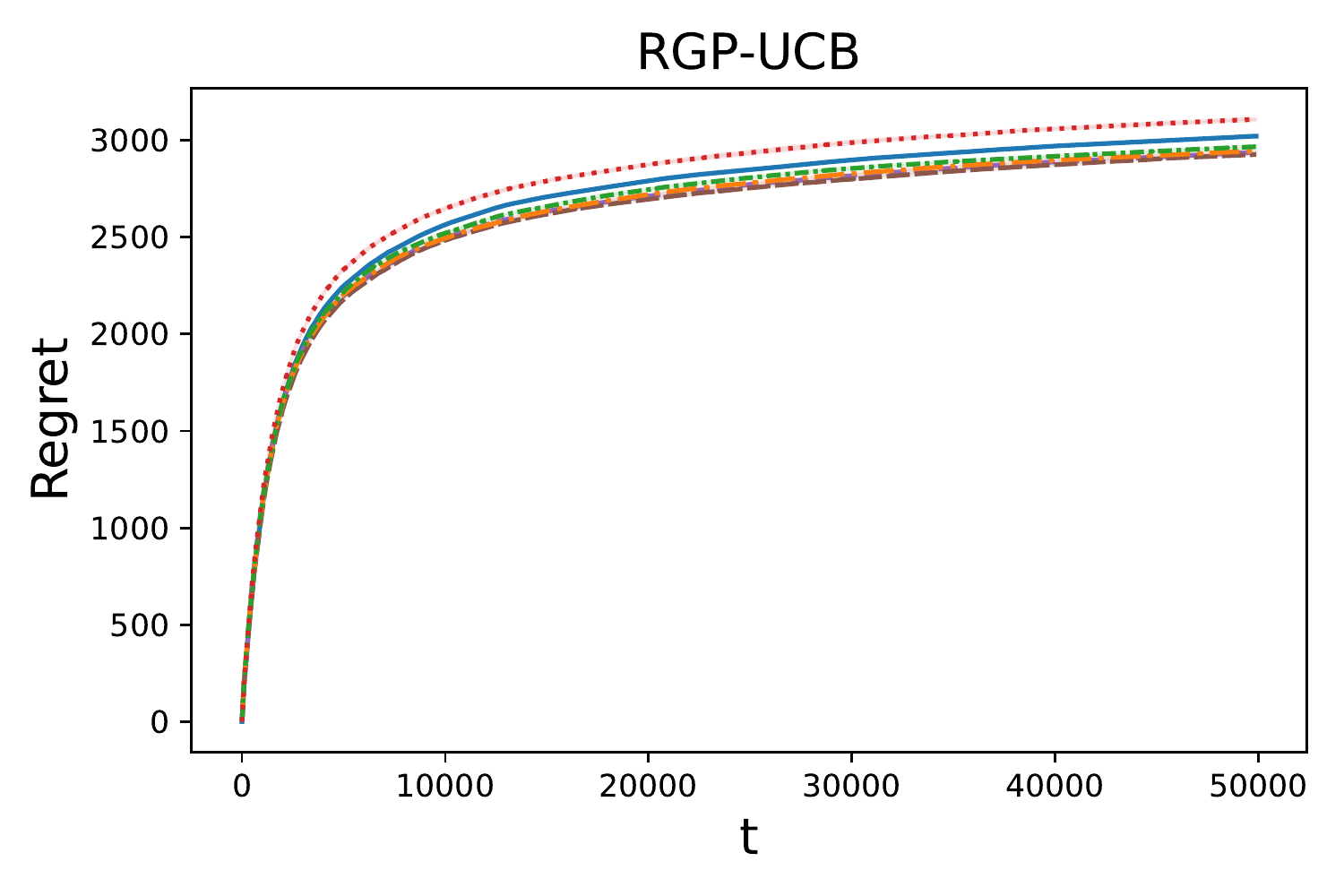}
    \endminipage
    \newline
    \minipage[t]{0.33\textwidth}
    \includegraphics[width=\linewidth]{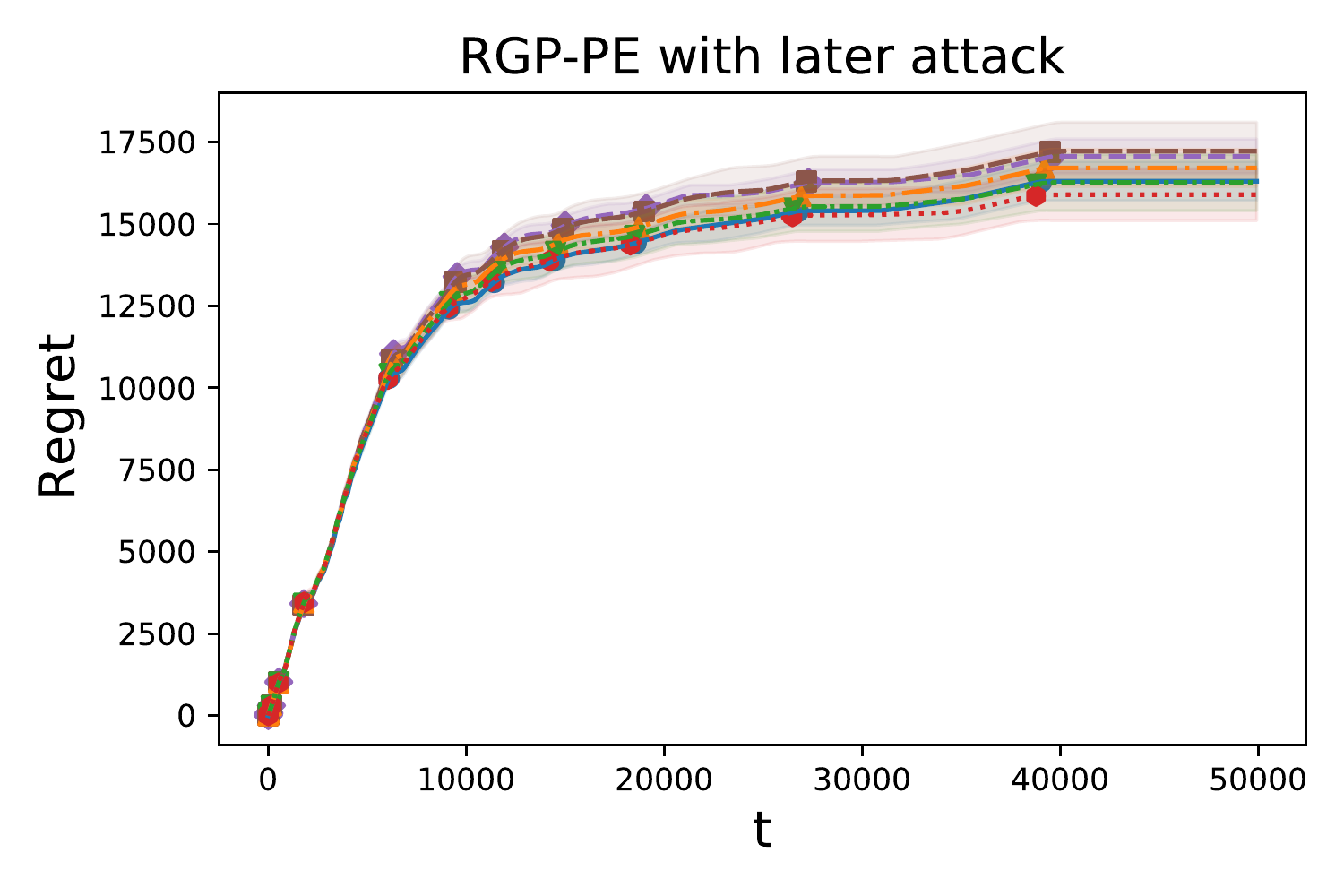}
    \endminipage
    \minipage[t]{0.33\textwidth}
    \includegraphics[width=\linewidth]{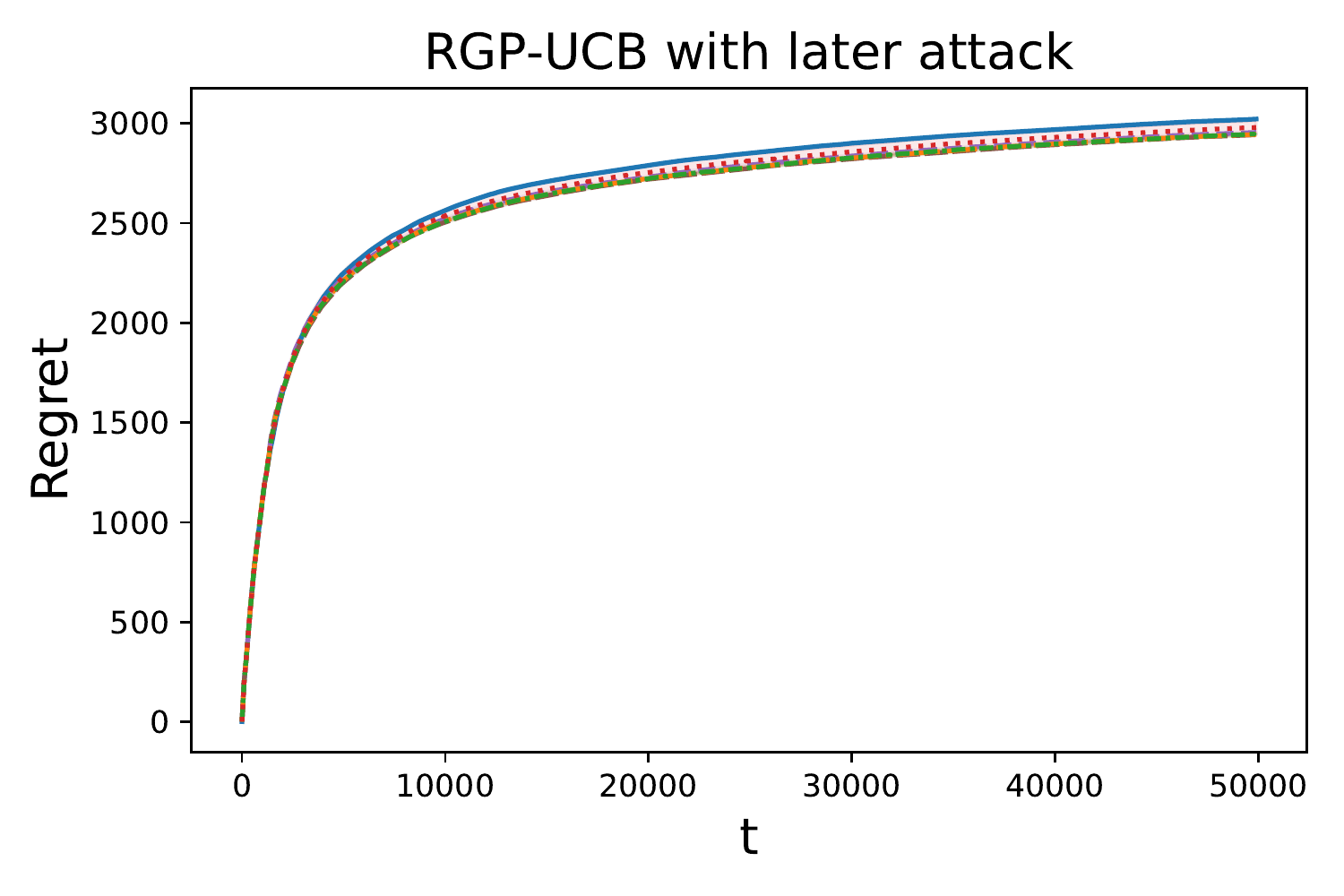}
    \endminipage
    \caption{Performance on $f_1$ with $C=50$.  We observe that GP-UCB incurs linear regret for several attacks, whereas the other algorithms exhibit robustness to all of the attacks.}
    \label{fig:f1_50}
\end{figure*}

\begin{figure*}[t!]
    \centering
    \minipage[t]{0.33\textwidth}
    \includegraphics[width=\linewidth]{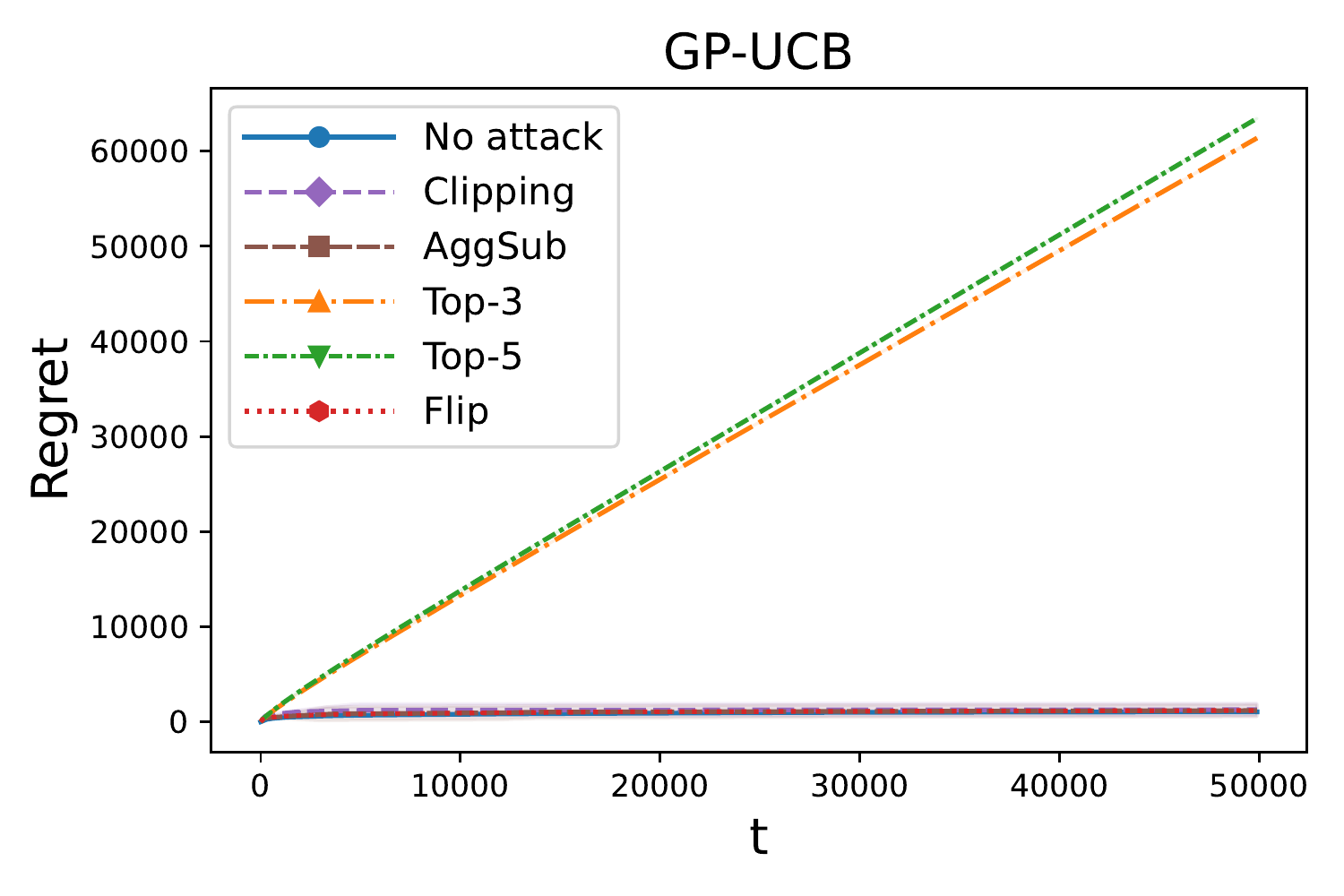}
    \endminipage
    \minipage[t]{0.33\textwidth}
    \includegraphics[width=\linewidth]{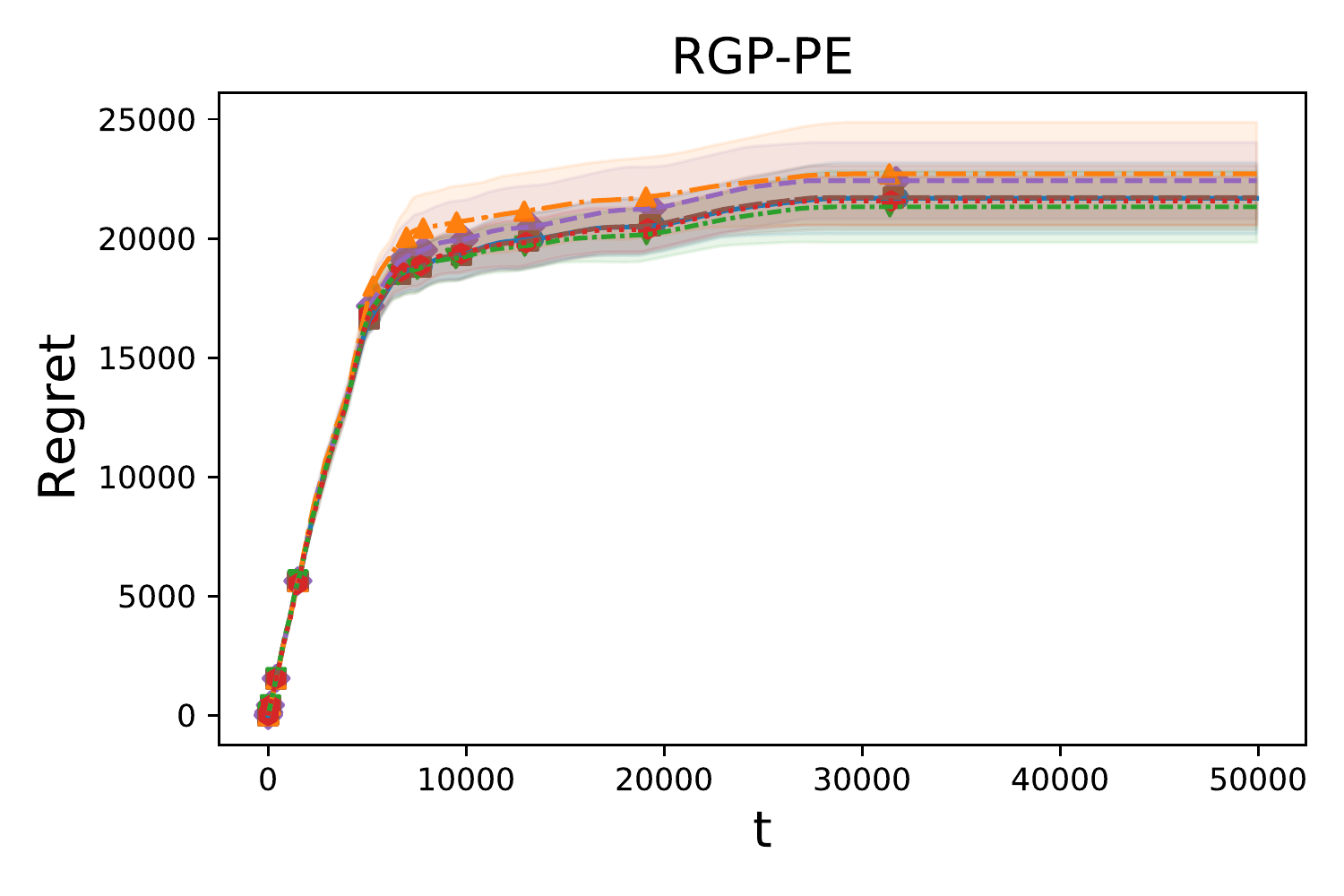}
    \endminipage
    \minipage[t]{0.33\textwidth}
    \includegraphics[width=\linewidth]{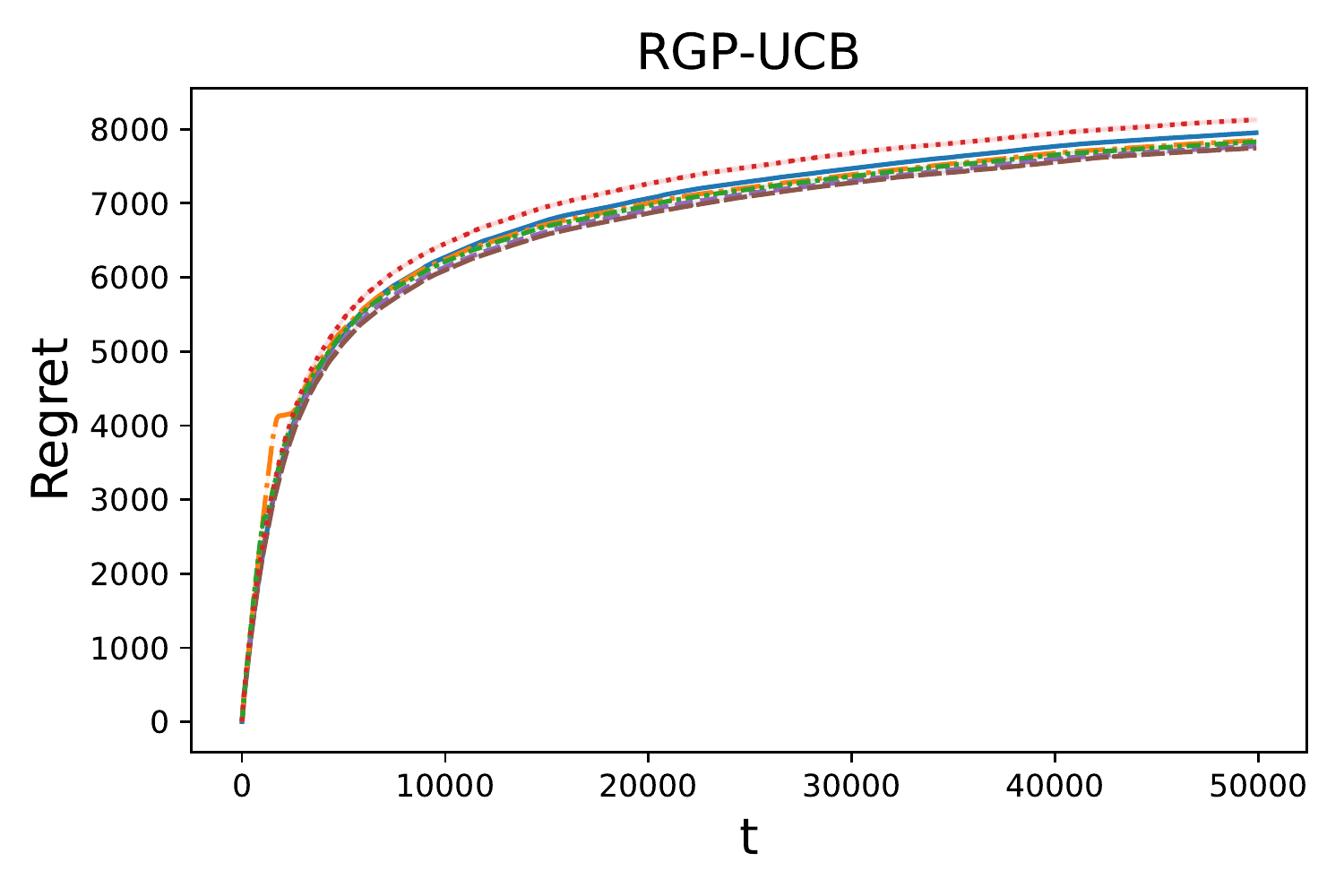}
    \endminipage
    \newline
    \minipage[t]{0.33\textwidth}
    \includegraphics[width=\linewidth]{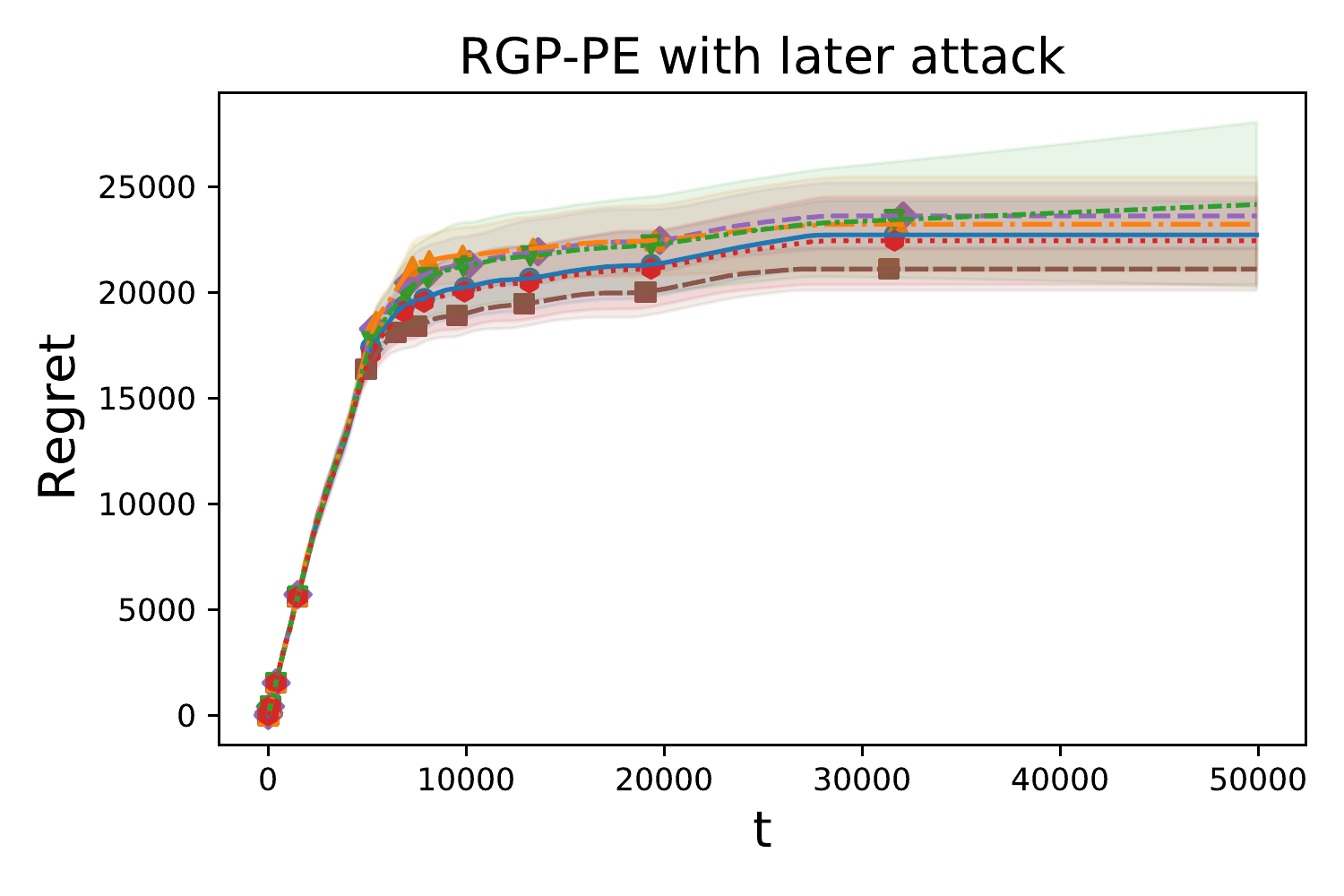}
    \endminipage
    \minipage[t]{0.33\textwidth}
    \includegraphics[width=\linewidth]{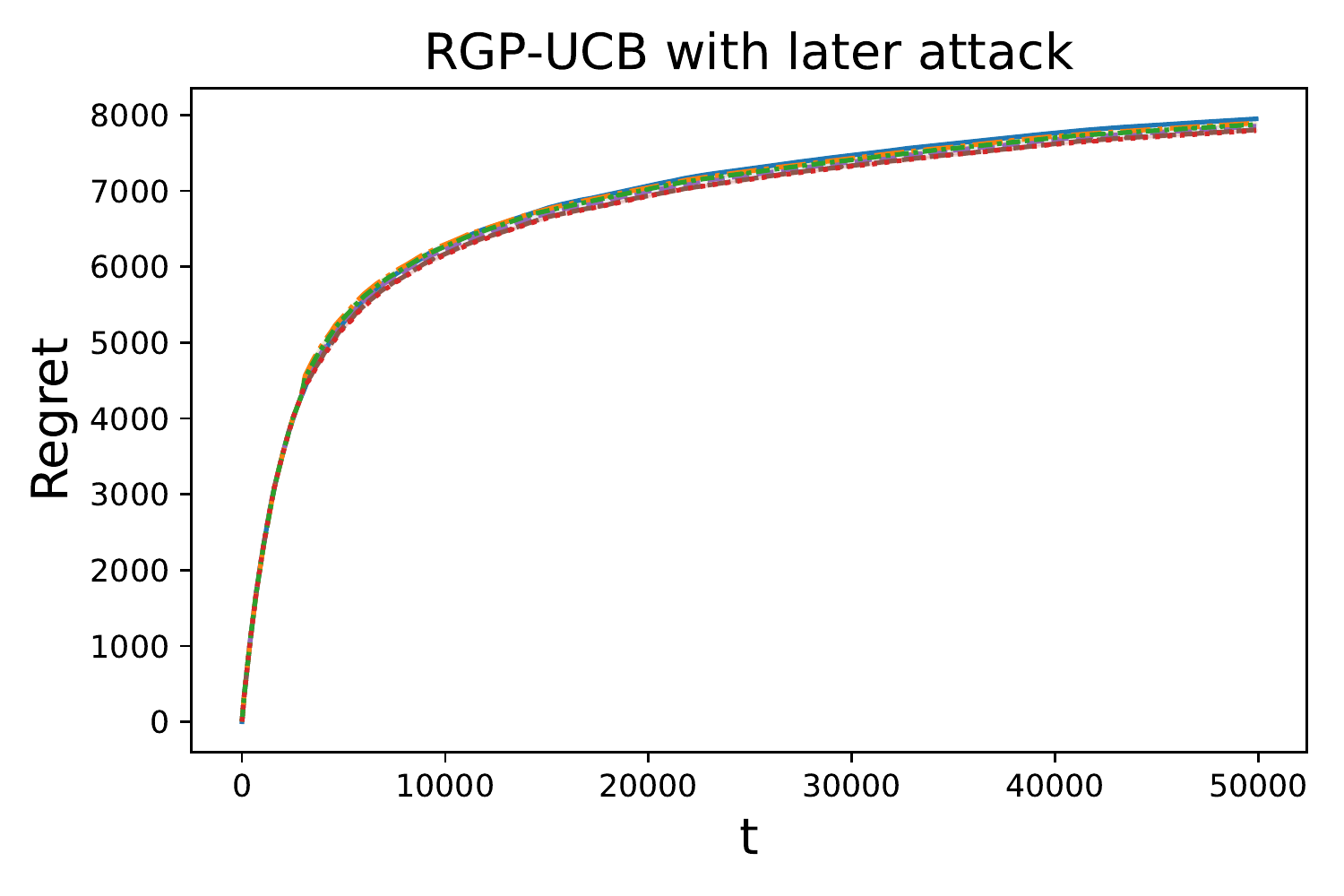}
    \endminipage
    \caption{Performance on Robot3D with $C=100$.  We observe that GP-UCB incurs linear regret for two attacks, whereas the other algorithms exhibit robustness to all of the attacks.}
    \vspace*{-1.5ex}
    \label{fig:robot3d_100}
\end{figure*}

\subsubsection{Comparison of Algorithms}

As shown in Figures \ref{fig:f1_50} and \ref{fig:robot3d_100}, the non-robust algorithm GP-UCB succeeds when no attack is applied. However, the cumulative regret for $f_1$ associated with the Clipping, AggSub, Top-3, and Top-5 attacks grow linearly, indicating that these four attacks succeed in driving GP-UCB towards a suboptimal action.  Similarly, the Top-3 and Top-5 attacks incur linear regret for Robot3D.  In contrast, we find that RGP-PE has only one action remaining at the end of the 13th epoch, and manages to defend against all five attack methods for both functions. 

The baseline robust algorithm RGP-UCB also successfully defends against all the attacks, and generally has lower cumulative regret than RGP-PE, despite RGP-PE having a stronger regret guarantee.  There are at least two possibly reasons for this: (i) The analysis of RGP-UCB of \citet{bogunovic2020corruption} could be loose, with a tighter analysis potentially giving an additive dependence similar to Theorem \ref{theorem:main}, and (ii) the strong scaling laws in our theory may still leave room for significant improvements in the constant factors (or logarithmic, etc.).  Further addressing these findings remains an interesting direction for future work.

We note that even in the more specialized problem of corrupted stochastic linear bandits, analogous practical limitations of a phased elimination algorithm were observed by \citet{bogunovic2021stochastic}.

\subsubsection{Later Attack}

We observe that RGP-PE and RGP-UCB are also able to defend against the later attack, and their performance is similar to when the attack starts from the beginning.  There are only two trials of RGP-PE (budget $C=100$ and Top-5 attack on Robot3D in Figure \ref{fig:robot3d_100}), in which the only action remaining at the end of the 13th epoch is slightly suboptimal. In Appendix \ref{sec:supp_exp} in the supplementary material, we additionally show the experiment results for $f_1$ with $C = 100$, and Robot3D with $C=50$.
% There is only one particular trial of KE with budget $C=100$ and the Flip attack on $f_1$ in Figure \ref{fig:KE_100_1}, that has two actions remaining at the end of the 13th epoch, and proceeds with the last epoch until the time horizon is reached.

% \subsubsection*{Acknowledgements}
% All acknowledgments go at the end of the paper, including thanks to reviewers who gave useful comments, to colleagues who contributed to the ideas, and to funding agencies and corporate sponsors that provided financial support. 
% To preserve the anonymity, please include acknowledgments \emph{only} in the camera-ready papers.

\section{Conclusion}

We have provided a new algorithm for corruption-tolerant Gaussian process bandits based on phased elimination, incorporating a key idea of {\em rare switching} based on a certain condition relating to the information gain, along with a robust estimator, enlarged confidence bounds, and truncation to ensure a minimal number of plays of each selected action.  Our regret bound recovers the best known existing bound under the linear kernel, is provably near-optimal under the SE kernel, and improves on the best existing bound in all cases where the latter is non-trivial.  Perhaps the most immediate direction for future work is to establish to what extent the $C\gamma_T^{3/2}$ dependence can be further improved, particularly in the case of the Mat\'ern kernel.

\section*{Acknowledgment}
This project has received funding from the European Research Council (ERC) under the European Unions Horizon 2020 research and innovation programme grant agreement No 815943. J.~Scarlett was supported by the Singapore National Research Foundation (NRF) under grant number R-252-000-A74-281.

% \newpage

%\bibliographystyle{myIEEEtran}
\bibliographystyle{icml2022}
\bibliography{refs}

\begin{thebibliography}{45}
\providecommand{\natexlab}[1]{#1}
\providecommand{\url}[1]{\texttt{#1}}
\expandafter\ifx\csname urlstyle\endcsname\relax
  \providecommand{\doi}[1]{doi: #1}\else
  \providecommand{\doi}{doi: \begingroup \urlstyle{rm}\Url}\fi

\bibitem[Abbasi-Yadkori et~al.(2011{\natexlab{a}})Abbasi-Yadkori, P{\'a}l, and
  Szepesv{\'a}ri]{abbasi2011improved}
Abbasi-Yadkori, Y., P{\'a}l, D., and Szepesv{\'a}ri, C.
\newblock Improved algorithms for linear stochastic bandits.
\newblock \emph{Conference on Neural Information Processing Systems},
  2011{\natexlab{a}}.

\bibitem[Abbasi-Yadkori et~al.(2011{\natexlab{b}})Abbasi-Yadkori, P\'{a}l, and
  Szepesv\'{a}ri]{abbasi2011linear}
Abbasi-Yadkori, Y., P\'{a}l, D., and Szepesv\'{a}ri, C.
\newblock Improved algorithms for linear stochastic bandits.
\newblock pp.\  2312--2320, 2011{\natexlab{b}}.

\bibitem[Banihashem et~al.(2021)Banihashem, Singla, and
  Radanovic]{banihashem2021defense}
Banihashem, K., Singla, A., and Radanovic, G.
\newblock Defense against reward poisoning attacks in reinforcement learning.
\newblock \emph{arXiv preprint arXiv:2102.05776}, 2021.

\bibitem[Beland \& Nair(2017)Beland and Nair]{beland2017bayes}
Beland, J.~J. and Nair, P.~B.
\newblock Bayesian optimization under uncertainty.
\newblock {NIPS} BayesOpt 2017 workshop, 2017.

\bibitem[Bogunovic \& Krause(2021)Bogunovic and
  Krause]{bogunovic2021misspecified}
Bogunovic, I. and Krause, A.
\newblock Misspecified {G}aussian process bandit optimization.
\newblock \emph{Conference on Neural Information Processing Systems}, 34, 2021.

\bibitem[Bogunovic et~al.(2018)Bogunovic, Scarlett, Jegelka, and
  Cevher]{bogunovic2018adversarially}
Bogunovic, I., Scarlett, J., Jegelka, S., and Cevher, V.
\newblock Adversarially robust optimization with {G}aussian processes.
\newblock In \emph{Advances in Neural Information Processing Systems
  (NeurIPS)}, pp.\  5760--5770, 2018.

\bibitem[Bogunovic et~al.(2020)Bogunovic, Krause, and
  Jonathan]{bogunovic2020corruption}
Bogunovic, I., Krause, A., and Jonathan, S.
\newblock Corruption-tolerant {G}aussian process bandit optimization.
\newblock In \emph{International Conference on Artificial Intelligence and
  Statistics (AISTATS)}, 2020.

\bibitem[Bogunovic et~al.(2021)Bogunovic, Losalka, Krause, and
  Scarlett]{bogunovic2021stochastic}
Bogunovic, I., Losalka, A., Krause, A., and Scarlett, J.
\newblock Stochastic linear bandits robust to adversarial attacks.
\newblock In \emph{International Conference on Artificial Intelligence and
  Statistics}, pp.\  991--999, 2021.

\bibitem[Cai \& Scarlett(2021)Cai and Scarlett]{cai2020lower}
Cai, X. and Scarlett, J.
\newblock On lower bounds for standard and robust {G}aussian process bandit
  optimization.
\newblock In \emph{International Conference on Machine Learning}, 2021.

\bibitem[Cakmak et~al.(2020)Cakmak, Astudillo, Frazier, and
  Zhou]{cakmak2020bayesian}
Cakmak, S., Astudillo, R., Frazier, P., and Zhou, E.
\newblock Bayesian optimization of risk measures.
\newblock \emph{arXiv preprint arXiv:2007.05554}, 2020.

\bibitem[Camilleri et~al.(2021)Camilleri, Jamieson, and
  Katz-Samuels]{camilleri2021high}
Camilleri, R., Jamieson, K., and Katz-Samuels, J.
\newblock High-dimensional experimental design and kernel bandits.
\newblock In \emph{International Conference on Machine Learning}, 2021.

\bibitem[Chen et~al.(2021)Chen, Du, and Jamieson]{chen2021corruption}
Chen, Y., Du, S.~S., and Jamieson, K.
\newblock Corruption robust active learning.
\newblock \emph{arXiv preprint arXiv:2106.11220}, 2021.

\bibitem[Chowdhury \& Gopalan(2017)Chowdhury and
  Gopalan]{chowdhury17kernelized}
Chowdhury, S.~R. and Gopalan, A.
\newblock On kernelized multi-armed bandits.
\newblock In \emph{International Conference on Machine Learning (ICML)}, pp.\
  844--853, 2017.

\bibitem[Dai~Nguyen et~al.(2017)Dai~Nguyen, Gupta, Rana, and
  Venkatesh]{dai2017stable}
Dai~Nguyen, T., Gupta, S., Rana, S., and Venkatesh, S.
\newblock Stable {B}ayesian optimization.
\newblock In \emph{Pacific-Asia Conference on Knowledge Discovery and Data
  Mining}, pp.\  578--591. Springer, 2017.

\bibitem[Dani et~al.(2008)Dani, Hayes, and Kakade]{dani2008stochastic}
Dani, V., Hayes, T.~P., and Kakade, S.~M.
\newblock Stochastic linear optimization under bandit feedback.
\newblock In \emph{Conference on Learning Theory}, 2008.

\bibitem[Durand et~al.(2018)Durand, Maillard, and Pineau]{durand2018streaming}
Durand, A., Maillard, O.-A., and Pineau, J.
\newblock Streaming kernel regression with provably adaptive mean, variance,
  and regularization.
\newblock \emph{The Journal of Machine Learning Research}, 19\penalty0
  (1):\penalty0 650--683, 2018.

\bibitem[Griffiths \& Hern\'andez-Lobato(2020)Griffiths and
  Hern\'andez-Lobato]{griffiths202constrained}
Griffiths, R.-R. and Hern\'andez-Lobato, J.~M.
\newblock Constrained {B}ayesian optimization for automatic chemical design
  using variational autoencoders.
\newblock \emph{Chem. Sci."}, 11:\penalty0 577--586, 2020.

\bibitem[Gupta et~al.(2019)Gupta, Koren, and Talwar]{gupta2019better}
Gupta, A., Koren, T., and Talwar, K.
\newblock Better algorithms for stochastic bandits with adversarial
  corruptions.
\newblock In \emph{Conference on Learning Theory (COLT)}, 2019.

\bibitem[Han \& Scarlett(2021)Han and Scarlett]{han2021adversarial}
Han, E. and Scarlett, J.
\newblock Adversarial attacks on {G}aussian process bandits.
\newblock \emph{arXiv preprint arXiv:2110.08449}, 2021.

\bibitem[Ito(2021)]{ito2021optimal}
Ito, S.
\newblock On optimal robustness to adversarial corruption in online decision
  problems.
\newblock \emph{Conference on Neural Information Processing Systems}, 2021.

\bibitem[Janz et~al.(2020)Janz, Burt, and Gonz\'alez]{janz2020bandit}
Janz, D., Burt, D.~R., and Gonz\'alez, J.
\newblock Bandit optimisation of functions in the {M}at\'ern kernel {RKHS}.
\newblock In \emph{International Conference on Artificial Intelligence and
  Statistics}, 2020.

\bibitem[Kanagawa et~al.(2018)Kanagawa, Hennig, Sejdinovic, and
  Sriperumbudur]{Kan18}
Kanagawa, M., Hennig, P., Sejdinovic, D., and Sriperumbudur, B.~K.
\newblock Gaussian processes and kernel methods: A review on connections and
  equivalences.
\newblock https://arxiv.org/abs/1807.02582, 2018.

\bibitem[Kirschner \& Krause(2021)Kirschner and Krause]{kirschner2021bias}
Kirschner, J. and Krause, A.
\newblock Bias-robust bayesian optimization via dueling bandits.
\newblock In \emph{International Conference on Machine Learning}, 2021.

\bibitem[Kirschner et~al.(2020)Kirschner, Bogunovic, Jegelka, and
  Krause]{kirschner2020distributionally}
Kirschner, J., Bogunovic, I., Jegelka, S., and Krause, A.
\newblock Distributionally robust bayesian optimization.
\newblock In \emph{International Conference on Artificial Intelligence and
  Statistics}, pp.\  2174--2184. PMLR, 2020.

\bibitem[Krishnamurthy et~al.(2018)Krishnamurthy, Wu, and
  Syrgkanis]{krishnamurthy2018semiparametric}
Krishnamurthy, A., Wu, Z.~S., and Syrgkanis, V.
\newblock Semiparametric contextual bandits.
\newblock In \emph{International Conference on Machine Learning}, 2018.

\bibitem[Lattimore et~al.(2020)Lattimore, Szepesvari, and
  Weisz]{lattimore2019learning}
Lattimore, T., Szepesvari, C., and Weisz, G.
\newblock Learning with good feature representations in bandits and in {RL}
  with a generative model.
\newblock In \emph{International Conference on Machine Learning}, 2020.

\bibitem[Li et~al.(2019)Li, Lou, and Shan]{li2019stochastic}
Li, Y., Lou, E.~Y., and Shan, L.
\newblock Stochastic linear optimization with adversarial corruption.
\newblock \emph{arXiv preprint arXiv:1909.02109}, 2019.

\bibitem[Li \& Scarlett(2021)Li and Scarlett]{li2021gaussian}
Li, Z. and Scarlett, J.
\newblock Gaussian process bandit optimization with few batches.
\newblock \emph{arXiv preprint arXiv:2110.07788}, 2021.

\bibitem[Liu et~al.(2021)Liu, Li, and Li]{liu2021cooperative}
Liu, J., Li, S., and Li, D.
\newblock Cooperative stochastic multi-agent multi-armed bandits robust to
  adversarial corruptions.
\newblock \emph{arXiv preprint arXiv:2106.04207}, 2021.

\bibitem[Lizotte et~al.(2007)Lizotte, Wang, Bowling, and
  Schuurmans]{lizotte2007automatic}
Lizotte, D.~J., Wang, T., Bowling, M.~H., and Schuurmans, D.
\newblock Automatic gait optimization with {G}aussian process regression.
\newblock In \emph{International Joint Conference on Artificial Intelligence
  (IJCAI)}, pp.\  944--949, 2007.

\bibitem[Lykouris et~al.(2018)Lykouris, Mirrokni, and
  Paes~Leme]{lykouris2018stochastic}
Lykouris, T., Mirrokni, V., and Paes~Leme, R.
\newblock Stochastic bandits robust to adversarial corruptions.
\newblock In \emph{ACM Symposium on Theory of Computing (STOC)}, pp.\
  114--122. ACM, 2018.

\bibitem[Lykouris et~al.(2021)Lykouris, Simchowitz, Slivkins, and
  Sun]{lykouris2021corruption}
Lykouris, T., Simchowitz, M., Slivkins, A., and Sun, W.
\newblock Corruption-robust exploration in episodic reinforcement learning.
\newblock In \emph{Conference on Learning Theory}, pp.\  3242--3245. PMLR,
  2021.

\bibitem[Makarova et~al.(2021)Makarova, Usmanova, Bogunovic, and
  Krause]{makarova2021risk}
Makarova, A., Usmanova, I., Bogunovic, I., and Krause, A.
\newblock Risk-averse heteroscedastic bayesian optimization.
\newblock \emph{Advances in Neural Information Processing Systems}, 34, 2021.

\bibitem[Martinez-Cantin et~al.(2018)Martinez-Cantin, Tee, and
  McCourt]{martinez2018practical}
Martinez-Cantin, R., Tee, K., and McCourt, M.
\newblock Practical {B}ayesian optimization in the presence of outliers.
\newblock In \emph{International Conference on Artificial Intelligence and
  Statistics (AISTATS)}, 2018.

\bibitem[Nguyen et~al.(2021)Nguyen, Dai, Low, and Jaillet]{nguyen2021value}
Nguyen, Q.~P., Dai, Z., Low, B. K.~H., and Jaillet, P.
\newblock Value-at-risk optimization with gaussian processes.
\newblock \emph{arXiv preprint arXiv:2105.06126}, 2021.

\bibitem[Nogueira et~al.(2016)Nogueira, Martinez-Cantin, Bernardino, and
  Jamone]{nogueira2016unscented}
Nogueira, J., Martinez-Cantin, R., Bernardino, A., and Jamone, L.
\newblock Unscented {B}ayesian optimization for safe robot grasping.
\newblock In \emph{IEEE/RSJ International Conference on Intelligent Robots and
  Systems (IROS)}, 2016.

\bibitem[Snoek et~al.(2012)Snoek, Larochelle, and Adams]{snoek2012practical}
Snoek, J., Larochelle, H., and Adams, R.~P.
\newblock Practical {B}ayesian optimization of machine learning algorithms.
\newblock In \emph{Conference on Neural information Processing Systems}, pp.\
  2951--2959, 2012.

\bibitem[Srinivas et~al.(2010)Srinivas, Krause, Kakade, and
  Seeger]{srinivas2009gaussian}
Srinivas, N., Krause, A., Kakade, S.~M., and Seeger, M.
\newblock Gaussian process optimization in the bandit setting: No regret and
  experimental design.
\newblock In \emph{International Conference on Machine Learning (ICML)}, 2010.

\bibitem[Takemori \& Sato(2021)Takemori and Sato]{Tak20}
Takemori, S. and Sato, M.
\newblock Approximation theory based methods for rkhs bandits.
\newblock In \emph{International Conference on Machine Learning}, 2021.

\bibitem[Vakili et~al.(2021{\natexlab{a}})Vakili, Bouziani, Jalali, Bernacchia,
  and shan Shiu]{vakili2021optimal}
Vakili, S., Bouziani, N., Jalali, S., Bernacchia, A., and shan Shiu, D.
\newblock Optimal order simple regret for {G}aussian process bandits.
\newblock In \emph{Conference on Neural information Processing Systems},
  2021{\natexlab{a}}.

\bibitem[Vakili et~al.(2021{\natexlab{b}})Vakili, Khezeli, and
  Picheny]{vakili2020information}
Vakili, S., Khezeli, K., and Picheny, V.
\newblock On information gain and regret bounds in {G}aussian process bandits.
\newblock In \emph{Conference on Neural information Processing Systems},
  2021{\natexlab{b}}.

\bibitem[Wang et~al.(2021)Wang, Zhou, and Gu]{wang2021provably}
Wang, T., Zhou, D., and Gu, Q.
\newblock Provably efficient reinforcement learning with linear function
  approximation under adaptivity constraints.
\newblock \emph{arXiv preprint arXiv:2101.02195}, 2021.

\bibitem[Wang \& Jegelka(2017)Wang and Jegelka]{wang2017max}
Wang, Z. and Jegelka, S.
\newblock Max-value entropy search for efficient {B}ayesian optimization.
\newblock In \emph{International Conference on Machine Learning (ICML)}, pp.\
  3627--3635, 2017.

\bibitem[Wei et~al.(2021)Wei, Dann, and Zimmert]{wei2021model}
Wei, C.-Y., Dann, C., and Zimmert, J.
\newblock A model selection approach for corruption robust reinforcement
  learning.
\newblock \emph{arXiv preprint arXiv:2110.03580}, 2021.

\bibitem[Zhao et~al.(2021)Zhao, Zhou, and Gu]{zhao2021linear}
Zhao, H., Zhou, D., and Gu, Q.
\newblock Linear contextual bandits with adversarial corruptions.
\newblock \emph{arXiv preprint arXiv:2110.12615}, 2021.

\end{thebibliography}

%!TEX root = sample_paper.tex
% Supplementary material: To improve readability, you must use a single-column format for the supplementary material.
\appendix
\onecolumn
{\centering
    {\huge \bf Supplementary Material (Appendix) \\ [3mm]}
    {\Large \bf A Robust Phased Elimination Algorithm for \\ Corruption-Tolerant Gaussian Process Bandits\\} %$\\ [2mm] {\normalsize \bf {Submitted to ICML 2022} \par }  
}
\section{Preliminaries}\label{sec:prelim}

Here, we outline some useful and well-known results and definitions typically used in kernelized/GP bandit (Bayesian optimization) algorithms. 

\textbf{RKHS and kernel functions.} We denote by $\Hc_k$ the reproducing kernel Hilbert space (RKHS) corresponding to the kernel $k$, defined as a Hilbert space of functions equipped with an inner product $\langle \cdot,\cdot\rangle_k$, satisfying the reproducing property, i.e., $\langle f(\cdot),k(\cdot,x)\rangle_k=f(x), \forall x\in\Xc,\forall f\in\Hc_k$.

Since we assume that the kernel is bounded (i.e., $k(x,x') \le 1$), continuous, and has a compact domain (namely, $D = [0,1]^d$), the conditions of Mercer's theorem are satisfied \cite{Kan18}, and the kernel admits a countably infinite (or finite) dimensional feature space, i.e., there exists $\lbrace (\lambda_m,\phi_m)\rbrace_{m=1}^\infty$ such that $k(x,x')=\sum_{m=1}^\infty\lambda_m\phi_m(x)\phi_m(x')$ where the $\phi_m(\cdot)$ are eigenfunctions, and the $\lambda_m \ge 0$ are eigenvalues.  We form an infinite-dimensional feature vector as follows:
\begin{equation}
    \phi(x) = (\sqrt{\lambda_1}\phi_1(x), \sqrt{\lambda_2}\phi_2(x), \dotsc),
\end{equation}
which yields $k(x,x') = \phi(x)^T\phi(x')$.  As stated in the main text, we assume that the RKHS norm is upper bounded by some constant $B>0$.

%\todo{Explain/prove the alternative definition of $\sigma(\cdot)$:}
%\begin{equation} \label{eq:sigma_alt_def}
%	\sigma^2_{t}(x) = \lambda \phi(x)^T ( \Phi_t^T \Phi_t + \lambda I)^{-1} \phi(x) 
%\end{equation}
%where $\|x \|_{M} := \sqrt{x^T M x}$.

The following lemma provides a useful expression for $\sigma^2_t(x)$.  This result is fairly standard, but for completeness, we provide a short proof.  Here and subsequently, we use $I$ to denote the infinite-dimensional identity matrix in feature space.

\begin{lemma}
    Defining $\Phi_t = [\phi(x_1),\dots,\phi(x_t)]^T$, we have
    \begin{equation} \label{eq:sigma_alt_def}
        \sigma^2_{t}(x) = \lambda \phi(x)^T ( \Phi_t^T \Phi_t + \lambda I)^{-1} \phi(x).
    \end{equation}
\end{lemma}
\begin{proof}
    We can rewrite $\sigma^2_t(x)$ as follows,
    % \begin{align}
    %     \sigma^2_t(x)&=k(x,x)-k_t(x)^T(K_t+\lambda I_t)^{-1}k_t(x)\\
    %     &=\phi(x)^T\phi(x)-\phi(x)^T\Phi_t^T(\Phi_t\Phi_t^T+\lambda I_t)^{-1}\Phi_t\phi(x)\\
    %     &=\phi(x)^T\big(I-\Phi_t^T(\Phi_t\Phi_t^T+\lambda I_t)^{-1}\Phi_t\big)\phi(x)\\
    %     &=\lambda \phi(x)^T ( \Phi_t^T \Phi_t + \lambda I_\Hc)^{-1} \phi(x),
    % \end{align}
% {\color{purple} [TODO: Should perhaps be revised. Should be more complicated than this.]
%     where the last step can be obtained by substituting $A=I_\Hc, C=\tfrac{1}{\lambda}I_t,$ and $U^T=V=\Phi_t$ into the following Woodbury matrix identity:
%     \begin{align*}
%     	A^{-1}-A^{-1}U(VA^{-1}U+C^{-1})^{-1}VA^{-1}
%     	&= (UCV+A)^{-1}.
% 	\end{align*} 
	
% 	[OR]
	
	\begin{align}
        \sigma^2_t(x)&=k(x,x)-k_t(x)^T(K_t+\lambda I_t)^{-1}k_t(x)\\
        &=\phi(x)^T\phi(x)-\phi(x)^T\Phi_t^T(\Phi_t\Phi_t^T+\lambda I_t)^{-1}\Phi_t\phi(x)\\
		&=\phi(x)^T \Phi_t^T(\Phi_t\Phi_t^T+\lambda I_t)^{-1}\Phi_t\phi(x)+\lambda\phi(x)^T(\Phi_t^T\Phi_t+\lambda I)^{-1}\phi(x)-\phi(x)^T\Phi_t^T(\Phi_t\Phi_t^T+\lambda I_t)^{-1}\Phi_t\phi(x)\label{eq:phi_x}\\       
        &=\lambda \phi(x)^T ( \Phi_t^T \Phi_t + \lambda I)^{-1} \phi(x),
    \end{align}
	where \cref{eq:phi_x} uses $\phi(x)=\Phi_t^T(\Phi_t\Phi_t^T+\lambda I_t)^{-1}\Phi_t\phi(x)+\lambda(\Phi_t^T\Phi_t+\lambda I)^{-1}\phi(x)$, which can be obtained as follows.
	\begin{align}
		(\Phi_t^T\Phi_t+\lambda I)\phi(x)&=\Phi_t^T\Phi_t\phi(x)+\lambda\phi(x) \\
		\phi(x)&=(\Phi_t^T\Phi_t+\lambda I)^{-1}\Phi_t^T\Phi_t\phi(x)+\lambda(\Phi_t^T\Phi_t+\lambda I)^{-1}\phi(x)\\
		&=\Phi_t^T(\Phi_t\Phi_t^T+\lambda I_t)^{-1}\Phi_t\phi(x)+\lambda(\Phi_t^T\Phi_t+\lambda I)^{-1}\phi(x),
	\end{align}
	where the last step follows from the standard push-through identity $(\Phi_t^T\Phi_t+\lambda I)\Phi_t^T = \Phi_t^T(\Phi_t\Phi_t^T+\lambda I_t)$ (e.g., see Eq.~(12) of \cite{chowdhury17kernelized}), which implies $\Phi_t^T(\Phi_t\Phi_t^T+\lambda I_t)^{-1} = (\Phi_t^T\Phi_t+\lambda I)^{-1}\Phi_t^T$.
	
\end{proof}

Some of the most commonly used kernels are: 
\begin{itemize}
  \item Linear kernel: $k_{\text{lin}}(x,x') = x^T x'$,
  \item Squared exponential kernel: $k_{\text{SE}}(x,x') = \exp \big(- \dfrac{\|x - x'\|^2}{2l^2} \big)$,
  \item Mat\'ern kernel: $k_{\text{Mat}}(x,x') = \frac{2^{1-\nu}}{ \Gamma(\nu) } \Big( \frac{\sqrt{2\nu} \|x - x'\| }{l} \Big) J_{\nu}\Big( \frac{\sqrt{2\nu} \|x - x'\| }{l} \Big)$,
\end{itemize}
where $l$ denotes the length-scale hyperparameter, $\nu > 0$ is an additional hyperparameter that dictates the smoothness, and $J(\cdot)$ and $\Gamma(\cdot)$ denote the modified Bessel function and the Gamma function, respectively.

\textbf{Maximum information gain.} 
The maximum information gain is defined as \cite{srinivas2009gaussian}
\begin{align*}
	\gamma_t&:=\max_{A\subseteq \mathcal{X}:|A|=t}I(f_A;y_A)\\
	&=\max_{x_1,\dots,x_t}\frac{1}{2}\log\det(I_t+\lambda^{-1}K_t),
\end{align*}
where $f_A=[f(x_t)]_{x_t\in A}$, $y_A=[y_t]_{x_t\in A}$, and $I(\cdot;\cdot)$ denotes mutual information. The maximum information gain quantifies the maximum reduction in uncertainty about $f$ after $t$ observations. The following upper bounds for specific kernels have been shown previously \cite{srinivas2009gaussian,vakili2020information}:
\begin{itemize}
  \item Linear kernel: $\gamma_t^\text{lin} = \Otilde(d\log t)$,
  \item Squared exponential kernel: $\gamma_t^\text{SE} = \Otilde((\log t)^d)$,
  \item Mat\'ern kernel: $\gamma_t^\text{Mat} = \Otilde\big(t^\frac{d}{2\nu+d}\big)$.
\end{itemize}

The following lemma shows that $\sum_{t=1}^T  \sigma_{t-1}(x_{t})$ can be upper bounded in terms of $\gamma_T$.

\begin{lemma}\label{lemma:sum_of_sd}
With $\sigma_{t-1}(x_t)$ denoting the posterior standard deviation at $x_t$ based on $(x_1,\dots,x_{t-1})$, we have
\begin{align*}
\sum_{t=1}^T  \sigma_{t-1}(x_{t}) \leq \sqrt{T \sum_{t=1}^T  \sigma^2_{t-1}(x_{t})} \leq \sqrt{\frac{2}{\log (1+\lambda^{-1})} T\gamma_T}\leq \sqrt{(2\lambda+1)T\gamma_T}.
\end{align*}
\end{lemma}
\begin{proof}
The first inequality follows by Cauchy-Schwartz inequality; the second inequality follows from (\cite{srinivas2009gaussian}, Lemma 5.4); the last inequality follows since $(2\lambda+1)\log(1+\lambda^{-1})> 2$ for $\lambda>0$.
\end{proof}

%\textbf{RKHS function concentration.} \todo{Add confidence lemma; Definition of $\beta_t$; From the new paper of Satar;}

\section{Corrupted Confidence Bounds} 

For convenience, we first restate our main assumption regarding non-corrupted confidence bounds.

\rcb*

In this appendix, we prove \cref{lemma:corrupted_confidence_bounds}, which is restated as follows.

\ccb*

\begin{proof}
For simplicity, we denote the epoch length $u_h$ by $t$ in this proof, and use $\mu_t(\cdot),\tmu_t(\cdot),$ and $\sigma_t(\cdot)$ to denote $\mu^{(h)}(\cdot),\tmu^{(h)}(\cdot),$ and $\sigma^{(h)}(\cdot)$, respectively.  Thus, here $\sigma_t(\cdot)$ is defined with respect to the $t = u_h$ sampled points, whereas \Cref{alg:cpe} only computes the posterior variance with respect to the points selected in the for loop, of which there are $l_h$ (possibly strictly fewer than $u_h$).  This part of the analysis only requires the former notion, so there should be no confusion between the two.

We first recall the definition of the robust-corrupted mean estimator from \cref{eq:corrupted_mean}, i.e.,
\begin{equation}
    \tmu_{t}(x) = k_t(x)^T(K_t + \lambda I_t)^{-1} \tY_t, 
\end{equation}
where $\tY_t \in \R^t$ and $\tY_t[i] = \tfrac{\sum_{j=1}^t \bone \lbrace x_i = x_j \rbrace \ty_j}{\sum_{j=1}^t \bone \lbrace x_i = x_j \rbrace}$ for $i \in [t]$. We use $z_t(x;\lambda) \in \R^t$ to denote $k_t(x)^T(K_t + \lambda I_t)^{-1}$ which implies $\tmu_{t}(x) = \sum_{i=1}^{t} z_t(x;\lambda)[i] \cdot \tY_t[i]$. 

We will also use the following equivalent feature-based expression: $z_t(x;\lambda) = k_t(x)^T(K_t + \lambda I)^{-1} = \phi(x)^T\Phi_t^T (\Phi_t \Phi_t^T + \lambda I_t)^{-1}$, where $k(x,x') = \phi(x)^T \phi(x')$, $\phi(x) \in \hilbert$ for every $x \in \Xc$, and $\Phi_t = (\phi(x_{t'}))_{t' \leq t}$ denotes the matrix of (potentially infinite-dimensional) features placed in $t$ rows. Finally, recalling that $I$ denotes the infinite-dimensional identity matrix in feature space, we also have
\begin{equation} \label{eq:z_t_alt}
	z_t(x;\lambda) = \phi(x)^T (\Phi_t^T \Phi_t + \lambda I)^{-1}\Phi_t^T,
\end{equation}
which follows from the standard push-through identity $\Phi_t^T\big(\Phi_t \Phi_t^T + \lambda I_{t}\big)^{-1} = \big(\Phi_t^T \Phi_t + \lambda I \big)^{-1}\Phi_t^T$ (e.g., see Eq.~(12) of \cite{chowdhury17kernelized}).

We proceed to analyze the corrupted estimator $\tmu_{t}(x)$:
\begin{align}
	\tmu_{t}(x) &= \sum_{i=1}^{t} z_t(x;\lambda)[i]\; \tY_t[i] \\
	&= \sum_{i=1}^{t}  \tfrac{\sum_{j=1}^t \bone \lbrace x_i = x_j \rbrace \ty_j}{\sum_{j=1}^t \bone \lbrace x_i = x_j \rbrace} z_t(x;\lambda)[i]  \label{eq:corr_mean_1}\\
	&= \sum_{i=1}^{t}   \tfrac{\sum_{j=1}^t \bone \lbrace x_i = x_j \rbrace (f(x_i) + \epsilon_j + c_j)}{\sum_{j=1}^t \bone \lbrace x_i = x_j \rbrace} z_t(x;\lambda)[i] \label{eq:corr_mean_2}\\
	&= \sum_{i=1}^{t}   f(x_i) z_t(x;\lambda)[i] + \sum_{i=1}^{t} \tfrac{\sum_{j=1}^t \bone \lbrace x_i = x_j \rbrace \epsilon_j}{\sum_{j=1}^t \bone \lbrace x_i = x_j \rbrace} z_t(x;\lambda)[i]  + \sum_{i=1}^{t} \tfrac{\sum_{j=1}^t \bone \lbrace x_i = x_j \rbrace  c_j}{\sum_{j=1}^t \bone \lbrace x_i = x_j \rbrace}z_t(x;\lambda)[i] \label{eq:corr_mean_3}  \\
	&= \sum_{i=1}^{t}  f(x_i)z_t(x;\lambda)[i]  + \sum_{i=1}^{t}\epsilon_i  z_t(x;\lambda)[i] + \sum_{i=1}^{t} \tfrac{\sum_{j=1}^t \bone \lbrace x_i = x_j \rbrace  c_j}{\sum_{j=1}^t \bone \lbrace x_i = x_j \rbrace} z_t(x;\lambda)[i] \label{eq:corr_mean_4} \\
	&= \sum_{i=1}^{t}  (f(x_i)  + \epsilon_i)z_t(x;\lambda)[i] + \sum_{i=1}^{t} \tfrac{\sum_{j=1}^t \bone \lbrace x_i = x_j \rbrace  c_j}{\sum_{j=1}^t \bone \lbrace x_i = x_j \rbrace}z_t(x;\lambda)[i]  \\
	&= \mu_t(x) + \sum_{i=1}^{t}  \tfrac{\sum_{j=1}^t \bone \lbrace x_i = x_j \rbrace  c_j}{\sum_{j=1}^t \bone \lbrace x_i = x_j \rbrace} z_t(x;\lambda)[i]. \label{eq:corr_mean_5}
\end{align}  
Here, we used the definition of $\tY_t[i]$ in \cref{eq:corr_mean_1} and the corrupted observation $\ty_j$ corresponding to $x_j = x_i$ at time $j$ in \cref{eq:corr_mean_2}, while \cref{eq:corr_mean_3} follows from rearranging.  The proof of \cref{eq:corr_mean_4} is deferred to the next paragraph.  Finally, \cref{eq:corr_mean_5} follows from the definition of the noisy stochastic observation $y_i = f(x_i) + \epsilon_i$ and the definition of the standard (non-corrupted) mean estimator from \cref{eq:posterior_mean}.% that depends on the non-corrupted observations $\lbrace y_i\rbrace_{i=1}^t$.

To prove \cref{eq:corr_mean_4}, we define $\tepsilon_t \in \R^t$ such that $\tepsilon_t[i] = \tfrac{\sum_{j=1}^t \bone \lbrace x_i = x_j \rbrace \epsilon_j}{\sum_{j=1}^t \bone \lbrace x_i = x_j \rbrace}$ for $i \in [t]$, and use $u_t(x)$ to denote $\sum_{j=1}^t \bone \lbrace x = x_j \rbrace$, i.e., the number of times action $x$ was played during the $t$ rounds. Then,
\begin{align} 
	\sum_{i=1}^{t} \tfrac{\sum_{j=1}^t \bone \lbrace x_i = x_j \rbrace \epsilon_j}{\sum_{j=1}^t \bone \lbrace x_i = x_j \rbrace} z_t(x;\lambda)[i] &=  z_t(x;\lambda) \tepsilon_t \label{eq:inter_0} \\
	&= \phi(x)^T (\Phi_t^T \Phi_t + \lambda I)^{-1}\Phi_t^T \tepsilon_t \label{eq:inter_1}\\
	&= \phi(x)^T (\Phi_t^T \Phi_t + \lambda I)^{-1} \sum_{i=1}^t \phi(x_i) \tepsilon_t[i] \label{eq:inter_1a} \\
	&= \phi(x)^T (\Phi_t^T \Phi_t + \lambda I)^{-1} \sum_{x \in \Xc, u_t(x) \neq 0} u_t(x) \phi(x) \tfrac{\sum_{j=1}^t \bone \lbrace x = x_j \rbrace \epsilon_j}{u_t(x)}  \label{eq:inter_2} \\
	&= \phi(x)^T (\Phi_t^T \Phi_t + \lambda I)^{-1} \sum_{x \in \Xc, u_t(x) \neq 0} \phi(x) \sum_{j=1}^t \bone \lbrace x = x_j \rbrace \epsilon_j \\
	&= \phi(x)^T (\Phi_t^T \Phi_t + \lambda I)^{-1}  \sum_{j=1}^t \phi(x_j)  \epsilon_j   \label{eq:inter_4} \\
	%&= \phi(x)^T (\Phi_t^T \Phi_t + \lambda I_{\Hc})^{-1}  \sum_{j=1}^t \phi(x_j)  \epsilon_j \\
	&= \phi(x)^T (\Phi_t^T \Phi_t + \lambda I)^{-1} \Phi_t^T \epsilon_t  \label{eq:inter_5} \\
	&= z_t(x;\lambda) \epsilon_t = \sum_{i=1}^t \epsilon_i z_t(x;\lambda)[i],  \label{eq:inter_6}
\end{align} 
where \cref{eq:inter_1} holds due to \cref{eq:z_t_alt}, and \cref{eq:inter_2} uses the definitions of $ \tepsilon_t$ and $u_t(x)$, and \eqref{eq:inter_4}--\eqref{eq:inter_6} are analogous to \eqref{eq:inter_0}--\eqref{eq:inter_1a} in the opposite order.

By rearranging \cref{eq:corr_mean_5}, it follows that we can bound the absolute difference between the corrupted mean estimator and the standard one as follows:
\begin{equation} \label{eq:corr_mean_diff}
| \tmu_{t}(x) - \mu_t(x)| \leq \Big| \sum_{i=1}^{t} \tfrac{\sum_{j=1}^t \bone \lbrace x_i = x_j \rbrace  c_j}{\sum_{j=1}^t \bone \lbrace x_i = x_j \rbrace} z_t(x;\lambda)[i] \Big|.
\end{equation}

Next, we proceed to analyze the right hand side term.  We use $C_t$ to denote a vector in $\R^t$ such that $C_t[i] = \tfrac{\sum_{j=1}^t \bone \lbrace x_i = x_j \rbrace  c_j}{\sum_{j=1}^t \bone \lbrace x_i = x_j \rbrace}$ for every $i \in [t]$.  Then, continuing from \cref{eq:corr_mean_diff}, we have
\begin{align}
	\Big| \sum_{i=1}^{t} \tfrac{\sum_{j=1}^t \bone \lbrace x_i = x_j \rbrace  c_j}{\sum_{j=1}^t \bone \lbrace x_i = x_j \rbrace} z_t(x;\lambda)[i] \Big| % &= \Big| \phi(x)^T\Phi_t^T (\Phi_t \Phi_t^T + \lambda I_t)^{-1} C_t   \Big|\\
	&= \Big| \phi(x)^T \underbrace{(\Phi_t^T \Phi_t + \lambda I)^{-1}}_{:=\Gamma_t^{-1}}\Phi_t^T C_t \Big|  \label{eq:standard_identity} \\
	%&= \Big| \langle \phi(x)^T \Gamma_t^{-1}, \Phi_t c_t  \rangle_{\Hc} \Big| 
	&= \Big| \sum_{i=1}^t C_t[i]  \phi(x)^T \Gamma_t^{-1} \phi(x_i) \Big|, \label{eq:abs_bound}
\end{align}  
where we again used the form of $z_t$ given in \cref{eq:z_t_alt}.

Let $C_t(x) = \sum_{j=1}^t \bone \lbrace x = x_j \rbrace  c_j$ for $x \in \Xc$. Then, we can rewrite \eqref{eq:abs_bound} as
\begin{align}
	&\Big| \sum_{i=1}^{t} \tfrac{\sum_{j=1}^t \bone \lbrace x_i = x_j \rbrace  c_j}{\sum_{j=1}^t \bone \lbrace x_i = x_j \rbrace} z_t(x;\lambda)[i] \Big| = \Big| \sum_{x' \in \Xc, u_t(x') \neq 0} \frac{C_t(x')}{u_t(x')} u_t(x')  \phi(x)^T \Gamma_t^{-1} \phi(x') \Big| \\
	&\quad \quad \leq \sum_{x' \in \Xc, u_t(x') \neq 0} \frac{C}{u_t(x')} u_t(x') \Big| \phi(x)^T \Gamma_t^{-1} \phi(x') \Big| \label{eq:enter_C} \\
	&\quad \quad \leq \frac{C}{u_{\min}}\sum_{x' \in \Xc, u_t(x') \neq 0} u_t(x') \Big| \phi(x)^T \Gamma_t^{-1} \phi(x') \Big| \label{eq:enter_umin} \\
	&\quad \quad \leq \frac{C}{u_{\min}} \sqrt{ \Big(\sum_{x' \in \Xc, u_t(x') \neq 0} u_t(x') \Big) \phi(x)^T \sum_{x' \in \Xc, u_t(x') \neq 0} u_t(x')  \Gamma_t^{-1} \phi(x') \phi(x')^T \Gamma_t^{-1}  \phi(x)} \label{eq:apply_jensen} \\ 
	&\quad \quad \leq \frac{C}{u_{\min}} \sqrt{ \Big(\sum_{x' \in \Xc, u_t(x') \neq 0} u_t(x') \Big) \phi(x)^T \sum_{x' \in \Xc, u_t(x') \neq 0} u_t(x')  \Gamma_t^{-1} \big(\phi(x') \phi(x')^T + \tfrac{\lambda}{t} I \big) \Gamma_t^{-1}  \phi(x)} \\ 
	&\quad \quad = \frac{C}{u_{\min}} \sqrt{ \sum_{x' \in \Xc, u_t(x') \neq 0} u_t(x') \| \phi(x) \|^2_{\Gamma_t^{-1}}} \label{eq:alt_def_of_gamma_matrix} \\ 
	&\quad \quad = \frac{C}{u_{\min}} \sqrt{t} \| \phi(x) \|_{\Gamma_t^{-1}} =  \frac{C \sqrt{t}}{ \lambda u_{\min}} \sigma_{t}(x), \label{eq:final_term_cb}
\end{align}
where:
\begin{itemize}
    \item \cref{eq:enter_C} holds since $C \geq |C_t(x)|$ for every $x \in \Xc$.
    \item \cref{eq:enter_umin} follows from the definition of $u_{\min}$ in the lemma statement.
    \item To obtain \cref{eq:apply_jensen}, we multiply and divide by $\sum_{x \in \Xc, u_t(x) \neq 0} u_t(x)$ and apply $\E[|X|] \leq \sqrt{\E[X^2]} $ considering the distribution  $\tfrac{u_t(x')}{\sum_{x \in \Xc, u_t(x) \neq 0} u_t(x)}$.  (Note also that, in generic vector-matrix notation, $(a^T M b)^2 = a^T M b b^T M a$ when $M$ is a symmetric matrix. )
    \item To obtain \cref{eq:alt_def_of_gamma_matrix}, we use $ \sum_{x' \in \Xc, u_t(x') \neq 0} u_t(x') \tfrac{\lambda}{t} I = \lambda I$ (i.e., $\sum_{x' \in \Xc, u_t(x') \neq 0} u_t(x') = t$), and note that $\Gamma_t = \big(\sum_{x' \in \Xc, u_t(x') \neq 0} u_t(x') \phi(x') \phi(x')^T\big) + \lambda I$.  Combining these facts gives $ \sum_{x' \in \Xc, u_t(x') \neq 0} u_t(x') \big(\phi(x') \phi(x')^T + \tfrac{\lambda}{t} I \big) = \Gamma_t$, which cancels with one of the $\Gamma_t^{-1}$ terms.  The remaining quantity $\phi(x)^T \Gamma_t^{-1} \phi(x)$ is precisely the definition of $\| \phi(x) \|_{\Gamma_t^{-1}}^2$.
    \item Finally, \cref{eq:final_term_cb} holds since
    \begin{equation}
        \| \phi(x) \|^2_{\Gamma_t^{-1}} = \phi(x)^T \Gamma_t^{-1} \phi(x) = \lambda^{-1} \sigma^2_{t}(x),
    \end{equation}	
    which holds due to \cref{eq:sigma_alt_def}.
\end{itemize}
Conditioned on the event in \cref{assumption:regular_confidence_bounds}, the final result then follows since 
\begin{equation}
	|\tmu_t (x) - f (x)| \leq | \mu_t (x) - f (x) | + |\tmu_t (x) - \mu_t(x)| \le  \Big(\beta_h +  \tfrac{C \sqrt{t}}{ \lambda u_{\min}}\Big) \sigma_t(x),
\end{equation}
where we apply \cref{assumption:regular_confidence_bounds} and \cref{eq:final_term_cb} to upper bound $| \mu_t (x) - f (x) |$ and $|\tmu_t (x) - \mu_t(x)|$, respectively. 

\end{proof}
%\newpage
\section{Auxiliary Results}

In the following, we recall the notation in \Cref{alg:cpe}, particularly the truncation parameter $\psi > 0$.  In addition, in accordance with the algorithm statement, quantities such as $\sigma_{t}(\cdot)$ and $K_t$ implicitly depend on $h$, and are defined with respect to the $t \le l_h$ points chosen up to time $t$ in the for loop (as opposed to the $u_h \ge l_h$ points sampled {\em after} the for loop).

We first formalize the claim that the number of epochs is at most $\Hbar = \log_2 T$.

\begin{lemma} \label{lemma:Hbar}
    For any time horizon $T$, \Cref{alg:cpe} terminates after at most $\log_2 T$ epochs.
\end{lemma}
\begin{proof}
    This follows immediately from the fact that we initialize $l_0 = 2$, double $l_h$ after each epoch, and take at least $l_h$ actions in epoch $h$ (see Line 12 with $\sum_{x} \xi_h(x) = 1$) until $T$ actions have been played.
\end{proof}

Next, we state a simple result regarding the epoch lengths.

\begin{lemma} \label{lemma:length_of_epoch}
    The length $u_h$ of epoch $h$ in~\Cref{alg:cpe} satisfies $u_h \le l_h(2 + |\Sc_h|\psi)$.  
\end{lemma}
\begin{proof}
	The number of times each action from $\Sc_h$ is played is $u_h(x)$, and is given in \Cref{alg:cpe} (Line 12). Hence, we have
	\begin{align}
		u_h  &= \sum_{x \in \Sc_h} \lceil l_h \max \lbrace \xi_h(x), \psi \rbrace  \rceil \\
			 &\leq  \sum_{x \in \Sc_h} (l_h \max \lbrace \xi_h(x), \psi \rbrace + 1) \\
			 &\leq |\Sc_h| + \sum_{x \in \Sc_h} (l_h \xi_h(x) + l_h \psi) \\
			 &\leq 2l_h +  l_h \psi |\Sc_h| = l_h(2 + \psi |\Sc_h|),
	\end{align}
	where in the last inequality, we use $|\Sc_h| \leq l_h$ and $\sum_{x \in \Sc_h} \xi_h(x) = 1$.
\end{proof}

% The following lemma extends \cite[Lemma 12]{abbasi2011linear} to non-parametric kernel matrices (and potentially infinite feature spaces) and kernelized variance estimators (\todo{Not quite. The previous one is more related to the exploration parameter than the ratio of variances.}). 

The following result characterizes the posterior uncertainty of points sampled in between the switching events in \Cref{alg:cpe}, and may be of independent interest for problems in RKHS function spaces, particularly in settings where infrequent action switching is desirable.

\begin{lemma}\label{lemma:sigma_in_between}
Consider any epoch $h$, the corresponding set of actions $\Xc_h$, and the regularization parameter $\lambda > 0$. Let $t, t' \in [l_h]$ denote two rounds in epoch $h$ such that $t \geq t'$, and for which 
\begin{equation}
	\label{eq:lemma_condition}
	\det(I_t + \lambda^{-1} K_t) \leq \eta \det(I_{t'} + \lambda^{-1} K_{t'})	
\end{equation}	
(i.e., the condition in Line 6 in \Cref{alg:cpe} does not hold), where $\eta > 1$. Then, for every $x \in \Xc_h$, it holds that 
\begin{equation}\label{eq:lemma_main_result}
	\sigma_{t'}(x) \leq \sqrt{\eta} \sigma_{t}(x).
\end{equation}
\end{lemma}

\begin{proof}
We first consider the case that $k(x,x') = \phi(x)^T \phi(x')$ for every $x,x' \in \Xc$ with finite-dimensional features: $\phi(x) \in \R^{d_{\phi}}$ for some $d_{\phi} < \infty$.  We let $\Phi_t = (\phi(x_{t'}))_{t' \leq t} \in \RR^{t \times d_{\phi}}$ denote the matrix of features placed in $t$ rows. We will later drop the assumption of finite dimensionality to obtain the result in our original setup.

We also note that if $\phi(x)$ contains all zeros for some input $x \in \Xc$, the statement in \Cref{eq:lemma_main_result} trivially holds (i.e., both sides are zero), so in the rest of the analysis, we assume that this is not the case.\looseness=-1

In the following, let $x$ be any fixed point in the domain. From \cref{eq:lemma_condition}, we have: 
	\begin{align}
		\eta &\geq  \frac{\det(\lambda^{-1}K_t +  I_t)}{\det(\lambda^{-1}K_{t'} +  I_{t'})}  \\
			 &= \frac{\det(K_t +  \lambda I_t)}{\det(K_{t'} + \lambda I_{t'})} \label{eq:still_holds}  \\
			  &= \frac{\det\Big(\Phi^T_{t}\Phi_{t} +  \lambda I_d\Big)}{\det\Big( \Phi^T_{t'}\Phi_{t'} +  \lambda I_d\Big)} \label{eq:WA} \\
		     &= \frac{\det\Big( \big(\Phi^T_{t'}\Phi_{t'} +  \lambda I_d\big)^{-1}\Big)}{\det\Big( \big(\Phi^T_{t}\Phi_{t} +  \lambda I_d\big)^{-1}\Big)} \label{eq:det_inv} \\
		     &\geq \frac{ \phi(x)^T\big(\Phi^T_{t'}\Phi_{t'} +  \lambda I_d\big)^{-1}\phi(x)}{\phi(x)^T \big(\Phi^T_{t}\Phi_{t} +  \lambda I_d\big)^{-1}\phi(x)}  \label{eq:det_ineq_invoke} \\
             & = \frac{\sigma^2_{t'}(x)}{\sigma^2_{t}(x)}. \label{eq:det_eq}
	\end{align}

	% From \cref{eq:lemma_condition}, it follows: 
	% \begin{align}
	% 	\eta &\geq  \frac{\det(\lambda^{-1}K_t +  I_t)}{\det(\lambda^{-1}K_{t'} +  I_{t'})}  \\
	% 	     &= {\color{blue}\frac{\det(\lambda^{-1}\Phi^T_{t}\Phi_{t} +  I_{\Hc})}{\det(\lambda^{-1}\Phi^T_{t'}\Phi_{t'} +  I_{\Hc})}} \label{eq:WA}\\
	% 	     &= {\color{blue}\frac{\det(\Phi^T_{t}\Phi_{t} +  \lambda I_{\Hc})}{\det(\Phi^T_{t'}\Phi_{t'} +  \lambda I_{\Hc})}} \\
	% 	     &= {\color{blue}\frac{\det\Big( \big(\Phi^T_{t'}\Phi_{t'} +  \lambda I_{\Hc}\big)^{-1}\Big)}{\det\Big( \big(\Phi^T_{t}\Phi_{t} +  \lambda I_{\Hc}\big)^{-1}\Big)}} \label{eq:inv_of_det} \\
	% 	     &\geq {\color{blue} \frac{ \phi(x)^T\big(\Phi^T_{t'}\Phi_{t'} +  \lambda I_{\Hc}\big)^{-1}\phi(x)}{\phi(x)^T \big(\Phi^T_{t}\Phi_{t} +  \lambda I_{\Hc}\big)^{-1}\phi(x)}} = \frac{\sigma^2_{t'}(x)}{\sigma^2_{t}(x)}.
	% 	     \label{eq:det_ineq_invoke}
	% 	% \frac{\det(\lambda^{-1}\Phi^T_{t}\Phi_{t} +  I_{\Hc})}{\det(\lambda^{-1}\Phi^T_{t'}\Phi_{t'} +  I_{\Hc})} \label{eq:det_ineq_invoke} \\
	% 	% \frac{\sigma^2_{t'}(x)}{\sigma^2_{t}(x)} &= \frac{\lambda \phi(x)^T(\Phi^T_{t'}\Phi_{t'} + \lambda I_{\Hc})^{-1} \phi(x)}{\lambda \phi(x)^T(\Phi^T_{t}\Phi_{t} + \lambda I_{\Hc	})^{-1} \phi(x)} \\
	% 	% &=  \frac{\phi(x)^T(\lambda^{-1}\Phi^T_{t}\Phi_{t} +  I_{\Hc}) \phi(x)}{ \phi(x)^T(\lambda^{-1}\Phi^T_{t'}\Phi_{t'} +  I_{\Hc}) \phi(x)}\\
	% \end{align}
	Here, \cref{eq:WA} holds due to the Weinstein–Aronszajn identity (i.e., $\det(I+AB) = \det(I+BA)$), and in \cref{eq:det_inv} we use the fact that $\det(A) = (\det(A^{-1}))^{-1}$ for any invertible matrix $A$.  \cref{eq:det_ineq_invoke} is proved in the following paragraph, and \cref{eq:det_eq} follows from the alternative definition of $\sigma_t(\cdot)$ in \cref{eq:sigma_alt_def}.

	It remains to prove the inequality in \cref{eq:det_ineq_invoke}, which closely follows the proof of [Lemma 12, \citet{abbasi2011linear}]. For any $i \in [t]$, let $V_i := \lambda^{-1}\Phi^T_{i}\Phi_{i} + I$. We first show that
	\begin{equation}
		\label{eq:rec_step_det}
		\frac{\phi(x)^T V_t \phi(x)}{\phi(x)^T V_{t-1} \phi(x)} \leq 1 + \|\lambda^{-1/2} \phi(x_t) \|^2_{V^{-1}_{t-1}}.
	\end{equation}
	We have for any $x \in \Xc_h$ that
	\begin{align}
		\phi(x)^T V_t \phi(x) &= \phi(x)^T V_{t-1} \phi(x) + \phi(x)^T \big( \lambda^{-1} \phi(x_t) \phi(x_t)^T \big) \phi(x) \\ 
		&= \phi(x)^T V_{t-1} \phi(x) + \lambda^{-1} \big(\phi(x)^T\phi(x_t)\big)^2 \\
		&= \phi(x)^T V_{t-1} \phi(x) + \lambda^{-1} \big(\phi(x)^T V_{t-1}^{1/2} V_{t-1}^{-1/2}\phi(x_t)\big)^2 \\
		&\leq \phi(x)^T V_{t-1} \phi(x) + \lambda^{-1} \|\phi(x)^T V_{t-1}^{1/2} \|_{2}^2 \| V_{t-1}^{-1/2}\phi(x_t) \|_{2}^2 \label{eq:cs}\\
		&= \phi(x)^T V_{t-1} \phi(x) +  \lambda^{-1} (\phi(x)^T V_{t-1} \phi(x)) ( \phi(x_t) V_{t-1}^{-1} \phi(x_t))\\
		&= \Big(1 + \|\lambda^{-1/2} \phi(x_t) \|_{V^{-1}_{t-1}}^2\Big) \phi(x)^T V_{t-1} \phi(x), 
	\end{align}
	where \cref{eq:cs} follows from Cauchy-Schwarz inequality.  Hence, \cref{eq:rec_step_det} follows by rearranging.

	Since $t > t'$, we have:
	\begin{align}
		\frac{\phi(x)^T V_t \phi(x)}{\phi(x)^T V_{t'} \phi(x)} &= \frac{\phi(x)^T V_t \phi(x)}{\phi(x)^T V_{t-1} \phi(x)} \cdot \frac{\phi(x)^T V_{t-1} \phi(x)}{\phi(x)^T V_{t-2} \phi(x)} \cdot \dots \frac{\phi(x)^T V_{t'+1} \phi(x)}{\phi(x)^T V_{t'} \phi(x)}  \\
		&\leq  \big(1 + \|\lambda^{-1/2} \phi(x_t) \|^2_{V^{-1}_{t-1}}\big) \cdot \big(1 + \|\lambda^{-1/2} \phi(x_{t-1}) \|^2_{V^{-1}_{t-2}}\big) \cdot \dots \big(1 + \|\lambda^{-1/2} \phi(x_{t'+1}) \|^2_{V^{-1}_{t'}}\big) \label{eq:det_1} \\
		&= \frac{\det(V_t)}{\det(V_{t-1})} \cdot \frac{\det(V_{t-1})}{\det(V_{t-2})} \cdot \dots \frac{\det(V_{t'+1})}{\det(V_{t'})} \label{eq:det_2}\\
		&= \frac{\det(V_t)}{\det(V_{t'})} \label{eq:final_det_det},
	\end{align}
	where \cref{eq:det_1} follows from \cref{eq:rec_step_det}, and \cref{eq:det_2} uses the fact that 
	\begin{equation}
		\frac{\det(V_t)}{\det(V_{t-1})} =  1 + \|\lambda^{-1/2} \phi(x_t) \|^2_{V^{-1}_{t-1}},
	\end{equation}
	which is shown in [Proof of Theorem 2.2, \citet{durand2018streaming}].

	It remains to handle the possibly infinite feature dimension. %To this end, it is enough to take a look at the approximations using the $D$ first dimension of the sequence for each d. 
	%\footnote{This holds since $k(\cdot,\cdot)$ is a kernel function that is
	%continuous, symmetric positive definite on a compact set equipped with a positive finite
	%Borel measure, then there is at most countable sequence...\todo{Finish this; Isn't this all equivalent to saying that the kernel is Mercer?} } 
	Consider $k(x,x') = \sum_{i=1}^{\infty} \lambda_i \phi_i(x) \phi_i(x')$ and let $k_{d_{\phi}}(x,x') = \sum_{i=1}^{d_{\phi}} \lambda_i \phi_i(x) \phi_i(x')$ denote the finite dimensional kernel that corresponds to the $d_{\phi}$-dimensional feature space such that $\lim_{d_{\phi} \to \infty} k_{d_{\phi}}(x,x') = k(x,x')$ for every $x,x' \in \Xc$. We use $K_{t,d_{\phi}}$ and $\sigma^2_{t,d_{\phi}}(\cdot)$ to denote the restriction of the corresponding quantities when the kernel $k_{d_{\phi}}(\cdot,\cdot)$ is used. First, we note that \cref{eq:still_holds} still holds. Moreover, we have $\frac{\det(K_t +  \lambda I_t)}{\det(K_{t'} + \lambda I_{t'})} = \lim_{d_{\phi} \to \infty}
	\frac{\det(K_{t,d_{\phi}} +  \lambda I_t)}{\det(K_{t',d_{\phi}} + \lambda I_{t'})}$ and $\frac{\sigma^2_{t'}(x)}{\sigma^2_{t}(x)} = \lim_{d_{\phi} \to \infty} \frac{\sigma^2_{t',d_{\phi}}(x)}{\sigma^2_{t,d_{\phi}}(x)}$, and the former limit is lower bounded by the latter due to the fact that
	\cref{eq:det_inv,eq:WA,eq:det_ineq_invoke} are all valid for the finite $d_{\phi}$-feature approximation. Thus, the final result still holds for infinite dimensional kernels.

	% More generally, the result obtained in \cref{eq:final_det_det} can be perhaps stated as follows (which then generalizes the result of Lemma 12 in Abbasi): Let $A$, $B$ and $C$ be positive semi-definite bi-infinite matrices such that $A = B + C$. Then, it holds that:
	% \begin{equation}
	% 	\sup_{x \neq 0} \frac{x^TA x}{x^T B x} \leq \frac{\det(A)}{\det(B)}
	% \end{equation}
	% (I guess $x = 0$ does not even belong to RKHS, so we should not worry about such case.)
	% The final result in \cref{eq:det_ineq_invoke} follows by setting $A=\big(\Phi^T_{t'}\Phi_{t'} +  \lambda I\big)^{-1}$ and $B=\big(\Phi^T_{t}\Phi_{t} +  \lambda I\big)^{-1}$.

	% which can be shown as follows:
	% \begin{align}
	% 	\det(V_t) &= \det \big(V_{t-1} + \lambda^{-1} \phi(x_t) \phi(x_t)^T\big) \\
	%      &= \det\Big(V_{t-1}\big(I_{\Hc} + V_{t-1}^{-1/2} (\lambda \phi(x_t) \phi(x_t)^T) V_{t-1}^{-1/2} \big)\Big) \\	
	% 	 &= \det(V_{t-1}) \det\Big(I_{\Hc} + \big(V_{t-1}^{-1/2} \lambda^	{-1/2} \phi(x_t) \big)\big( V_{t-1}^{-1/2} \lambda^{-1/2} \phi(x_t)\big)^T\Big) \\
	% 	 &= \det(V_{t-1}) \det\Big(1 + \big(V_{t-1}^{-1/2} \lambda^{-1/2} \phi(x_t) \big)^T\big( V_{t-1}^{-1/2} \lambda^{-1/2} \phi(x_t)\big)\Big) \label{eq:det_proof_1} \\
	% 	 &= \det(V_{t-1}) \det\Big(1 +  \| \lambda^{-1/2} \phi(x_t) \|^2_{V_{t-1}^{-1}} \Big) \label{eq:det_proof_2} \\
	% 	 &= \Big(1 + \|\lambda^{-1/2} \phi(x_t) \|^2_{V^{-1}_{t-1}}\Big)\det(V_{t-1}),
	% \end{align}
	% where \cref{eq:det_proof_1} is due to the Weinstein–Aronszajn identity.
\end{proof}

Next, we uniformly bound the posterior variance for the points remaining after a given epoch.

\begin{lemma}\label{lemma:max_var_lemma}
For any epoch $h$ and the corresponding set of actions $\Xc_h$, it holds that
\begin{equation}
	\max_{x \in \Xc_h} \sigma^{(h)}(x) \leq  \sqrt{\frac{\eta (2 \lambda + 1)\gamma_{l_h}}{l_h}}.
\end{equation}
\end{lemma}
\begin{proof}
Recall that $u_h$ corresponds to the length of epoch $h$ and that we can $\sigma^{(h)}(x)$ represents a posterior variance $\sigma_{u_h}(x)$ taken with respect to the $u_h$ sampled points after the epoch.  We first relate this to the posterior variance $\sigma_{l_h}(x)$ (abusing notation slightly) taken only with respect to the $l_h$ points in the for loop in \Cref{alg:cpe}.  In particular, we claim that the former is upper bounded by the latter, and so it suffices to work with $\sigma_{l_h}(x)$.  To see this, we recall that each $x$ is sampled $u_h(x)=\lceil l_{h} \max\lbrace \xi_h(x), \psi \rbrace \rceil$ times, and the definition $\xi_h(x)= \tfrac{\sum_{i=1}^{l_h} \bone \lbrace  x = x_i \rbrace}{l_h}$ gives $ l_{h} \xi_h(x) = \sum_{i=1}^{l_h} \bone \lbrace  x = x_i \rbrace$.  Thus, the number of times each point is sampled is at least as high as the number of times it is selected in the for loop.  Since conditioning on a higher number of points always decreases (or at least does not increase) the posterior variance in a Gaussian process, the desired claim follows.

We proceed to upper bound $\max_{x \in \Xc_h} \sigma_{l_h}(x)$.  Let $\Tc_h = \lbrace t \in [l_h]: \det(I_t + \lambda^{-1} K_t) > \eta \det(I_{t'} + \lambda^{-1} K_{t'}) \rbrace$ be the rounds in which the condition in Line 6 (\cref{alg:cpe}) is satisfied. Moreover, let $\bar{\Tc}_h = \Tc_h \cup \lbrace 0 \rbrace$ and let its elements $\bar{\Tc}_h = \lbrace t'_0, \dots, t'_i, \dots, t'_{|\Tc_h|} \rbrace$ be increasingly ordered. We note that $\max_{x \in \Xc_h} \sigma_{l_h}(x) \leq \sigma_{t'_i}(x_{t'_i + 1})$ for every $t'_i \in \bar{\Tc}_h$ according to the selection rule in \cref{alg:cpe} (Line 4) and the fact that $\sigma_t(\cdot)$ is decreasing with respect to $t$.  
% Let $\lbrace t'_0, t'_1, \dots, t'_{n} \rbrace$ (with $t'_0=0$) denote the set of rounds $t'$ in epoch $h$ when the condition in Line 5 holds true. Since it holds that $\max_{x \in \Xc_h} \sigma_{l_h}(x) \leq \sigma_{t'_i}(x_{t'_{i} + 1})$ for every $t'_i \in \lbrace t'_0, t'_1, \dots, t'_{n} \rbrace$.
It follows that
\begin{equation} \label{eq:max_var_bound}
	 l_h \big(\max_{x \in \Xc_h} \sigma_{l_h}(x)\big) \leq \Big(\sum_{i=0}^{|\Tc_h| - 1} (t'_{i+1} - t'_i) \sigma_{t'_i}(x_{t'_{i}+1})\Big) + (l_h - t'_{|\Tc_h|}) \sigma_{t'_{|\Tc_h|}}(x_{t'_{|\Tc_h|} + 1}).
\end{equation}	
Observe that by definition, we have $x_{t'_i + 1} = x_{t'_i + 2}= \cdots = x_{t'_{i+1}}$, i.e., these form a chain of identical points up to when the switching condition in Line 6 holds.  Accordingly, by \Cref{lemma:sigma_in_between}, it holds that $\sigma_{t'_i}(x_{t'_{i}+1}) \leq \sqrt{\eta} \sigma_{t}(x_{t+1})$ for every $t \in \lbrace t'_i, \dots, t'_{i+1} - 1 \rbrace$. By combining this with \cref{eq:max_var_bound}, we obtain
\begin{equation} \label{eq:bound_var}
	 l_h \big(\max_{x \in \Xc_h} \sigma_{l_h}(x)\big) \leq \sqrt{\eta} \sum_{t=0}^{l_h-1} \sigma_{t}(x_{t+1}).
\end{equation}
Finally, from \cref{lemma:sum_of_sd}, we have $\sum_{t=0}^{l_h-1} \sigma_{t}(x_{t+1}) \leq \sqrt{(2 \lambda + 1) \gamma_{l_h} l_h}$. By combining this with \Cref{eq:bound_var} and rearranging, we obtain the final result. 
\end{proof}

Finally, we provide a result bounding the size of the set $\Sc_h$ in \Cref{alg:cpe}.

\begin{lemma}\label{lemma:number_of_switches}
For any epoch $h$ and the corresponding set $\Sc_h$, we have
\begin{equation}
 | \Sc_h |  \leq \tfrac{2}{\ln \eta} \gamma_{T}.
\end{equation}
\end{lemma}
\begin{proof}
	By the algorithm design, the set $\Sc_h$ grows by at most one element after the condition in Line 6 is satisfied, i.e., when
	\begin{equation}
		\det(I_t + \lambda^{-1} K_t) > \eta \det(I_{t'} + \lambda^{-1} K_{t'}),
	\end{equation}  
    where $t$ is the current iteration, and $t'$ is iteration prior to $t$ for which Line 6 held (or $t' = 0$).  As before, let $\Tc_h = \lbrace t \in [l_h]: \det(I_t + \lambda^{-1} K_t) > \eta \det(I_{t'} + \lambda^{-1} K_{t'}) \rbrace$ be the rounds in which this holds, ordered with respect to time. 
	% For consecutive $t_i$ and $t_{i-1}$ that belong to $\bar{\Tc}_h = \Tc_h \cup \lbrace 1 \rbrace$ , we have:
	Thus, for consecutive $t_i$ and $t_{i-1}$ belonging to $\Tc_h$ , we have
	\begin{equation}
		\det(I_{t_i} + \lambda^{-1} K_{t_i}) > \eta \det(I_{t_{i-1}} + \lambda^{-1} K_{t_{i-1}}).
	\end{equation}
	By applying the previous relation recursively, it follows that
	\begin{equation}
		\det(I_{t_i} + \lambda^{-1} K_{t_i}) > \eta \det(I_{t_{i-1}} + \lambda^{-1} K_{t_{i-1}}) > \eta^2 \det(I_{t_{i-2}} + \lambda^{-1} K_{t_{i-2}})  \dots >  \eta^{i+1} \det(1 + \lambda^{-1}) = \eta^{i+1}(1 + \lambda^{-1}).
	\end{equation}
	% Hence, it follows that 
	% \begin{equation}
	% 	\det(I_{t_{|\Sc_h|}} + \lambda^{-1} K_{t_{|\Sc_h|}}) > \eta^{|\Sc_h|} (1 + \lambda^{-1}).
	% \end{equation}

	Using the definition of $\gamma_{l_h}$ given in \eqref{eq:max_info_gain}, and noting that the size of the set $\Tc_h$ is at least $|\Sc_h| - 1$, we obtain
	\begin{equation}
		\gamma_{l_h} \geq \tfrac{1}{2} \ln \det(I_{l_h} + \lambda^{-1} K_{l_h}) \geq \tfrac{1}{2} \ln(\eta^{|\Sc_h|} (1 + \lambda^{-1})) \geq \tfrac{1}{2} \ln(\eta^{|\Sc_h|}).
	\end{equation}
	By rearranging, we obtain
	\begin{equation}
		|\Sc_h| \leq \tfrac{2}{\ln \eta} \gamma_{l_h}.  
	\end{equation}
	The result then follows since $\gamma_T \geq \gamma_{l_h}$ for every $h$.

\end{proof}

\section{Regret Analysis}

In this appendix, we prove our main result, \cref{theorem:main}.  We first upper bound the regret of any point sampled in a given epoch.

\begin{lemma} \label{lemma:instant_regret_lemma}
	With probability at least $1-\delta$, we have for every epoch $h$ and $x \in \Xc_h$ that
	\begin{equation}
		\max_{x \in \Xc_h} f(x) - f(x) \leq 4 \Big(\beta_{h-1} + \tfrac{C\sqrt{u_{h-1}}}{l_{h-1} \psi \lambda}\Big) \sqrt{\tfrac{\eta (2\lambda + 1) \gamma_{l_{h-1}}}{l_{h-1}}}.
	\end{equation} 
\end{lemma}
\begin{proof}
Recall that $u_h$ denotes the epoch length, and let $x_h^* \in \argmax_{x \in \Xc_h} f(x)$. By using the validity of the confidence bounds from the end of the previous epoch $h-1$ (see \cref{lemma:corrupted_confidence_bounds}), we have for all $x \in \Xc_h$ that
	\begin{equation}
		f(x_h^*) - f(x) \leq  \tmu^{(h-1)} (x_h^*) + \big(\beta_{h-1} + \tfrac{C}{l_{h-1} \psi \lambda} \sqrt{u_{h-1}}\big) \sigma^{(h-1)} (x^*)- \tmu^{(h-1)} (x) + \big(\beta_{h-1} + \tfrac{C}{l_{h-1} \psi \lambda} \sqrt{u_{h-1}}\big) \sigma^{(h-1)} (x), \label{eq:instant_regret_1}
	\end{equation}
where in \Cref{lemma:corrupted_confidence_bounds} we substitute $h-1$ and set $u_{\min} = l_{h-1} \psi$ (since each action selected in epoch $h-1$ in \Cref{alg:cpe} is played at least $\lceil l_{h-1} \psi \rceil$ times), to upper and lower bound $\max_{x \in \Xc_h} f(x)$ and $f(x)$, respectively.

Next, for any $x \in \Xc_h$, it holds that
\begin{align}
	\tmu^{(h-1)} (x) + \big(\beta_{h-1} + \tfrac{C}{l_{h-1} \psi \lambda} \sqrt{u_{h-1}}\big) \sigma^{(h-1)} (x) &\geq \max_{x \in \Xc_{h-1}} \Big( \tmu^{(h-1)} (x) - \big(\beta_{h-1} + \tfrac{C}{l_{h-1} \psi \lambda} \sqrt{u_{h-1}}\big) \sigma^{(h-1)} (x) \Big) \label{eq:instant_regret_2} \\
	&\geq  \tmu^{(h-1)} (x_h^*) - \big(\beta_{h-1} + \tfrac{C}{l_{h-1} \psi \lambda} \sqrt{u_{h-1}}\big) \sigma^{(h-1)} (x_h^*), \label{eq:instant_regret_3}
\end{align}
where \cref{eq:instant_regret_2} follows from the elimination condition (see Line 15 in \cref{alg:cpe}), and \cref{eq:instant_regret_3} holds since $x_h^* \in \Xc_{h} \subseteq \Xc_{h-1}$. 

Combining \cref{eq:instant_regret_3} with \cref{eq:instant_regret_1}, we obtain
\begin{align}
	f(x_h^*) - f(x) &\leq   
	2\big(\beta_{h-1} + \tfrac{C}{l_{h-1} \psi \lambda} \sqrt{u_{h-1}}\big) \sigma^{(h-1)} (x_h^*) + 
	2\big(\beta_{h-1} + \tfrac{C}{l_{h-1} \psi \lambda} \sqrt{u_{h-1}}\big) \sigma^{(h-1)} (x). \\
	&\leq 4\big(\beta_{h-1} + \tfrac{C}{l_{h-1} \psi \lambda} \sqrt{u_{h-1}}\big) \max_{x \in \Xc_{h-1}} \sigma^{(h-1)} (x).
\end{align}
The desired result then follows by upper bounding $\max_{x \in \Xc_{h-1}} \sigma^{(h-1)} (x)$ according to \Cref{lemma:max_var_lemma}. 
\end{proof}
We are ready to prove our main theorem, which is restated as follows.
\mainthm*

\begin{proof}
Throughout the proof, we condition on the confidence bounds from \cref{lemma:corrupted_confidence_bounds} holding true.
We use $u_h(x)$ to denote the number of times action $x$ is played in epoch $h$, and bound the cumulative regret of \Cref{alg:cpe} as follows:
	\begin{align}
		R_T &= \sum_{h=0}^{H-1} \sum_{x \in \Sc_h} \big(f(x^*) - f(x)\big) u_h(x) \label{eq:regret_main_1}  \\
		&\leq  u_0 B  + \sum_{h=1}^{H-1} \sum_{x \in \Sc_h} \big(f(x^*) - f(x)\big) u_h(x) \label{eq:regret_main_2}  \\
		&\leq  u_0 B  + \sum_{h=1}^{H-1} \sum_{x \in \Sc_h} u_h(x) \cdot 4 \Big(\beta_{h-1} + \tfrac{C\sqrt{u_{h-1}}}{l_{h-1} \psi \lambda}\Big) \sqrt{\tfrac{\eta (2\lambda + 1) \gamma_{l_{h-1}}}{l_{h-1}}}. \label{eq:regret_main_3}
		%	&\leq  u_0 B  + \sum_{h=1}^{H-1} 2u_h  \Big(\beta_{u_{h-1}} + \tfrac{C\sqrt{u_{h-1}}}{l_{h-1} \psi \lambda}\Big) \sqrt{\tfrac{\eta (2\lambda + 1) \gamma_{T}}{l_{h-1}}}.
		% &\leq  u_0 B  + \sum_{h=1}^{H-1} 2l_h(2 + \psi |S_h|) \Big(\beta_{u_{h-1}} + \tfrac{C\sqrt{u_{h-1}}}{l_{h-1} \psi \lambda}\Big) \sqrt{\tfrac{\eta (2\lambda + 1) \gamma_{T}}{l_{h-1}}}  \\
		% &\leq  u_0 B  + \sum_{h=1}^{H-1} 2(2 + \psi |S_h|) \Big(2\beta_{T} \sqrt{\eta (2\lambda + 1) l_{h-1}\gamma_{T}} + \tfrac{2C\sqrt{l_{h-1}(2 + \psi |S_{h-1}|)}}{\psi \lambda} \sqrt{\tfrac{\eta (2\lambda + 1) \gamma_{T}}{l_{h-1}}}\Big) \\
		% &= u_0 B  + \sum_{h=1}^{H-1} 2(2 + \psi |S_h|) \Big(2\beta_{T} \sqrt{\eta (2\lambda + 1) l_{h-1}\gamma_{T}} + \tfrac{2C\sqrt{(2 + \psi |S_{h-1}|)\eta (2\lambda + 1) \gamma_{T}}}{\psi \lambda}\Big) \\
		% &= u_0 B  + \sum_{h=1}^{H-1} 2(2 + \psi |S_h|) \Big(2\beta_{T} \sqrt{3 \eta l_{h-1}\gamma_{T}} + \tfrac{2C\sqrt{3(2 + \psi |S_{h-1}|)\eta\gamma_{T}}}{\psi}\Big) \\
		% &\leq u_0 B  + \sum_{h=1}^{H-1} 2(2 + \psi |S_h|) \Big(2\beta_{T} \sqrt{3 \eta T \gamma_{T}} + \tfrac{2C\sqrt{3(2 + \psi |S_{h-1}|)\eta\gamma_{T}}}{\psi}\Big) \\
	\end{align}
	Here, \cref{eq:regret_main_1} follows since only points from $\Sc_h$ are queried by the algorithm (and each point $x \in \Sc_h$ is queried $u_h(x)$ times), \cref{eq:regret_main_2} follows since the bound on the RKHS norm implies the same bound on the maximal function value when the kernel $k(\cdot, \cdot)$ is normalized (namely, $k(x,x)\leq 1$ for every $x$):
	\begin{equation}	
	  |f(x)| = |\langle f, k(x,\cdot) \rangle_k| \leq \|f\|_{k} \|k(x,\cdot) \|_k  = \|f\|_{k} \langle k(x,\cdot), k(x,\cdot)  \rangle_k^{1/2} \leq B \cdot k(x,x)^{1/2} \leq B,
	\end{equation}
	and \cref{eq:regret_main_3} follows from \cref{lemma:instant_regret_lemma} and by noting that $f(x^*) = \max_{x \in \Xc_h} f(x)$ for every $h =0,1,\dots,H-1$ (i.e., since the confidence bounds of \cref{lemma:corrupted_confidence_bounds} are valid, the global maximizer never gets eliminated). Next, from \cref{eq:regret_main_3}, by noting that $\sum_{x \in \Sc_h} u_h(x) = u_h$, we have:
	\begin{align}
		R_T &\leq u_0 B  + \sum_{h=1}^{H-1} 4u_h \Big(\beta_{h-1} + \tfrac{C\sqrt{u_{h-1}}}{l_{h-1} \psi \lambda}\Big) \sqrt{\tfrac{\eta (2\lambda + 1) \gamma_{l_{h-1}}}{l_{h-1}}} \\
			&\leq u_0 B  + \sum_{h=1}^{H-1} 4l_h(2 + \psi |\Sc_h|) \Big(\beta_{h-1} + \tfrac{C\sqrt{l_{h-1}(2 + \psi |\Sc_{h-1}|)}}{l_{h-1} \psi \lambda}\Big) \sqrt{\tfrac{\eta (2\lambda + 1) \gamma_{l_{h-1}}}{l_{h-1}}} \label{eq:regret_main_4} \\
			&\leq u_0 B  + \sum_{h=1}^{H-1} 4l_h(2 + \psi |\Sc_h|) \Big(\beta_{\Hbar} + \tfrac{C\sqrt{l_{h-1}(2 + \psi |\Sc_{h-1}|)}}{l_{h-1} \psi \lambda}\Big) \sqrt{\tfrac{\eta (2\lambda + 1) \gamma_{T}}{l_{h-1}}} \label{eq:regret_main_5} \\
			&= u_0 B + \sum_{h=1}^{H-1} 8(2 + \psi |\Sc_h|) \Big( \beta_{\Hbar} \sqrt{\eta (2\lambda + 1) l_{h-1}\gamma_{T}} +  \tfrac{C\sqrt{(2 + \psi |\Sc_{h-1}|) \eta (2\lambda + 1) \gamma_{T}}}{\psi \lambda}\Big) \label{eq:regret_main_missed} \\
			&\leq u_0 B + \sum_{h=1}^{H-1} 8(2 + \psi |\Sc_h|) \Big( \beta_{\Hbar} \sqrt{\eta (2\lambda + 1) T\gamma_{T}} +  \tfrac{C\sqrt{(2 + \psi |\Sc_{h-1}|) \eta (2\lambda + 1) \gamma_{T}}}{\psi \lambda}\Big) \label{eq:regret_main_6} \\
			&\leq u_0 B +  8 \Hbar (2 +  \tfrac{2\psi}{\ln \eta} \gamma_{T}) \Big( \beta_{\Hbar} \sqrt{\eta (2\lambda + 1) T\gamma_{T}} +  \tfrac{C\sqrt{(2 +  \tfrac{2\psi}{\ln \eta} \gamma_{T}) \eta (2\lambda + 1) \gamma_{T}}}{\psi \lambda}\Big), \label{eq:regret_main_7}
	\end{align}
	where \cref{eq:regret_main_4} follows from the bound on $u_h$ in \cref{lemma:length_of_epoch}, \cref{eq:regret_main_5} from the monotonicity of $\beta_h$ in $h \in \{1,\dotsc,\Hbar\}$ and $\gamma_t \in \{1,\dotsc,T\}$ in $t$ (see \cref{lemma:Hbar} for the statement that $h \le \Hbar$), \cref{eq:regret_main_missed} by rearranging and using $l_h = 2 l_{h-1}$, \cref{eq:regret_main_6} by upper bounding $l_{h-1}$ by $T$, and \cref{eq:regret_main_7} from the bound on $|\Sc_h|$ in \cref{lemma:number_of_switches}.

	By setting, $\psi = \tfrac{\ln \eta}{2 \gamma_{T}}$ as in the theorem statement, it follows that
		% \begin{equation}
		% 	R_T \leq u_0 B  +  6H \Big(2\beta_{T} \sqrt{3 \eta T \gamma_{T}} + \tfrac{12C\sqrt{\eta} \gamma_{T}^{3/2}}{\ln \eta}\Big).
		% \end{equation}
	\begin{equation}
		R_T \leq  u_0 B + 24 \Hbar \Big( \beta_{\Hbar} \sqrt{\eta (2\lambda + 1) T\gamma_{T}} +  C\sqrt{\tfrac{12\eta (2\lambda + 1) \gamma_T^3}{\lambda^2(\ln \eta)^2}}\Big).
	\end{equation}
	Treating $\lambda > 0$ as a constant, it suffices to set the switching parameter $\eta$ to some constant value (above one), so we choose $\eta = e$ (Euler's number). Then, we note that $u_0 = O(1)$ by design in the algorithm (recall that $l_0 = 2$, and note that $\psi \le 1$ except possibly when $T$ is small), and we write our regret bound as
    %Then, from \cref{lemma:length_of_epoch}, we note that $u_0 \leq 8B$ is also bounded by a constant (here, we trivially assume that $\gamma_T \geq 1$), and we can then write our regret bound as
	\begin{equation}
		R_T \leq  O\big( \Hbar ( \beta_{\Hbar} \sqrt{T\gamma_{T}} +  C\gamma_T^{3/2})\big).
	\end{equation}
	%Finally, the total number of epochs $H$ is random, but it can be deterministically upper bounded as follows. From noting that $u_h \geq 2l_h\gamma_T^{-1}$, and $l_h$ is exponentially increasing with the number of epochs, we have that $H = O(\log T)$.
	By using the notation $\Otilde(\cdot)$ to hide the multiplicative $\Hbar = \log_2 T$ factor, the final result then follows:
	\begin{equation}
		R_T \leq  \Otilde\big(\beta_{\Hbar} \sqrt{T\gamma_{T}} +  C\gamma_T^{3/2}\big).
	\end{equation}
\end{proof}

\section{Alternative Approach: Reduction to Linear Bandits} \label{sec:linear}

In this section, we introduce an alternative method for corrupted kernelized bandit optimization, and discuss its limitations. We reduce the kernelized bandit problem of dimension $d$ to a linear bandit problem of dimension $D$\footnote{The notation $D$ for the continuous domain $[0,1]^d$ will not be used in this appendix, so it it safe to use $D$ for this dimension quantity.} using techniques from \cite{Tak20}, and then solve the corrupted linear bandit problem using a modified version of the Robust Phased Elimination algorithm \cite{bogunovic2021stochastic}.

We consider a finite set of $D$ actions $\Xc_D=\{s_1,\dots,s_D\}\subseteq \Xc$,and denote by $V(\Xc_D)$ the vector subspace of $\Hc_k$ spanned by $\{k(\cdot, s_i):s_i\in \Xc_D\}$.  Following \cite{Tak20}, we consider using the orthogonal projection $\Pi_D(f)$ of $f$ onto $V(\Xc_D)$ as an approximation of $f$, where $\Pi_D(f)$ is also the unique interpolant of $f$ on $\Xc_D$ in $V(\Xc_D)$, i.e., $\Pi_D(f)(s_i)=f(s_i)$ for $i=1,\dots,D$. To design this set $\Xc_D$, we use Algorithm \ref{algo:newton} (taken from \cite{Tak20}), which takes the kernel $k$, domain $\Xc$, and an admissible error $e$ as input, and outputs $\Xc_D$ along with the Newton basis of $V(\Xc_D)$. Recalling that $\|f\|_k\leq B$, we run Algorithm \ref{algo:newton} with admissible error $e=\Delta/B$ for some constant $\Delta>0$. We will discuss the choice of $\Delta$ later.

\begin{algorithm}
    \caption{Newton Basis Construction \cite{Tak20}}
    \label{algo:newton}
    \begin{algorithmic}[1]
        \INPUT Kernel $k$, domain $\Xc$, admissible error $e$ \\
        \hspace*{-4.2ex} {\bf Output:} $\Xc_D=\{s_1,\dots,s_D\}\subseteq\Xc$, Newton basis $N_1,\dots,N_D$ of $V(\Xc_D)$
        \STATE $s_1\gets \argmax_{x\in \Xc} k(x,x)$
        \STATE $N_1(x)\gets k(x,s_1)/\sqrt{k(s_1,s_1)}$
        \FOR{$D \gets 1, 2, \dots$} 
        \STATE Define $P_D^2(x) = k(x,x)-\sum_{i=1}^D N_i^2 (x)$
        \IF{$\max_{x\in \Xc}P_D^2(x)<e^2$}
        \STATE {\bf return} {$\{s_1,\dots,s_D\}$ and $\{N_1,\dots,N_D\}$}
        \ENDIF
        \STATE $s_{D+1}\gets\argmax_{x\in \Xc}P_D^2(x)$
        \STATE $u(x)\gets k(x,s_{D+1})-\sum_{i=1}^D N_i(s_{D+1})N_i(x)$
        \STATE $N_{D+1}(x)\gets u(x)/\sqrt{P_D^2(s_{D+1})}$
        \ENDFOR
    \end{algorithmic}
\end{algorithm}

By rearranging the equations in \cite{Tak20} (Theorem 6 therein), we have that the number of points returned by the algorithm is $D=O\big((\log\frac{1}{\Delta})^d\big)$ for kernels with infinite smoothness (in particular, the SE kernel), and $D=O(\Delta^{-d/\nu})$ for kernels with finite smoothness $\nu$ (in particular, the Mat\'ern-$\nu$ kernel).

Since the Newton basis $\{N_1,\dots,N_D\}$ returned is the Gram-Schmidt orthonormalization of the basis $\{k(\cdot, s_i):s_i\in \Xc_D\}$, we have for any $f\in\Hc_k$ and $x\in\Xc$ that
\begin{align}
    \bigg|f(x)-\sum_{i=1}^D \langle f, N_i \rangle N_i(x)\bigg|\leq \|f\|_k \cdot e \le \Delta
\end{align}
under the choice $e = \Delta / B$.
Hence, for any fixed black-box $f\in\Hc_k$ with $\|f\|_k \le B$, there exists a $\theta\in\mathbb{R}^D$ with $\|\theta\|_2 \leq B$ such that for any $x\in\Xc$,
\begin{align}
    |f(x)-\langle \theta, \xtilde\rangle| \leq \Delta,
\end{align}
where for any given point $x$, we define $\xtilde=[N_1(x),\dots,N_D(x)]^T$. Now, we can reduce the corrupted kernelized bandit problem to a variant of the corrupted linear bandit problem \cite{bogunovic2021stochastic} on the transformed domain $\widetilde{\Xc}=\{\xtilde:x\in\Xc\}$ of dimension $D$, where $|\ytilde_t-\langle \theta, \xtilde_t\rangle-c_t-\epsilon_t|\leq \Delta$ for $t=1,\dots,T$.

\subsection{A Variant of Robust Phased Elimination}

We apply Algorithm \ref{algo:rpe}, a variant of the Robust Phased Elimination algorithm for stochastic linear bandits \cite{bogunovic2021stochastic}, on the space $\widetilde{\Xc}$ of dimension $D$, where the only difference from the original algorithm is the confidence bound in the elimination rule.

\begin{algorithm}
    \caption{Robust Phased Elimination}
    \label{algo:rpe}
    \begin{algorithmic}[1]
        \INPUT Actions $\widetilde{\Xc}\subseteq \RR^D$, kernel $k$, admissible error $e$, confidence $\delta\in(0,1)$, truncation parameter $\alpha\in(0,1)$, time horizon $T$
        \STATE $h\gets 0, m_0\gets 4D(\log\log D +18),\Ac_0\gets \widetilde{\Xc}$.
        % \STATE $\widehat{C}_h\gets C$ for known $C$, or $\widehat{C}_h\gets\min\{\frac{\sqrt{T}}{m_0\log_2T},m_0\sqrt{D}2^{\log_2T-h}\}$ for unknown $C$.
        \STATE Compute design $\zeta_h:\Ac_h\to[0,1]$ such that
        \begin{align}
            \max_{\xtilde\in\Ac_h}||\xtilde||_{\Gamma(\zeta_h)^{-1}}^2\leq 2D\text{, and }|\text{supp}(\zeta_h)|\leq m_0,
        \end{align}
        where $\Gamma(\zeta_h)=\sum_{\xtilde\in\Ac_h}\zeta_h(\xtilde)\xtilde\xtilde^T$ (e.g., using Frank-Wolfe \cite{lattimore2019learning})
        \STATE $u_h(\xtilde)\gets0$ if $\zeta(\xtilde)=0$, and $u_h(\xtilde)\gets\lceil m_h \max\{\zeta_h(\xtilde),\alpha\}\rceil$ otherwise.
        \STATE Take each action $x$ such that $\xtilde\in\Ac_h$ exactly $u_h(\xtilde)$ times, and get rewards $\{\ytilde_t\}_{t=1}^{u_h}$, where $u_h=\sum_{\xtilde\in\Ac_h}u_h(\xtilde)$.
        \STATE Estimate the parameter vector $\widetilde{\theta}_h$:
        \begin{align}
            \widetilde{\theta}_h=\Gamma_h^{-1}\sum_{t=1}^{u_h} \xtilde_t u_h(\xtilde_t)^{-1}\sum_{s\in\mathcal{T}(\xtilde_t)}\ytilde_s,
        \end{align}
        where $\Gamma_h^{-1}=\sum_{\xtilde\in\Ac_h} u_h(\xtilde)\xtilde\xtilde^T$ and $\Tc(\xtilde)=\{s\in\{1,\dots,u_h\}:\xtilde_s=\xtilde\}$.
        \STATE Update the active set of actions:
        \begin{align}
            \Ac_{h+1}\gets\Bigg\{\xtilde\in\Ac_h:\max_{\xtilde'\in\Ac_h}\langle\widetilde{\theta}_h,\xtilde'-\xtilde\rangle\leq4\Delta\sqrt{D(1+\alpha m_0)}+4\sqrt{\frac{D}{m_h}\log\frac{1}{\delta}}+\frac{4C}{\alpha m_h}\sqrt{D(1+\alpha m_0)}\Bigg\}.\label{eq:elimination}
        \end{align}
        \STATE $m_{h+1}\gets 2m_h, h\gets h+1$ and return to step 3 (terminating after $T$ actions are played).
    \end{algorithmic}
\end{algorithm}

The analysis of Algorithm \ref{algo:rpe} is very similar to that of \cite{bogunovic2021stochastic}, so we heavily rely on their auxiliary results and only focus on explaining the differences here.  
With $\widetilde{\theta}_h$ denoting the estimate of $\theta$ based on the corrupted observations $\{\widetilde{y}_t\}_{t=1}^{u_h}$ in the algorithm, and $\widehat{\theta}_h$ denoting the estimate of $\theta$ based on $\{\langle \theta, \xtilde_t\rangle+c_t+\epsilon_t\}_{t=1}^{u_h}$ (i.e., the corrupted observations if the linear model were exact) in the original algorithm, we have for all $h\geq 0$ and $\xtilde\in\Ac_h$ that
\begin{align}
    |\langle \xtilde, \widetilde{\theta}_h-\widehat{\theta}_h\rangle|
    \leq \Big\lvert \xtilde^T\Gamma_h^{-1}\sum_{t=1}^{u_h} \xtilde_t\Delta\Big\rvert
    \le \Delta\sum_{t=1}^{u_h} \big|\langle \xtilde,\Gamma_h^{-1}\xtilde_t\rangle\big|
    \stackrel{(a)}{\leq} \Delta\sqrt{u_h}\lvert\lvert \xtilde\rvert\rvert_{\Gamma_h^{-1}}
    \stackrel{(b)}{\leq} 2\Delta\sqrt{D(1+\alpha m_0)},\label{eq:extra}
\end{align}
where (a) uses the definition of $\|\cdot\|_{\Gamma_h^{-1}}$ and the fact that the $\ell_1$-norm is upper bounded by the $\ell_2$-norm times the square root of the vector length, and (b) uses Lemmas 2 and 3 from \cite{bogunovic2021stochastic}. Hence, in a fixed epoch $h$, we have for all $\xtilde\in\Ac_h$ that
\begin{align}
    |\langle \xtilde, \widetilde{\theta}_h-\theta\rangle|&\leq |\langle \xtilde, \widetilde{\theta}_h-\widehat{\theta}_h\rangle|+|\langle \xtilde, \widehat{\theta}_h-\theta\rangle|\\
    &\leq 2\Delta\sqrt{D(1+\alpha m_0)}+2\sqrt{\frac{D}{m_h}\log\frac{1}{\delta}}+\frac{2C}{\alpha m_h}\sqrt{D(1+\alpha m_0)},\label{eq:conf}
\end{align}
where the first term uses \eqref{eq:extra}, and the remaining terms are obtained with probability at least $1-2|\Xc|\delta$ by Lemma 4 in \cite{bogunovic2021stochastic}. 

Defining $\xbar=\argmax_{\xbar\in\widetilde{\Xc}} \langle \theta, \xbar\rangle$, by a similar analysis to Section A.2 in \cite{bogunovic2021stochastic}, we can show that the elimination rule in \eqref{eq:elimination} retains $\xbar$ in a given epoch with probability at least $1-2|\Xc|\delta$. 
%Using \eqref{eq:conf} and \eqref{eq:elimination}, we also have for a fixed epoch $h\geq 1$ and every $\xtilde\in\Ac_h$ that
%\begin{align}
%	\langle \theta, \xbar-\xtilde \rangle \leq 4 \Bigg( 2\Delta\sqrt{D(1+\alpha m_0)}+2\sqrt{\frac{D}{m_{h-1}}\log\frac{1}{\delta}}+\frac{2C}{\alpha m_{h-1}}\sqrt{D(1+\alpha m_0)} \Bigg)\label{eq:simple}
%\end{align}
%with probability at least $1-2|\Xc|\delta$.
Recalling that $x^\ast=\argmax_{x\in\Xc}f(x)$, we have
\begin{align}
    f(x^\ast)=\langle \theta, \xtilde^\ast \rangle + f(x^\ast)-\langle \theta, \xtilde^\ast\rangle\leq \langle \theta, \xtilde^\ast\rangle+\Delta\leq \langle \theta, \xbar \rangle+\Delta.
\end{align}
Hence, the cumulative regret can be upper bounded as follows
\begin{align}
    R_T=\sum_{t=1}^T f(x^\ast)-f(x_t)\leq \sum_{t=1}^T ( \langle \theta, \xbar \rangle+\Delta ) - ( \langle \theta, \xtilde_t \rangle-\Delta )= \sum_{t=1}^T \langle \theta, \xbar-\xtilde_t \rangle + 2\Delta T.
\end{align}
Again following the analysis of Section A.2 in \cite{bogunovic2021stochastic}, using \eqref{eq:conf} and \eqref{eq:elimination}, we can then show that the cumulative regret is
\begin{align}
    R_T=\Otilde\Bigg(\Delta T \sqrt{D} +\sqrt{DT\log\frac{|\Xc|}{\delta}} + C D^{3/2}\Bigg) \label{eq:RT_linear}
\end{align}
with probability at least $1-\delta$.
%For unknown $C$, \cite[Appendix A.3]{bogunovic2021stochastic} has shown that there is $O(C^2)$ additional regret incurred. Hence, the cumulative regret is
%\begin{align}
%    R_T=\Otilde\Bigg(\Delta T \sqrt{D} +\sqrt{DT\log\frac{|\Xc|}{\delta}} + D^{3/2}\log T + C^2\Bigg)
%\end{align}
%with probability at least $1-\delta$.

\subsection{The choice of $\Delta$}

The only remaining step now is to find a proper choice of $\Delta$, which is what dictates the choice of $D$ (along with the kernel).  The choice of $\Delta$ can be optimized with respect to the kernel parameters, and the optimal scaling is achieved by equating the first terms in \eqref{eq:RT_linear} with one of the other two terms (whichever is larger).  We first consider the choice $\Delta=\frac{1}{\sqrt{T}}$, which equates the first two terms (up to the $\log\frac{|\Xc|}{\delta}$ factor).

With $\Delta=\frac{1}{\sqrt{T}}$, it is known from \cite{Tak20} (Corollary 7 therein) that Algorithm \ref{algo:newton} results in $D=O\big((\log T)^d\big)$ for the SE kernel and $D=O(T^{\frac{d}{2\nu}})$ for the Mat\'ern kernel. Hence, the cumulative regret of our method is upper bounded as follows:
\begin{itemize}
    \item For the SE kernel, 
    \begin{equation}
        R_T=\Otilde\Big(\sqrt{T(\log T)^d\log\frac{|\Xc|}{\delta}} + C (\log T)^\frac{3d}{2}\Big). \label{eq:lin_se}
    \end{equation}
    % \item for the SE kernel with unkown $C$, $R_T=\Otilde\big(\sqrt{T(\log T)^d\log\frac{|\Xc|}{\delta}} + (\log T)^\frac{3d+2}{2} + C^2\Big)$;
    \item For the Mat\'ern kernel, 
    \begin{equation}
        R_T=\Otilde\Big(\sqrt{T^\frac{d+2\nu}{2\nu}\log\frac{|\Xc|}{\delta}} + C T^\frac{3d}{4\nu} \Big). \label{eq:lin_mat1}
    \end{equation}
    % \item for the Mat\'ern kernel with unknown $C$, $R_T=\Otilde\big(\sqrt{T^\frac{d+2\nu}{2\nu}\log\frac{|\Xc|}{\delta}} + T^\frac{3d}{4\nu} \log T + C^2\Big)$.
\end{itemize}
%\begin{rem}
%    Our method and analysis can be extended to infinite action settings using a covering argument \cite{Lat20}.
%\end{rem}
For the Mat\'ern kernel, we can sometimes do better by equating the first and third terms in \eqref{eq:RT_linear}, whereas for the SE kernel this is never the case.  The exact optimal choice depends on how $C$ scales with respect to $T$, but to avoid unwieldy expressions, we focus here on the direct $T$ dependence in \eqref{eq:RT_linear} so treat $C$ as a constant.  Equating the first and third terms, and ignoring the $\log T$ term, we find that we should set $\Delta=1/T^\frac{\nu}{d+\nu}$, which yields $D=O(T^\frac{d}{d+\nu})$ \cite{Tak20}, and gives
\begin{equation}
    R_T=\Otilde\Big(\sqrt{T^\frac{2d+\nu}{d+\nu}\log\frac{|\Xc|}{\delta}} + C T^\frac{3d}{2(d+\nu)}\Big). \label{eq:lin_mat2}
\end{equation}
We compare \eqref{eq:lin_mat1} and \eqref{eq:lin_mat2} for various $(\nu,d)$ pairs below.

For the SE kernel, the bound \eqref{eq:lin_se} turns out to be strong, matching our main result (Section \ref{sec:alg}), though we believe that our algorithm's feature of directly using the GP model (i.e., avoiding linear approximations) is still desirable.

For the Mat\'ern kernel, however, the resulting bound is not as strong; in particular, the non-corrupted terms in both \eqref{eq:lin_mat1} and \eqref{eq:lin_mat2} are larger than the corresponding term $\sqrt{T \gamma_T} = \Otilde(T^{\frac{\nu+d}{2\nu + d}})$ in our main result.\footnote{For \eqref{eq:lin_mat1}, this is seen by writing $T^{\frac{d+2\nu}{2\nu}} = T^{1 + \frac{d}{2\nu}}$ and noting that $\frac{d}{2\nu}$ exceeds $\gamma_T = \Otilde(T^{\frac{d}{2\nu + d}})$.  For \eqref{eq:lin_mat2}, it is seen by writing $\sqrt{T^\frac{2d+\nu}{d+\nu}} = T^\frac{2d+\nu}{2d+2\nu}$, and noting that subtracting $d$ from both the numerator and denominator makes the fraction smaller.}  The same goes for the corrupted terms, with the root cause for both terms being that either choice of $D$ above is strictly higher than $\gamma_T$.  For the corrupted term, this is further highlighted by comparing the regimes in which the bound remains sublinear:
\begin{itemize}
    \item The term $T^\frac{3d}{4\nu}$ in \eqref{eq:lin_mat1} is sublinear when $\nu > \frac{3}{4}d$;
    \item The term $T^\frac{3d}{2(d+\nu)}$ in \eqref{eq:lin_mat2} is sublinear when $\nu > \frac{d}{2}$;
    \item The analogous term $\gamma_T^{3/2} = \Otilde(T^{\frac{3d}{4\nu+2d}})$ in the main body is sublinear under the milder condition $\nu > d/4$. 
\end{itemize}
Note that in general, we have for constant $C$ that \eqref{eq:lin_mat1} is a better bound than \eqref{eq:lin_mat2} when $\nu > d$, \eqref{eq:lin_mat2} is better than \eqref{eq:lin_mat1} when $\nu \in \big(\frac{d}{2},d\big)$, and both fail to be sublinear when $\nu \le \frac{d}{2}$.

\section{Supplementary Experimental Results}\label{sec:supp_exp}

This section contains the experimental results on $f_1$ with $C=100$ (Figure \ref{fig:f1_100}), and on Robot3D with $C=50$ (Figure \ref{fig:robot3d_50}).  The overall findings are generally similar to those in the main text, and are not repeated here.

\begin{figure*}[h!]
	\centering
    \minipage[t]{0.33\textwidth}
    \includegraphics[width=\linewidth]{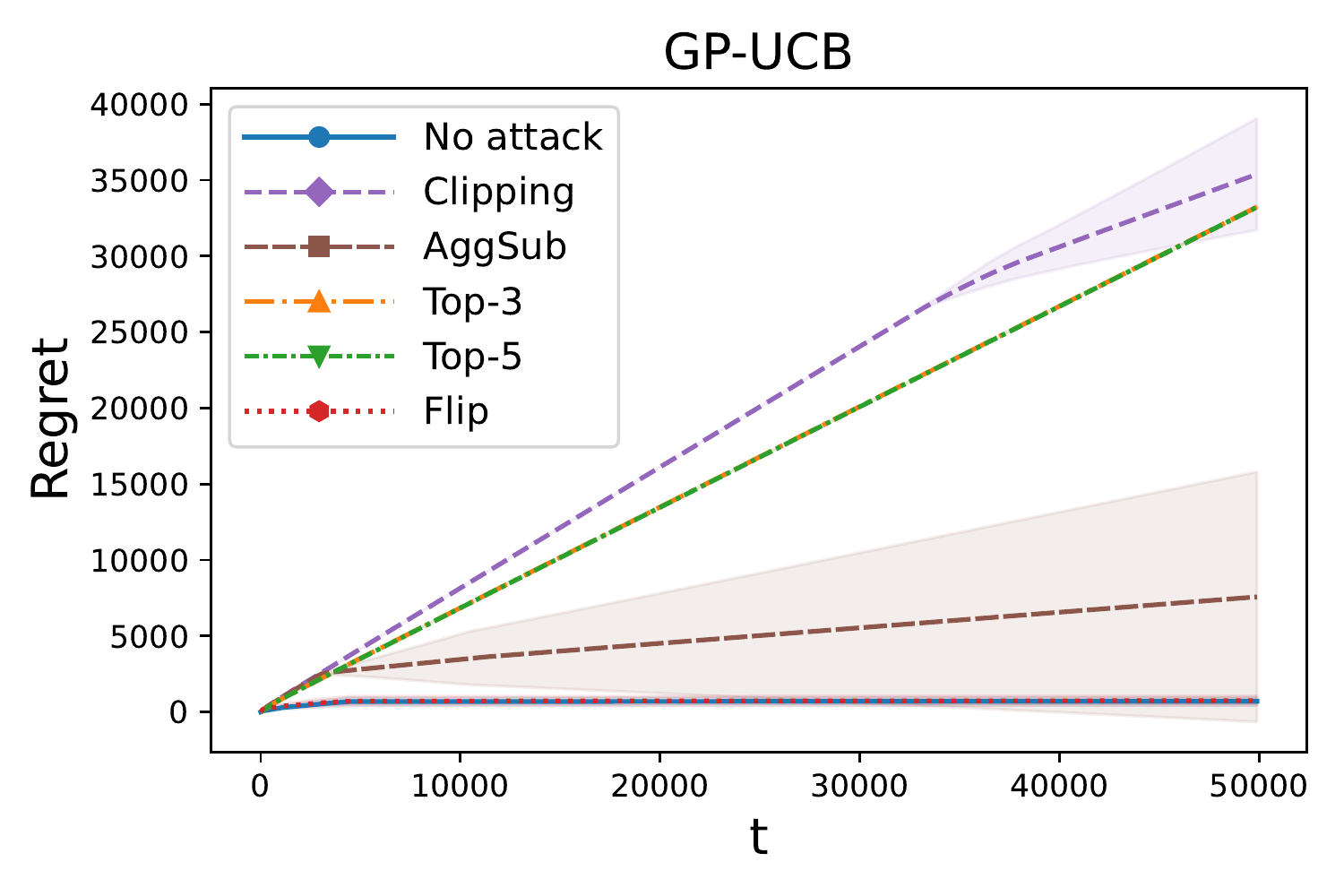}
    \endminipage
    \minipage[t]{0.33\textwidth}
    \includegraphics[width=\linewidth]{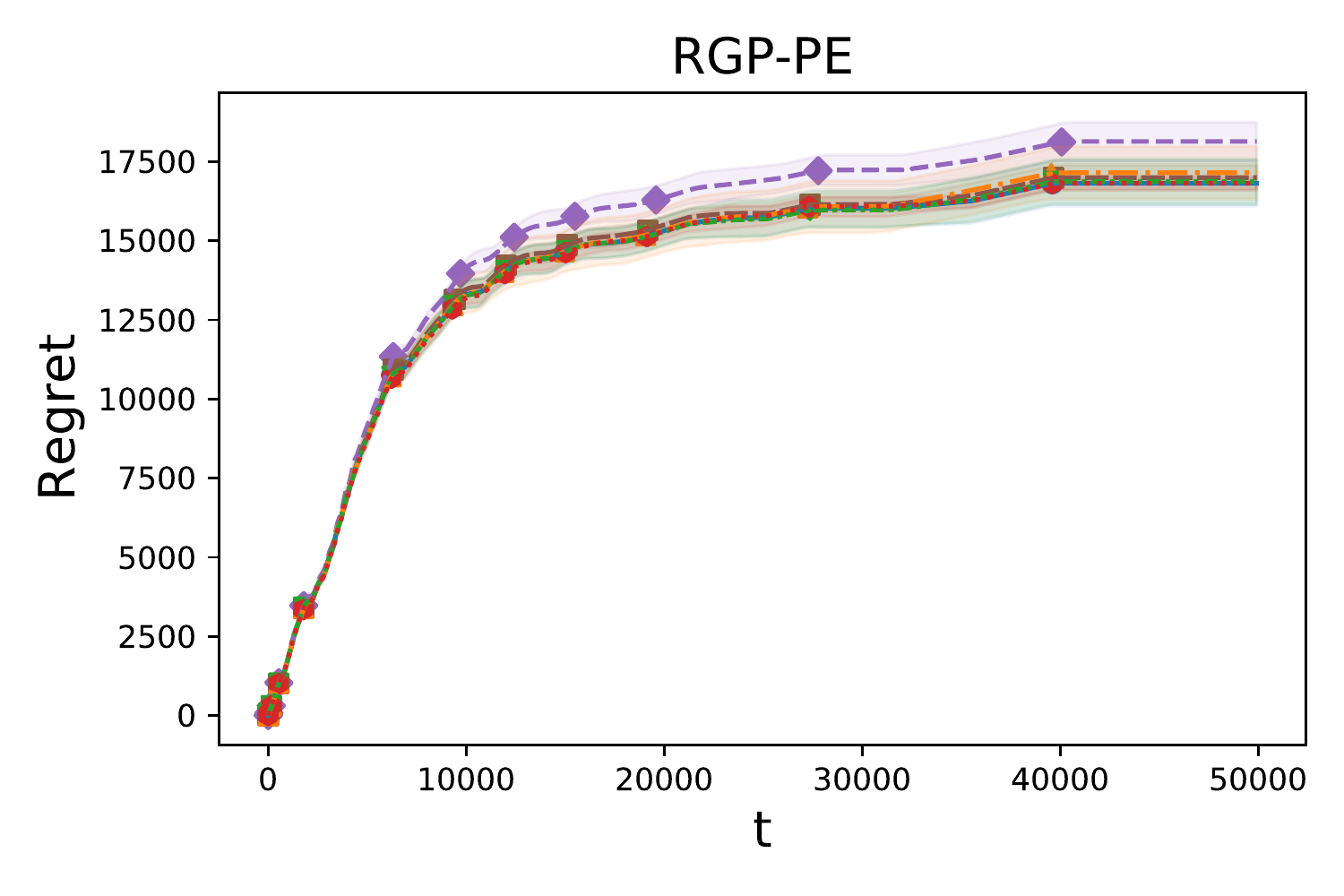}
    \endminipage
    \minipage[t]{0.33\textwidth}
    \includegraphics[width=\linewidth]{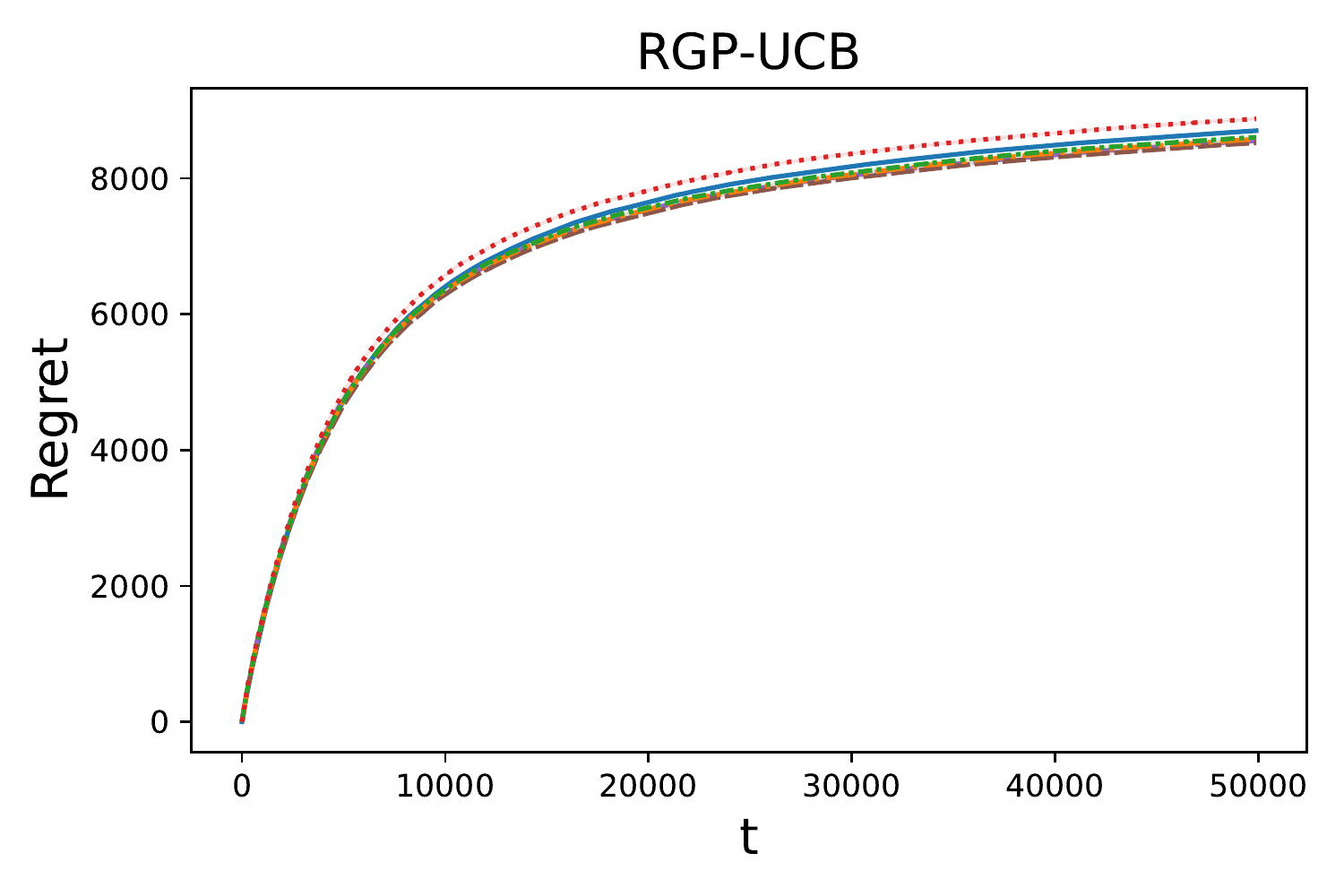}
    \endminipage
    \newline
    \minipage[t]{0.33\textwidth}
    \includegraphics[width=\linewidth]{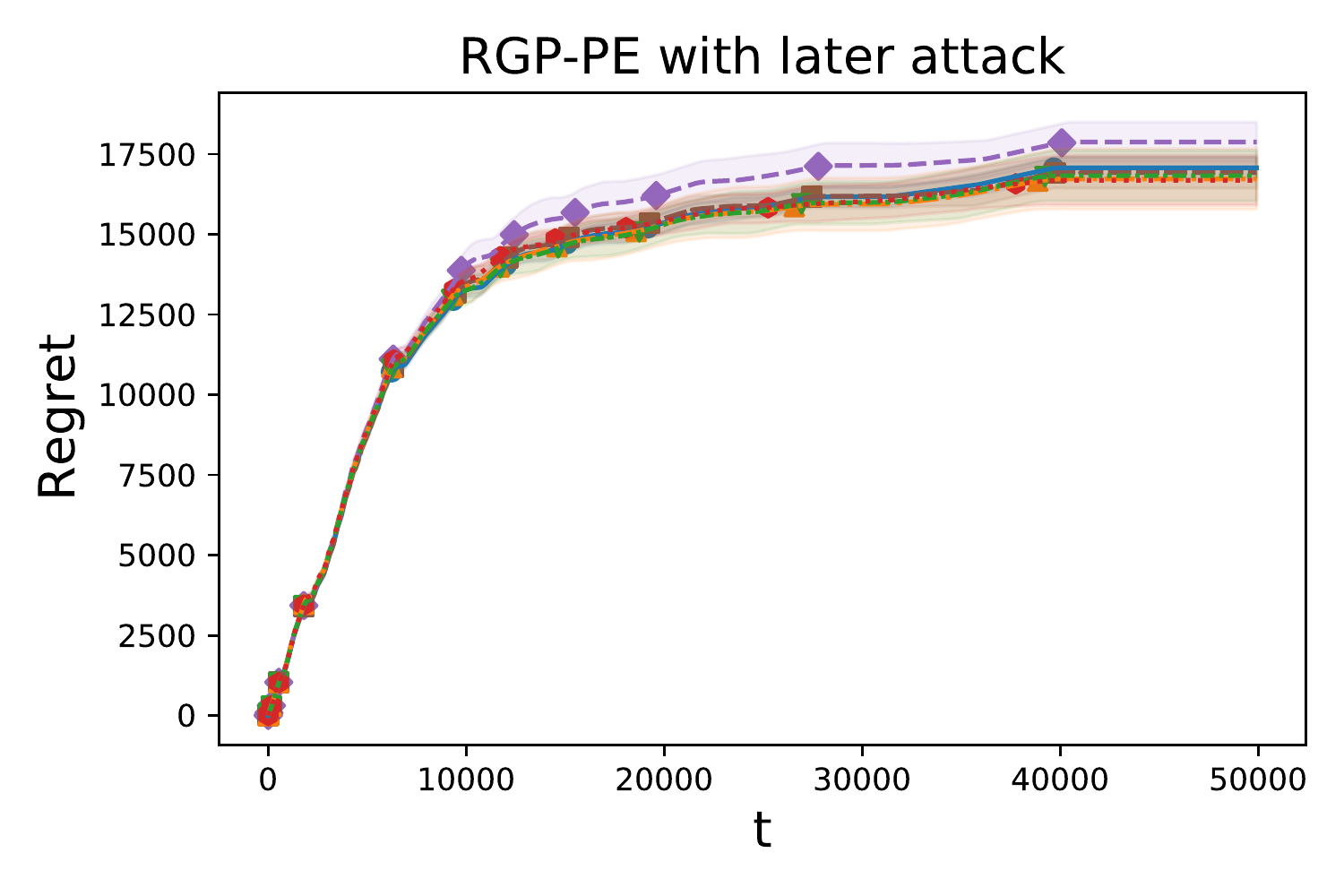}
    \endminipage
    \minipage[t]{0.33\textwidth}
    \includegraphics[width=\linewidth]{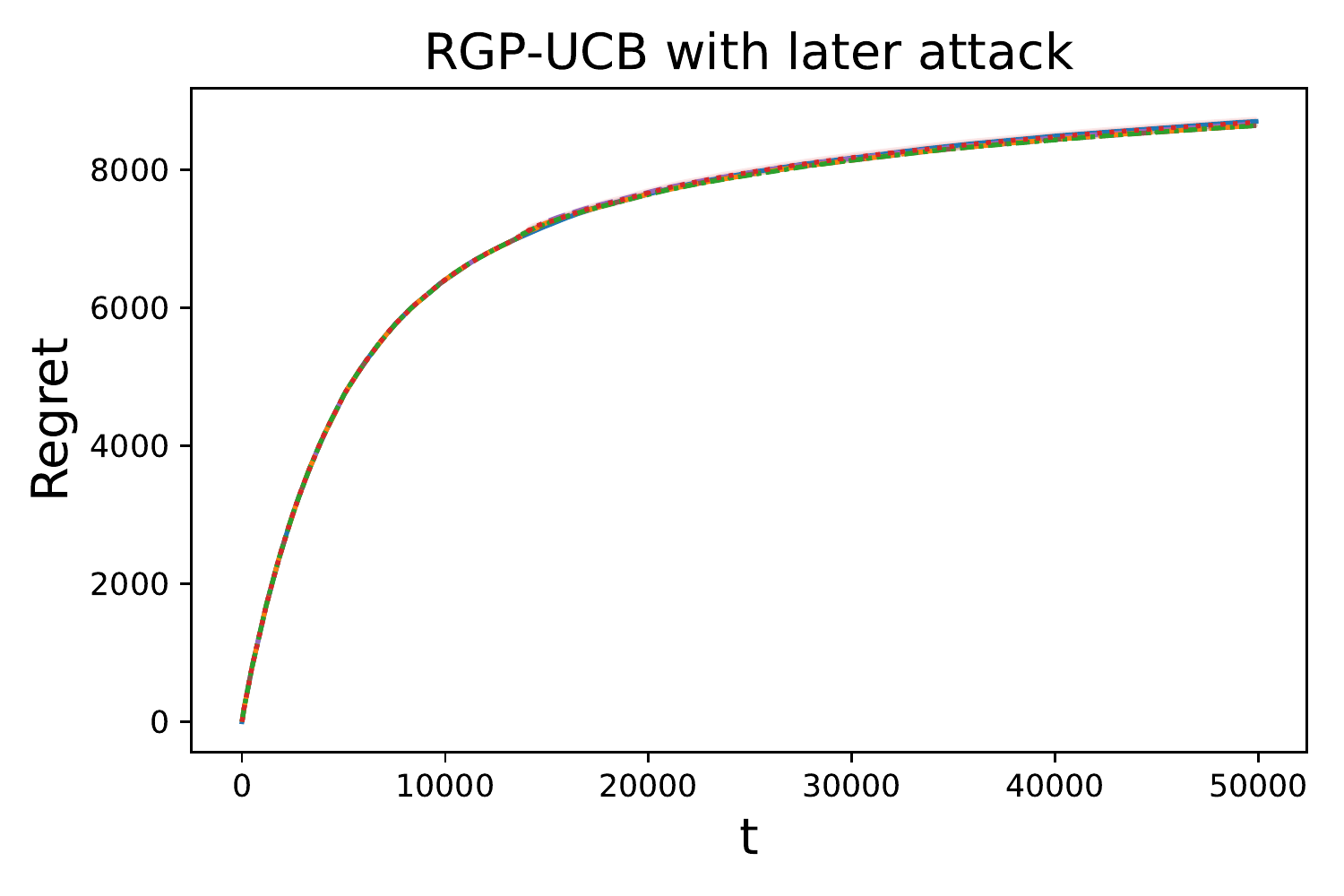}
    \endminipage
    \caption{Performance on $f_1$ with $C=100$. Note that for GP-UCB, the curves for Top-3 and Top-5 are indistinguishable, so only the latter is clearly visible.  Similar trends are observed to the case $C=50$ in Figure \ref{fig:f1_50}.}
    \label{fig:f1_100}
\end{figure*}

\begin{figure*}[h!]
	\centering
    \minipage[t]{0.33\textwidth}
    \includegraphics[width=\linewidth]{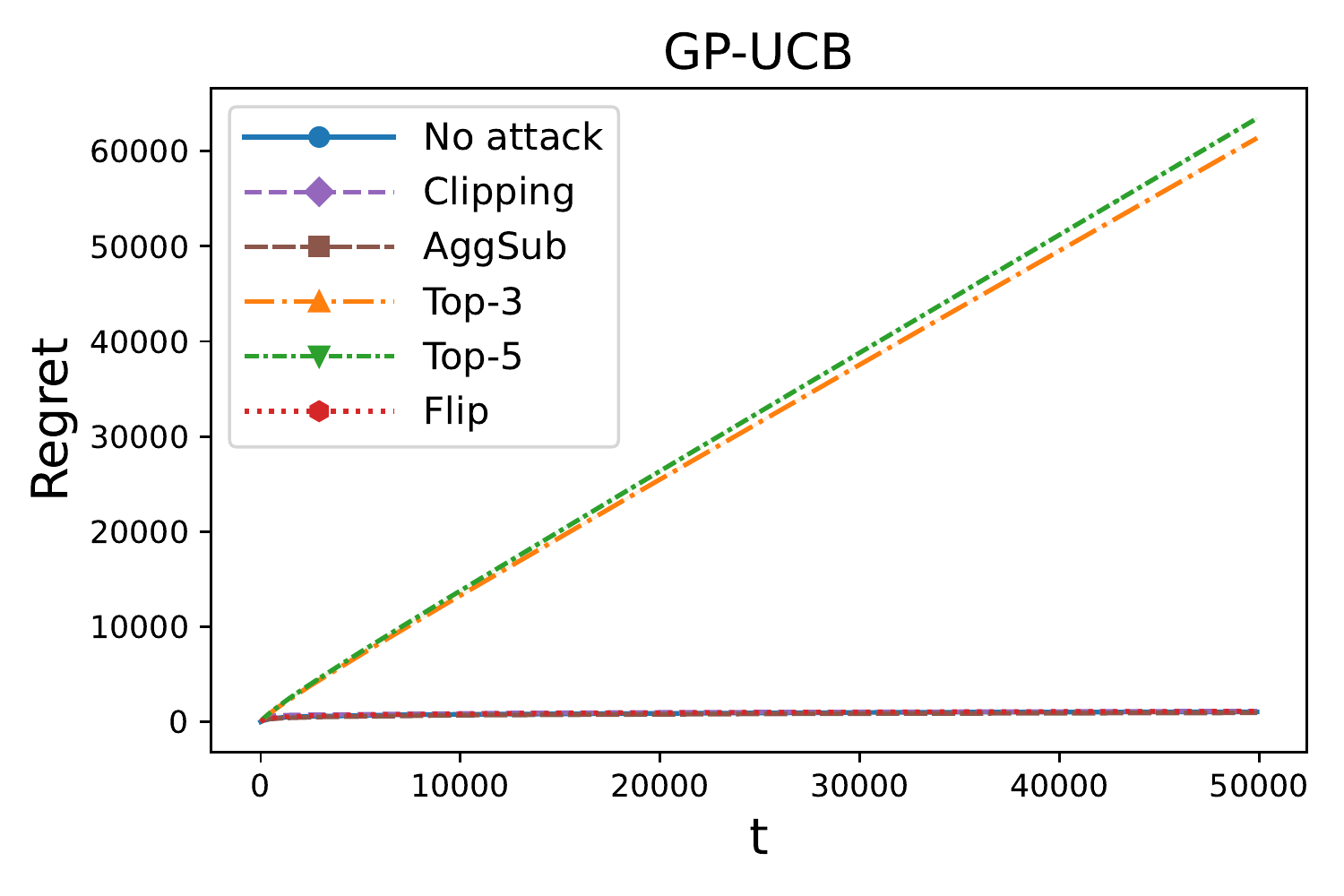}
    \endminipage
    \minipage[t]{0.33\textwidth}
    \includegraphics[width=\linewidth]{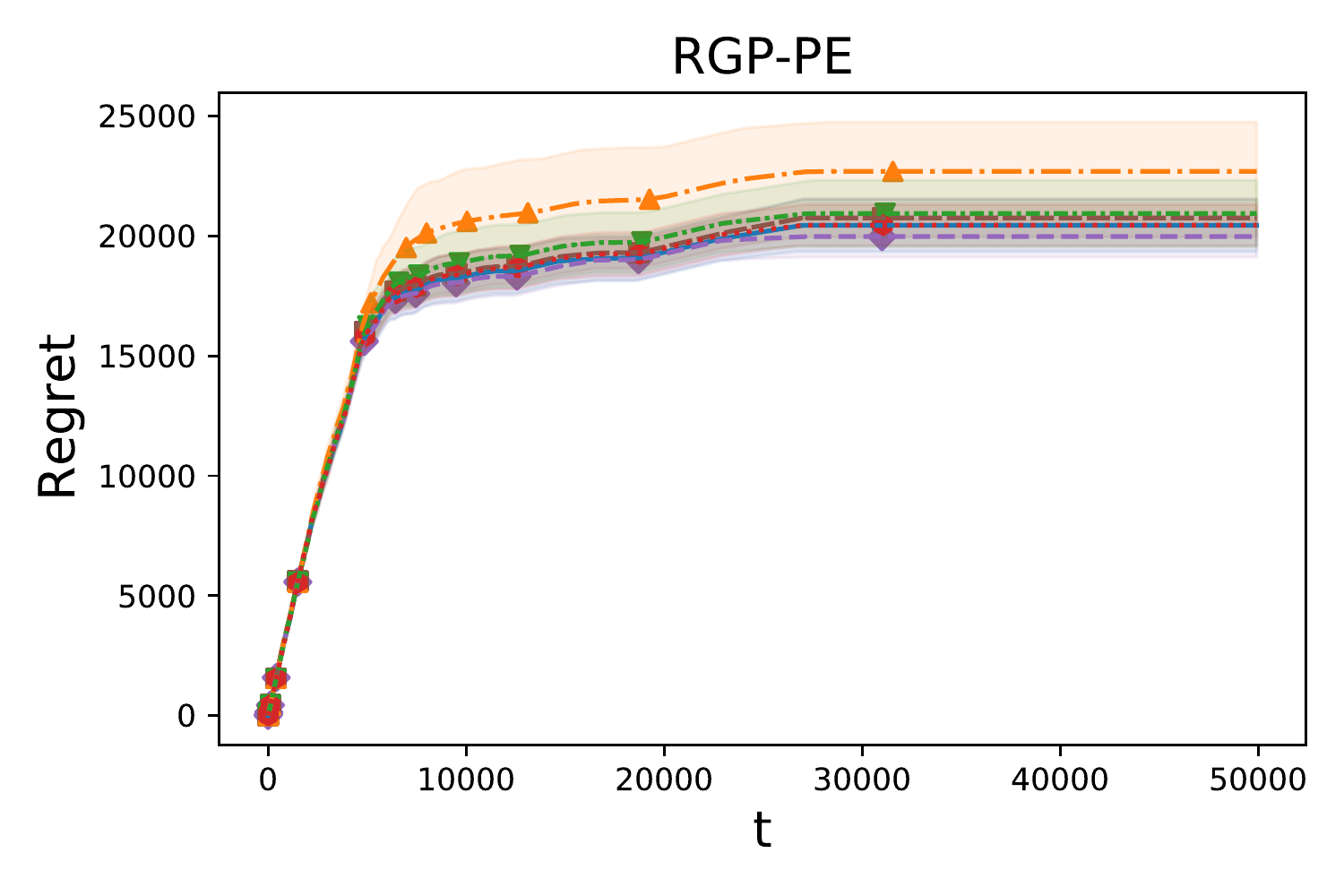}
    \endminipage
    \minipage[t]{0.33\textwidth}
    \includegraphics[width=\linewidth]{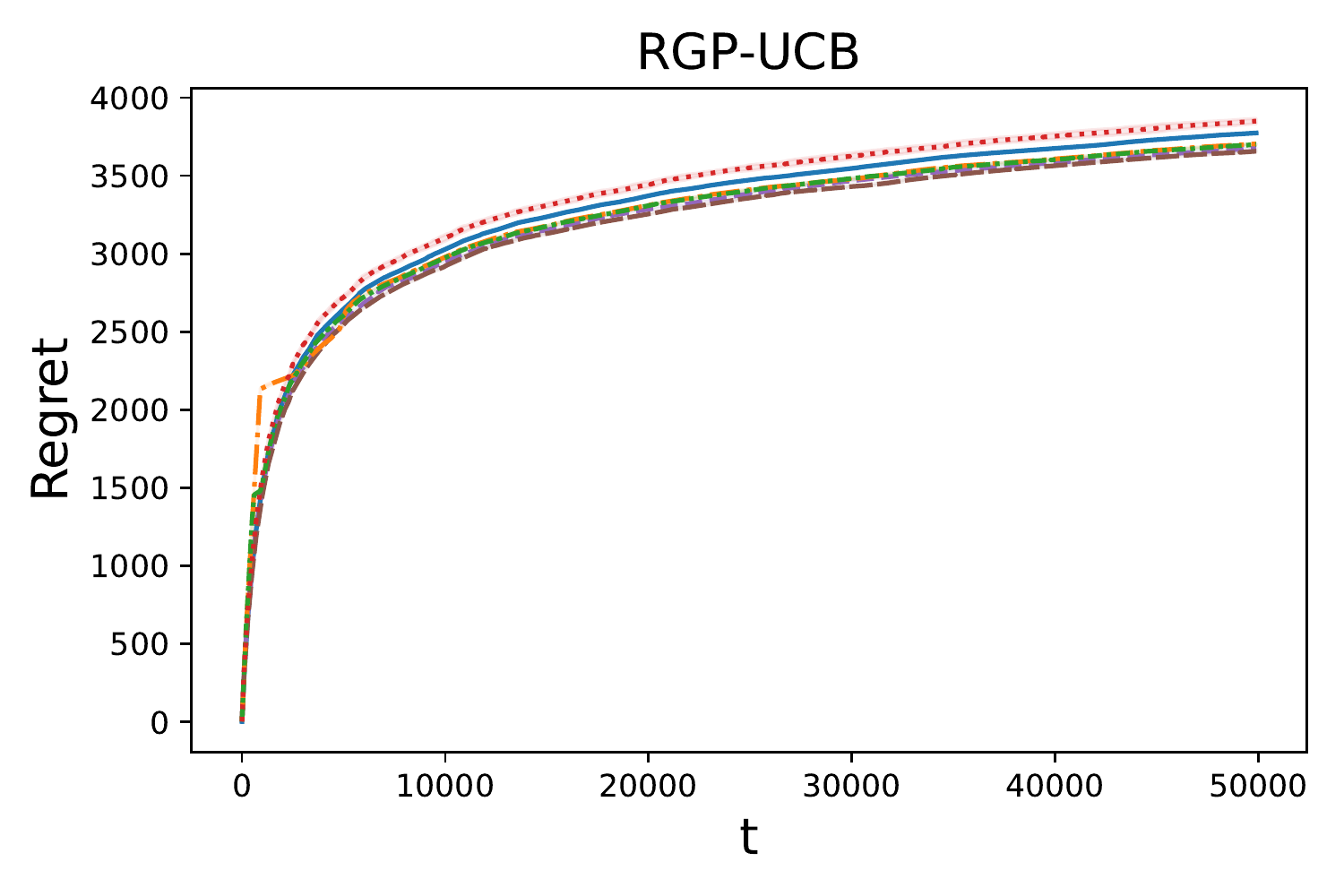}
    \endminipage
    \newline
    \minipage[t]{0.33\textwidth}
    \includegraphics[width=\linewidth]{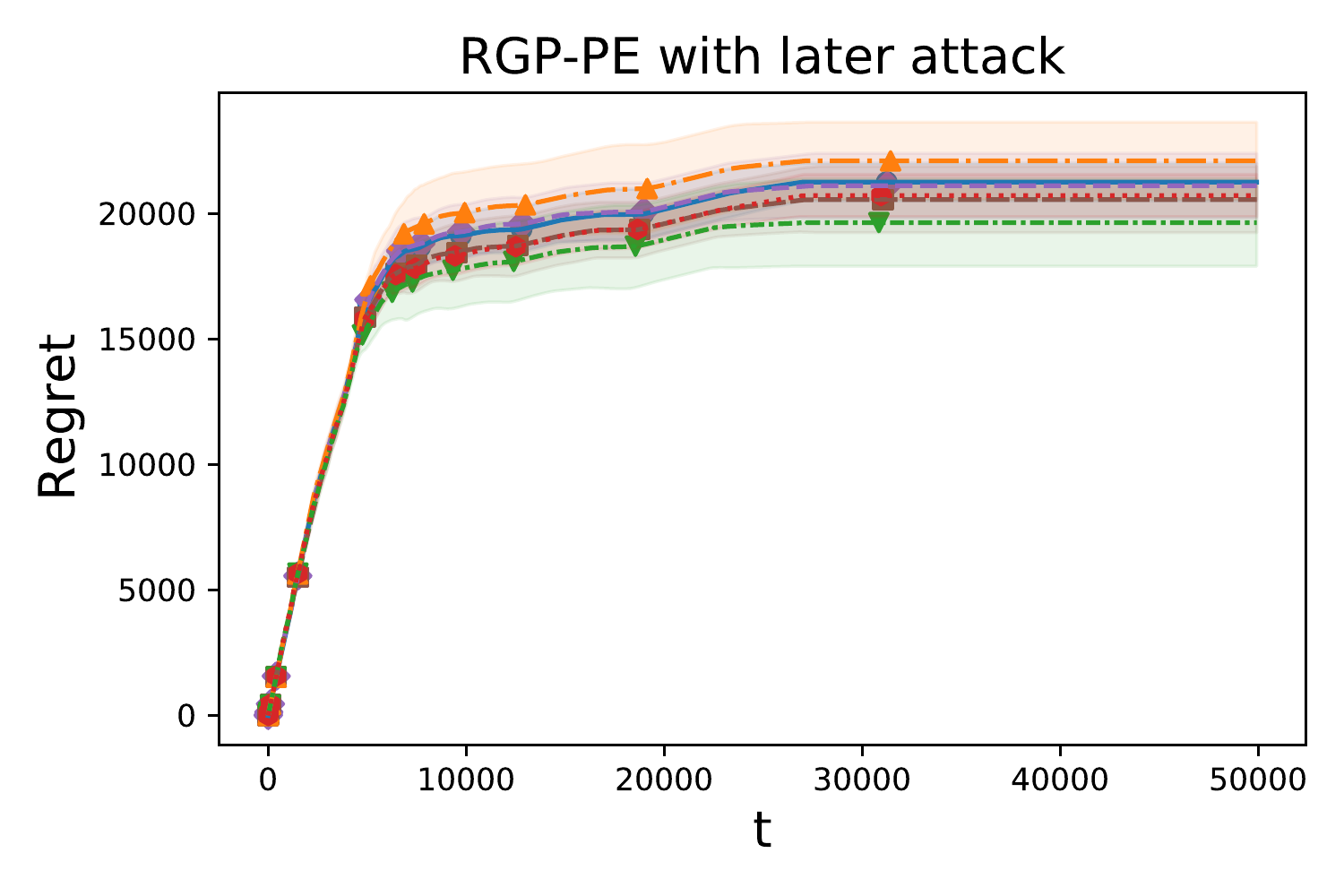}
    \endminipage
    \minipage[t]{0.33\textwidth}
    \includegraphics[width=\linewidth]{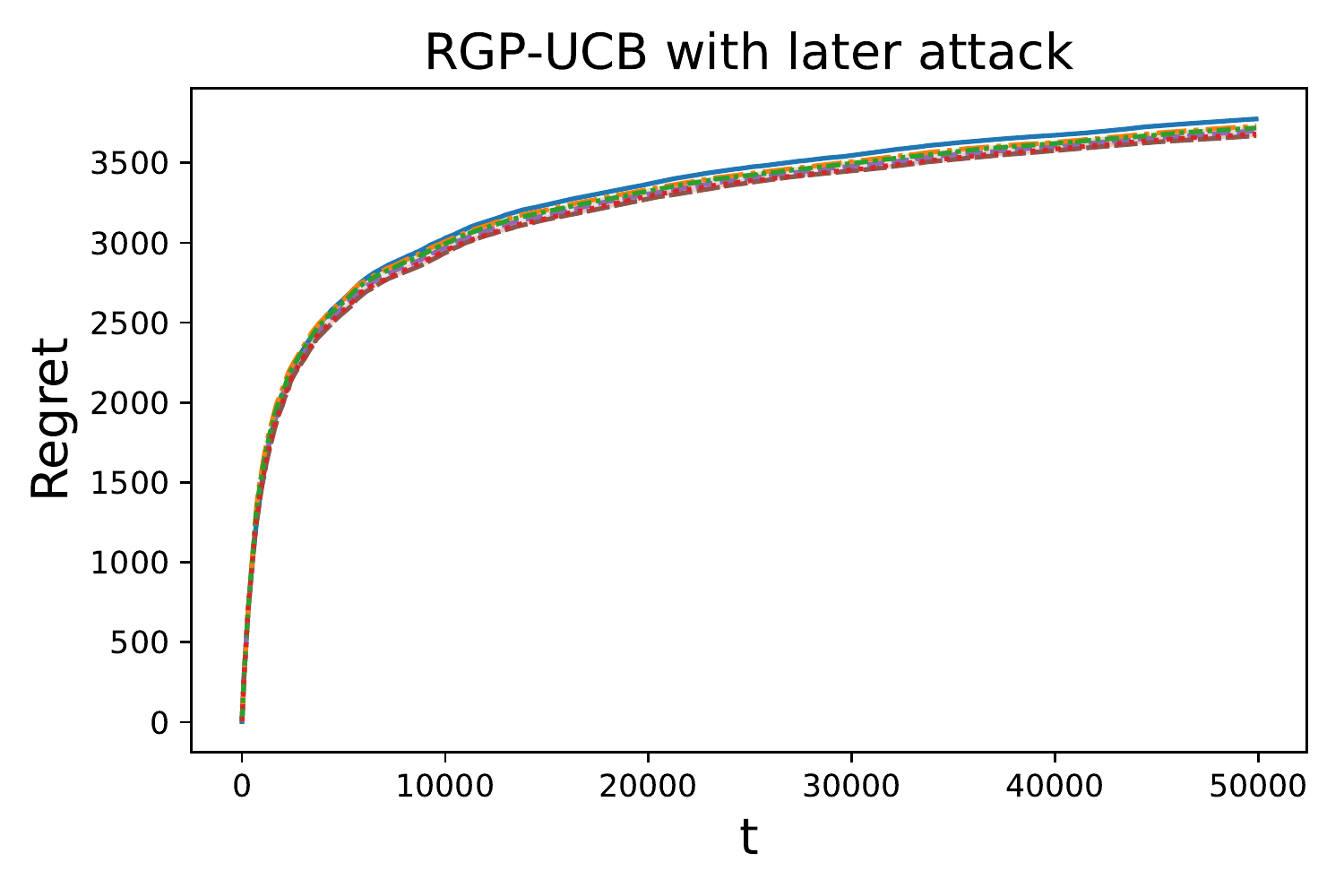}
    \endminipage
    \caption{Performance on Robot3D with $C=50$.  Similar trends are observed to the case $C=100$ in Figure \ref{fig:robot3d_100}.}
    \label{fig:robot3d_50}
\end{figure*}

\end{document}